\pdfoutput=1
\documentclass[11pt]{article} \usepackage{times}
\usepackage[letterpaper, left=1in, right=1in, top=1in, bottom=1in]{geometry}

\usepackage[utf8]{inputenc} \usepackage[T1]{fontenc}    \usepackage{url}            \usepackage{booktabs}       \usepackage{amsfonts}       \usepackage{nicefrac}       \usepackage{microtype}      \usepackage{mkolar_definitions}
\usepackage{xspace}
\usepackage{amsmath}
\usepackage{algorithm,algorithmic}
\usepackage{color}
\usepackage{enumitem}
\usepackage{comment}
\usepackage{bm}
\usepackage{graphicx}

\usepackage{apptools}
\usepackage[page, header]{appendix}
\usepackage{titletoc}
\usepackage{array}

\usepackage[scaled=.9]{helvet}

\usepackage{url}
\usepackage{authblk}
\usepackage{caption}

\usepackage[table,dvipsnames]{xcolor}
\usepackage[colorlinks=true, linkcolor=blue, citecolor=blue]{hyperref}

\usepackage{tikz}
\usetikzlibrary{decorations.pathreplacing}
\usepackage{wrapfig}

\newcommand{\veps}{\varepsilon}
\newcommand{\defeq}{:=}

\newcommand{\unif}{\mathrm{unif}}
\newcommand{\kl}[2]{\mathrm{KL}(#1 \vert\vert #2)}
\newcommand{\tv}{\mathrm{TV}}
\newcommand{\hell}[2]{\mathrm{H}^2(#1 \vert\vert #2)}
\newcommand{\etamin}{\eta_{\mathrm{min}}}
\newcommand{\random}{\mathrm{random}}

\newcommand{\err}{\mathrm{err}}

\newcommand{\epsopt}{\varepsilon_{\mathrm{opt}}}

\newcommand{\epssup}{\varepsilon_{\mathrm{sup}}}
\newcommand{\epstv}{\varepsilon_{\mathrm{TV}}}
\newcommand{\epssamp}{\varepsilon_{\mathrm{samp}}}
\newcommand{\epsstat}{\varepsilon_{\mathrm{stat}}}
\newcommand{\epsescape}{\varepsilon_{\mathrm{escape}}}

\newcommand{\homer}{\textsc{Homer}\xspace}
\newcommand{\pcid}{\textsc{Pcid}\xspace}
\newcommand{\olive}{\textsc{Olive}\xspace}
\newcommand{\ouralg}{\textsc{Flambe}\xspace}
\newcommand{\ouralgf}{\textsc{Flambe.F}\xspace}
\newcommand{\ucbvi}{\textsc{Lsvi-Ucb}\xspace}

\newcommand{\phih}{\phi_h(x_h,a_h)}
\newcommand{\phiH}{\phi_{H-1}(x_{H-1},a_{H-1})}
\newcommand{\phiht}{\phi_{\tilde{h}}(x_{\tilde{h}},a_{\tilde{h}})}
\newcommand{\muh}{\mu_h(x_{h+1})}
\newcommand{\rhotrain}{\rho_h^{\mathrm{train}}}
\newcommand{\rhopre}{\rho_h^{\mathrm{pre}}}
\newcommand{\Mhat}{\widehat{\Mcal}}
\newcommand{\Mtil}{\widetilde{\Mcal}}
\newcommand{\dlv}{d_{\mathrm{LV}}}
\newcommand{\mle}{\textsc{Mle}\xspace}
\newcommand{\samp}{\textsc{Samp}\xspace}
\newcommand{\xabsorb}{x_{\mathrm{absorb}}}
\newcommand{\aabsorb}{a_{\mathrm{absorb}}}
 \newcommand{\version}{arxiv}
\newcommand{\poly}{\mathrm{poly}}

\usepackage{natbib}
\bibpunct{(}{)}{;}{a}{,}{,}

\newcount\Comments  \Comments=0 \definecolor{darkgreen}{rgb}{0,0.5,0}
\definecolor{darkred}{rgb}{0.7,0,0}
\definecolor{teal}{rgb}{0.3,0.8,0.8}
\definecolor{orange}{rgb}{1.0,0.5,0.0}
\definecolor{purple}{rgb}{0.8,0.0,0.8}
\newcommand{\sk}[1]{\ifnum\Comments=1{\noindent{\textcolor{magenta}{\{{\bf SK:} \em #1\}}}}\fi}
\newcommand{\alekh}[1]{\ifnum\Comments=1{\noindent{\textcolor{darkred}{\{{\bf AA:} \em #1\}}}}\fi}
\newcommand{\akshay}[1]{\ifnum\Comments=1{\noindent{\textcolor{darkgreen}{\{{\bf AK:} \em #1\}}}}\fi}
\newcommand{\wen}[1]{\ifnum\Comments=1{\noindent{\textcolor{purple}{\{{\bf WS:} \em #1\}}}}\fi}

\usepackage{authblk}

\title{FLAMBE: Structural Complexity and \\Representation Learning of Low Rank MDPs}

\date{}
\author{Alekh Agarwal}
\author{Sham Kakade}
\author{Akshay Krishnamurthy}
\author{Wen Sun\thanks{alekha@microsoft.com, sham@cs.washington.edu, akshaykr@microsoft.com, sun.wen@microsoft.com}}
\affil{Microsoft Research}

\begin{document}

\maketitle

\vspace{-1cm}
\begin{abstract}

\iffalse
We consider reinforcement learning in complex environments where
function approximation is employed to cope with large state spaces. We
focus on low rank MDPs, in which transitions are linear functions of
an unknown state-action embedding. Our algorithm engages in
representation learning to approximate the state-action embedding and
exploration to refine the embedding based on new experience. Under
realizability and reachability assumptions, we prove that the
algorithm enjoys a sample complexity that is polynomial in all
relevant parameters, but independent of the size of the state space,
and is computationally tractable via a reduction to optimization
primitives. We also provide a number of structural results, in
particular showing that the low rank MDP model is significantly more
expressive than the recently studied block MDP. As a consequence, our
results greatly expand the scope of statistically and computationally
tractable reinforcement learning.
\fi

In order to deal with the curse of dimensionality in reinforcement
learning (RL), it is common practice to make parametric assumptions
where values or policies are functions of some low dimensional feature
space.  This work focuses on the representation learning question: how
can we \emph{learn} such features? Under the assumption that the
underlying (unknown) dynamics correspond to a low rank transition
matrix, we show how the representation learning question is related to
a particular \emph{non-linear} matrix decomposition problem.
Structurally, we make precise connections between these low rank MDPs
and latent variable models, showing how they significantly generalize
prior formulations for representation learning in RL.
Algorithmically, we develop \ouralg, which engages in exploration and
representation learning for provably efficient RL
in low rank transition models.
 \end{abstract}

\section{Introduction}
\label{sec:intro}
The ability to learn effective transformations of complex data
sources, sometimes called representation learning, is an essential
primitive in modern machine learning, leading to remarkable
achievements in language modeling, vision, and serving as a partial
explanation for the success of deep learning more
broadly~\citep{bengio2013representation}.
In Reinforcement Learning (RL), several works have shown empirically
that learning succinct representations of perceptual inputs can
accelerate the search for decision-making
policies~\citep{pathak2017curiosity,tang2017exploration,oord2018representation,srinivas2020curl}. However,
representation learning for RL is far more subtle than it is for
supervised
learning~\citep{du2019good,van2019comments,lattimore2019learning}, and
the theoretical foundations of representation learning for RL are
nascent.

The first question that arises in this context is: what is a good
representation? Intuitively, a good representation should help us
achieve greater sample efficiency on downstream tasks.  For supervised
learning, several theoretical works adopt the perspective that a good
representation should permit simple models to achieve high accuracy on
tasks of
interest~\citep{baxter2000model,maurer2016benefit,arora2019theoretical,tosh2020contrastive}. Lifting
this perspective to reinforcement learning, it is natural to ask that
we can express value functions and policies as simple functions of our
representation. This may allow us to leverage recent work on sample efficient RL with parametric function approximation.

The second question is: how do we \emph{learn} such a representation
when it is not provided in advance? This question is particularly
challenging because representation learning is intimately tied to
exploration. We cannot learn a good representation without a
comprehensive dataset of experience from the environment, but a good
representation may be critical for efficient exploration.

This work considers these questions in the context of low rank
MDPs~\citep{jiang2017contextual} (also known as factorizing
MDPs~\citep{rendle2010factorizing}, factored linear
MDPs~\citep{yao2014pseudo}, and linear
MDPs~\citep{jin2019provably,yang2019reinforcement}), which we argue
provide a natural framework for studying representation learning in
RL. Concretely, these models assume there exists low dimensional
embedding functions $\phi(x,a),\mu(x')$ such that the transition
operator $T$ satisfies $T(x' \mid x,a) = \inner{\phi(x,a)}{\mu(x')}$,
where $T(x' \mid x,a)$ specifies the probability of the next state
$x'$ given the previous state $x$ and action $a$.  Low rank MDPs
address the first issue above (on what constitutes a good
representation) in that if the features $\phi$ are known to the
learner, then sample efficient learning is
possible~\citep{jin2019provably,yang2019reinforcement}.

\begin{table}[t]
\begin{center}
\ifthenelse{\equal{\version}{arxiv}}{
\begin{tabular}{|>{\centering\arraybackslash} m{4.5cm}|>{\centering\arraybackslash} m{3cm}|>{\centering\arraybackslash} m{3.75cm}|>{\centering\arraybackslash} m{3cm}|}
}{
\begin{tabular}{|>{\centering\arraybackslash} m{3.5cm}|>{\centering\arraybackslash} m{2.5cm}|>{\centering\arraybackslash} m{3.75cm}|>{\centering\arraybackslash} m{2.2cm}|}
}
\hline
Algorithm & Setting & Sample Complexity & Computation\\
\hline\hline
{\small \pcid~\citep{du2019provably}} & block MDP & $d^4 H^2 K^4\rbr{\frac{1}{\eta^4\gamma^2} + \frac{1}{\veps^2}}$ & Oracle efficient\\
\hline
{\small \homer~\citep{misra2019kinematic}} & block MDP & $d^8 H^4 K^4\rbr{\frac{1}{\eta^3} + \frac{1}{\veps^2}}$ & Oracle efficient\\
\hline
{\small \olive~\citep{jiang2017contextual}} & low Bellman rank & \vspace*{0.05cm} $\left. \frac{d^2 H^3K}{\veps^2} \right.$ \vspace*{0.05cm}& Inefficient\\
\hline
{\small \cite{sun2018model}} & low Witness rank & \vspace*{0.05cm} $\frac{d^2 H^3K}{\veps^2}$ \vspace*{0.05cm}& Inefficient\\
\hline
{\small \ouralg (this paper)} & low rank MDP & \vspace*{0.05cm} $\left. \frac{d^{7}K^9H^{22}}{\veps^{10}}\right.$ \vspace*{0.05cm} & Oracle efficient\\
\hline
\end{tabular}
\end{center}
\ifthenelse{\equal{\version}{arxiv}}{\vspace{-0.5cm}}{}
\caption{Comparison of methods for representation learning in
  RL. Settings from least to most general are: block MDP, low rank
  MDP, low Bellman rank, low Witness rank. In all cases $d$ is the
  embedding dimension, $H$ is the horizon, $K$ is the number of
  actions, $\eta$ and $\gamma$ parameterize reachability and margin
  assumptions, and $\veps$ is the accuracy. Dependence on function
  classes and logarithmic factors are suppressed. Oracle and
  realizability assumptions vary. Block MDP algorithms discover a
  \emph{one-hot} representation to discrete latent states.
Bellman/Witness rank approaches can take a class $\Phi$ of embedding functions and search
  over simple policies or value functions composed with $\Phi$ (see~\pref{sec:complexity} and \pref{app:structural_2} for
  details).}
\label{tbl:comparison}
\ifthenelse{\equal{\version}{arxiv}}{\vspace{-0.25cm}}{\vspace{-0.75cm}}
\end{table}
\vspace{-0.5cm}
\paragraph{Our contributions.} We address the question of learning the representation $\phi$ in a low
rank MDP. To this end our contributions are both structural and
algorithmic.
\begin{enumerate}[leftmargin=*,topsep=0pt]
 \setlength{\parskip}{0pt}
\setlength{\parsep}{1pt}
\setlength{\itemsep}{1pt}
\item \textbf{Expressiveness of low rank MDPs.} Our algorithmic
  development leverages a re-formulation of the low
  rank dynamics in terms of an equally expressive, but more
  interpretable latent variable model. We provide several
  structural results for low rank MDPs, relating it to other models
  studied in prior work on representation learning for RL. In
  particular, we show that low rank MDPs are significantly more
  expressive than the block MDP
  model~\citep{du2019provably,misra2019kinematic}.
\item \textbf{Feature learning.} We develop a new algorithm, called
  \ouralg for ``Feature learning and model based exploration'', that
  learns a representation for low rank MDPs.  We prove that under
  realizability assumptions, \ouralg learns a \emph{uniformly
    accurate} model of the environment as well as a feature map that
  enables the use of linear methods for RL, in a statistically and
  computationally efficient manner. These guarantees enable downstream
  reward maximization, for \emph{any} reward function, with no
  additional data collection. 
\end{enumerate}
Our results and techniques provide new insights on representation
learning for RL and also significantly increase the scope for provably
efficient RL with rich observations (see~\pref{tbl:comparison}).
 \section{Low Rank MDPs}
\label{sec:model}

We consider an episodic Markov decision process $\Mcal$ with episode
length $H \in \NN$, state space $\Xcal$, and a finite action space
$\Acal = \{1,\ldots,K\}$. In each episode, a trajectory $\tau =
(x_0,a_0,x_1,a_1,\ldots,x_{H-1},a_{H-1},x_H)$ is generated, where (a)
$x_0$ is a starting state, and (b) $x_{h+1} \sim T_{h}(\cdot \mid
x_{h},a_{h})$, and (c) all actions $a_{0:H-1}$ are chosen by the
agent. We assume the starting state is fixed and that there is only
one available action at time $0$.\footnote{This easily accommodates
  the standard formulation with a non-degenerate initial distribution
  by defining $T_0(\cdot \mid x_0,a_0)$ to be the initial
  distribution. This setup is notationally more convenient, since we do not need special notation for the starting distribution.}
The operators $T_h: \Xcal \times \Acal \to \Delta(\Xcal)$ denote the
(non-stationary) transition dynamics for each time step.

As is standard in the literature, a policy $\pi:\Xcal \to
\Delta(\Acal)$ is a (randomized) mapping from states to actions. We
use the notation $\EE\sbr{ \cdot \mid \pi,\Mcal}$ to denote
expectations over states and actions observed when executing policy
$\pi$ in MDP $\Mcal$. We abuse notation slightly and use $[H]$ to
denote $\{0,\ldots,H-1\}$.

\begin{definition}
\label{def:lowrank}
An operator $T: \Xcal \times\Acal \to \Delta(\Xcal)$ admits a
\emph{low rank decomposition} with dimension $d \in \NN$ if there
exists two embedding functions $\phi^\star: \Xcal \times \Acal \to
\RR^d$ and $\mu^\star: \Xcal \to \RR^d$ such that
\begin{align*}
\textstyle\forall x,x' \in \Xcal, a \in \Acal: T(x' \mid x,a) = \inner{\phi^\star(x,a)}{\mu^\star(x')}.
\end{align*}
For normalization,\footnote{See the proof
  of Lemma B.1 in~\citet{jin2019provably} for this form of the normalization
  assumption.} we assume that $\nbr{\phi^\star(x,a)}_2 \leq 1$ for
all $x,a$ and for any function $g:\Xcal \to [0,1]$,
$\nbr{\int\mu^\star(x)g(x)dx}_2 \leq \sqrt{d}$. An MDP $\Mcal$ is a \emph{low rank MDP} if for each $h
\in [H]$, $T_h$ admits a low rank decomposition with
dimension $d$. We use $\phi_h^\star,\mu_h^\star$ to denote the embeddings for $T_h$.
\end{definition}

Throughout we assume that $\Mcal$ is a low rank MDP with dimension
$d$. Note that the condition on $\mu^\star$ ensures that the Bellman
backup operator is well-behaved.

\vspace{-0.1cm}
\paragraph{Function approximation for representation learning.}
We consider state spaces $\Xcal$ that are arbitrarily large, so that
some form of function approximation is necessary to generalize across
states.  For representation learning, it is natural to grant the agent
access to two function classes $\Phi \subset \Xcal \times \Acal \to
\RR^d$ and $\Upsilon \subset \Xcal \to \RR^d$ of candidate embeddings,
which we can use to identify the true embeddings
$(\phi^\star,\mu^\star)$.  To facilitate this model selection task, we
posit a \emph{realizability} assumption.
\begin{assum}[Realizability]
\label{assum:realizability}
We assume that for each $h \in [H]$: $\phi_h^\star \in \Phi$ and
$\mu_h^\star \in \Upsilon$.
\end{assum}
We desire sample complexity bounds that scale logarithmically
with the cardinality of the classes $\Phi$ and $\Upsilon$, which
we assume to be finite. 
Extensions that permit infinite classes with bounded statistical complexity (e.g., VC-classes) are not difficult.

In~\pref{app:structural}, we show that the low rank assumption alone, 
without~\pref{assum:realizability}, is not sufficient for obtaining
performance guarantees that are independent of the size of the state
space.  Hence, additional modeling assumptions are required, and we encode these in
$\Phi,\Upsilon$.

\vspace{-0.1cm}
\paragraph{Learning goal.}
We focus on the problem of reward-free
exploration~\citep{hazan2018provably,jin2020reward}, where the agent
interacts with the environment with no reward signal. When considering
model-based algorithms, a natural reward-free goal is \emph{system
  identification}: given function classes $\Phi,\Upsilon$, the
algorithm should learn a model $\Mhat \defeq
(\hat{\phi}_{0:H-1},\hat{\mu}_{0:H-1})$
that uniformly approximates the environment
$\Mcal$. We formalize this with the following performance criteria:
\begin{align}
\textstyle\forall \pi, h \in [H]: \EE\sbr{ \nbr{\big\langle\hat{\phi}_h(x_h,a_h)\hat{\mu}_h(\cdot)\big\rangle- T_h(\cdot \mid x_h,a_h)}_{\tv} \mid \pi,\Mcal} \leq \varepsilon.\label{eq:sys_id}
\end{align}
Here, we ask that our model accurately approximates the one-step
dynamics from the state-action distribution induced by following
\emph{any} policy $\pi$ for $h$ steps in the real environment.

System identification also implies a quantitative guarantee on the
learned representation $\hat{\phi}_{0:H-1}$: we can approximate the
Bellman backup of any value function on any data-distribution.
\begin{lemma}
\label{lem:representation}
If $\Mhat = (\hat{\phi}_{0:H-1},\hat{\mu}_{0:H-1})$ satisfies~\pref{eq:sys_id}, then
\begin{align*}
\textstyle \forall h \in [H], V: \Xcal \to [0,1], \exists \theta_h: \max_\pi \EE\sbr{ \abr{\big\langle\theta_h,\hat{\phi}_h(x_h,a_h)\big\rangle - \EE\sbr{V(x_{h+1}) \mid x_h,a_h}} \mid \pi,\Mcal} \leq \varepsilon.
\vspace{-0.1cm}
\end{align*}
\end{lemma}
Thus, linear function approximation using our learned features
suffices to fit the $Q$ function associated with any policy and
explicitly given reward.\footnote{Formally, we append the immediate
reward to the features.} The guarantee also enables dynamic
programming techniques for policy optimization.  In other
words,~\pref{eq:sys_id} verifies that we have found a good
representation, in a quantitative sense, and enables tractable reward
maximization for any known reward function.

\ifthenelse{\equal{\version}{arxiv}}{\vspace{-0.18cm}}{\vspace{-0.1cm}}
\section{Related work}
\label{sec:related}
\ifthenelse{\equal{\version}{arxiv}}{\vspace{-0.18cm}}{\vspace{-0.1cm}}

Low rank models are prevalent in dynamics and
controls~\citep{thon2015links, littman2002predictive,
  singh2004predictive}. The low rank MDP in particular has been
studied in several works in the context of
planning~\citep{barreto2011reinforcement,barreto2011computing}, estimation~\citep{duan2020adaptive}, and in
the generative model
setting~\citep{yang2019sample}.
Regarding nomenclature, to our knowledge the name \emph{low rank
  MDP} appears first in~\citet{jiang2017contextual},
although~\citet{rendle2010factorizing} refer to it as
\emph{factorizing MDP},~\citet{yao2014pseudo} call it a factored linear MDP, and \citet{barreto2011reinforcement} refer to
a similar model as \emph{stochastic factorization}. More recently,
it has been called the \emph{linear MDP}
by~\citet{jin2019provably}. We use \emph{low rank MDP} because it
highlights the key structural property of the dynamics, and because we
study the setting where the embeddings are unknown, which necessitates
non-linear function approximation.

Turning to reinforcement learning with function approximation and
exploration, a large body of effort focuses on (essentially) linear
methods~\citep{yang2019reinforcement,jin2019provably,cai2019provably,modi2019sample,du2019dsec,wang2019optimism}. Closest
to our work are the results of \citet{jin2019provably}
and~\citet{yang2019reinforcement}, who consider low rank MDPs with
known feature maps $\phi_{0:H-1}^\star$ (\citet{yang2019reinforcement}
also assumes that $\mu^\star_{0:H-1}$ is known up to a linear
map).  These results are encouraging and
motivate our representation learning formulation, but, on their own,
these methods cannot leverage the inductive biases provided
by neural networks to scale to rich state spaces.

There are methods for more general, non-linear, function
approximation, but these works either (a) require strong environment assumptions
such as
determinism~\citep{wen2013efficient,du2020agnostic}, (b) require
strong function class assumptions such as bounded Eluder
dimension~\citep{russo2013eluder,osband2014model}, (c) have sample complexity scaling linearly with the function class size~\citep{lattimore2013sample,ortner2014selecting} 
or (d) are
computationally intractable~\citep{jiang2017contextual,sun2018model,dong2019sqrt}.
Note that~\citet{ortner2014selecting,jiang2015abstraction} consider
a form of representation learning, abstraction selection, but
the former scales linearly with the number of candidate abstractions,
while the latter does not address exploration.

\textbf{Bellman/Witness rank.}
We briefly expand on this final category of computationally
inefficient methods.  For model-free reinforcement
learning,~\citet{jiang2017contextual} give an algebraic condition, in
terms of a notion called the Bellman rank, on the environment and a
given function approximation class, under which sample efficient
reinforcement learning is always possible.~\citet{sun2018model} extend
the definition to model-based approaches, with the notion of Witness
rank. As we will see in the next section, the low rank MDP with a
function class derived from $\Phi$ (and $\Upsilon$) admits low Bellman
(resp., Witness) rank, and so these results imply that our setting is
statistically tractable.

\textbf{Block MDPs.}
Finally, we turn to theoretical works on representation learning for
RL.~\citet{du2019provably} introduce the \emph{block MDP} model, in
which there is a finite latent state space $\Scal$ that governs the
transition dynamics, and each ``observation'' $x \in \Xcal$ is
associated with a latent state $s \in \Scal$, so the state
is \emph{decodable}. The natural representation learning goal is to
recover the latent states,
and~\citet{du2019provably,misra2019kinematic} show that this can be
done, in concert with exploration, in a statistically and
computationally efficient manner.  Since the block MDP can be easily
expressed as a low rank MDP, our results can be specialized to this
setting, where they yield comparable guarantees. On the other hand, we
will see that the low rank MDP is significantly more expressive, and
so our results greatly expand the scope for provably efficient
representation learning and reinforcement learning.

\section{Expressiveness of low rank MDPs}
\label{sec:complexity}

Before turning to our algorithmic development, we discuss connections
between low-rank MDPs and models studied in prior work. This
discussion is facilitated by formalizing a connection between MDP
transition operators and latent variable graphical models.

\begin{definition}
\label{def:latent_var}
The \emph{latent variable representation} of a transition operator $T:
\Xcal \times \Acal \to \Delta(\Xcal)$ is a latent space $\Zcal$ along
with functions $\psi:\Xcal\times\Acal \to \Delta(\Zcal)$ and
$\nu:\Zcal \to \Delta(\Xcal)$, such that $T(\cdot \mid x,a) = \int
\nu(\cdot \mid z) \psi(z \mid x,a)dz$. The \emph{latent variable
  dimension} of $T$, denoted $\dlv$ is the cardinality of smallest
latent space $\Zcal$ for which $T$ admits a latent variable
representation.
\end{definition}
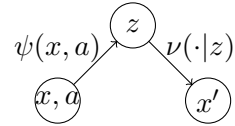
\begin{wrapfigure}{r}{0.25\textwidth}
\vspace{-0.5cm}
\begin{center}
\begin{tikzpicture}
\draw (0,0) circle [radius=0.3] node {$x,a$};
\draw (1,1) circle [radius=0.3] node {$z$};
\draw (2,0) circle [radius=0.3] node {$x'$};

\draw[->] (0.21,0.21) -- (0.79,0.79);
\draw[->] (1.21,0.79) -- (1.79,0.21);
\draw (0,0.7) node {$\psi(x,a)$};
\draw (1.9,0.7) node {$\nu(\cdot|z)$};
\end{tikzpicture}
 \end{center}
\caption{The latent variable interpretation.}
\vspace{-0.25cm}
\label{fig:lvm}
\end{wrapfigure}

See~\pref{fig:lvm}. In this representation, (1) each $(x,a)$ pair
induces a ``posterior'' distribution $\psi(x,a) \in \Delta(\Zcal)$
over $z$, (2) we sample $z \sim \psi(x,a)$, and (3) then sample $x'
\sim \nu(\cdot \mid z)$, where $\nu$ specifies the ``emission''
distributions. As notation, we typically write $\nu(x) \in
\RR^{\Zcal}$ with coordinates $\nu(x)[z] = \nu(x\mid z)$ and we call
$\psi,\nu$ the \emph{simplex features}, following the example
described by~\citet{jin2019provably}. When considering $H$-step MDPs,
this representation allows us to augment the trajectory $\tau$ with
the latent variables $\tau = (x_0,a_0,z_1,x_1,\ldots
z_{H-1},x_{H-1},a_{H-1},x_H)$. Here note that $z_h$ is the latent
variable that generates $x_h$.

Note that all transition operators admit a trivial latent variable
representation, as we may always take $\psi(x,a) = T(\cdot \mid x,a)$.
However, when $T$ is endowed with additional structure, the latent
variable representations are more interesting. For example, this
viewpoint already certifies a factorization $T(x' \mid x,a) =
\inner{\psi(x,a)}{\nu(x')}$ with embedding dimension $|\Zcal|$, and so
$\dlv$ (if it is finite) is an upper bound on the rank of the
transition operator. On the other hand, compared with~\pref{def:lowrank},
this factorization additionally requires that $\psi(x,a)$
and $\nu(\cdot \mid z)$ are probability distributions. Since
the factorization is non-negative,  $\dlv$ is the
\emph{non-negative rank} of the transition operator.

The latent variable representation enables a natural comparison of the
expressiveness of various models, and, as we will see in the next
section, yields insights that facilitate algorithm design. We now
examine models that have been introduced in prior works
and their properties relative to~\pref{def:lowrank}.

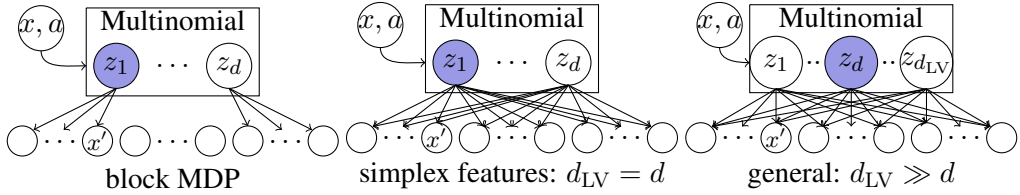
\begin{figure}[t]
\begin{center}
\definecolor{lblue}{rgb}{0.6,0.6,0.9}
\begin{tikzpicture}
\draw (0,-1.5) node {block MDP};
\draw(-1.75,0.5) circle [radius=0.3] node {$x,a$};
\draw[->] (-1.75,0.2) to [out=-90,in=180] (-1.15,0.0);
\draw (0, 0.6) node {Multinomial};
\draw (-1.15,0.77) -- (1.15,0.77) -- (1.15,-0.4) -- (-1.15,-0.4) -- (-1.15,0.77);
\draw[fill=lblue] (-0.75,0) circle [radius=0.3] node {$z_1$};
\draw (0,0) node {$\ldots$};
\draw (0.75,0) circle [radius=0.3] node {$z_d$};
\draw[->](-0.75,-0.3) -- (-1.85,-0.85);
\draw (-2,-1) circle [radius=0.2];
\draw[->](-0.75,-0.3) -- (-1.45,-0.85);
\draw (-1.5,-1) node {$\ldots$};
\draw[->](-0.75,-0.3) -- (-1,-0.8);
\draw (-1,-1) circle [radius=0.2] node {\footnotesize $x'$};
\draw (-0.5,-1) circle [radius=0.2];
\draw (0,-1) node {$\ldots$};
\draw (0.5,-1) circle [radius=0.2];
\draw[->](0.75,-0.3) -- (1,-0.8);
\draw (1,-1) circle [radius=0.2];
\draw[->](0.75,-0.3) -- (1.45,-0.85);
\draw (1.5,-1) node {$\ldots$};
\draw[->](0.75,-0.3) -- (1.85,-0.85);
\draw (2,-1) circle [radius=0.2];
\end{tikzpicture}
 \begin{tikzpicture}
\draw (0,-1.5) node {simplex features: $\dlv = d$};
\draw(-1.75,0.5) circle [radius=0.3] node {$x,a$};
\draw[->] (-1.75,0.2) to [out=-90,in=180] (-1.15,0.0);
\draw (0, 0.6) node {Multinomial};
\draw (-1.15,0.77) -- (1.15,0.77) -- (1.15,-0.4) -- (-1.15,-0.4) -- (-1.15,0.77);
\draw[fill=lblue] (-.75,0) circle [radius=0.3] node {$z_1$};
\draw (0,0) node {$\ldots$};
\draw (0.75,0) circle [radius=0.3] node {$z_d$};
\draw[->](-0.75,-0.3) -- (-1.85,-0.85);
\draw[->](0.75,-0.3) -- (-1.85,-0.85);
\draw (-2,-1) circle [radius=0.2];
\draw[->](-0.75,-0.3) -- (-1.45,-0.85);
\draw[->](0.75,-0.3) -- (-1.45,-0.85);
\draw (-1.5,-1) node {$\ldots$};
\draw[->](-0.75,-0.3) -- (-1,-0.8);
\draw[->](0.75,-0.3) -- (-1,-0.8);
\draw (-1,-1) circle [radius=0.2] node {\footnotesize $x'$};
\draw[->](-0.75,-0.3) -- (-0.38,-0.85);
\draw[->](0.75,-0.3) -- (-0.38,-0.85);
\draw (-0.5,-1) circle [radius=0.2];
\draw[->](-0.75,-0.3) -- (0,-0.85);
\draw[->](0.75,-0.3) -- (0,-0.85);
\draw (0,-1) node {$\ldots$};
\draw[->](-0.75,-0.3) -- (0.38,-0.85);
\draw[->](0.75,-0.3) -- (0.38,-0.85);
\draw (0.5,-1) circle [radius=0.2];
\draw[->](-0.75,-0.3) -- (1,-0.8);
\draw[->](0.75,-0.3) -- (1,-0.8);
\draw (1,-1) circle [radius=0.2];
\draw[->](-0.75,-0.3) -- (1.45,-0.85);
\draw[->](0.75,-0.3) -- (1.45,-0.85);
\draw (1.5,-1) node {$\ldots$};
\draw[->](-0.75,-0.3) -- (1.85,-0.85);
\draw[->](0.75,-0.3) -- (1.85,-0.85);
\draw (2,-1) circle [radius=0.2];
\end{tikzpicture}
 \begin{tikzpicture}
\draw (0,-1.5) node {general: $\dlv \gg d$};
\draw(-1.75,0.5) circle [radius=0.3] node {$x,a$};
\draw[->] (-1.75,0.2) to [out=-90,in=180] (-1.35,0.0);
\draw (0, 0.6) node {Multinomial};
\draw (-1.35,0.77) -- (1.35,0.77) -- (1.35,-0.4) -- (-1.35,-0.4) -- (-1.35,0.77);
\draw (-1,0) circle [radius=0.35] node {$z_1$};
\draw (-.5,0) node {$..$};
\draw[fill=lblue] (0,0) circle [radius=0.35] node {$z_d$};
\draw (0.5,0) node {$..$};
\draw (1.,0) circle [radius=0.35] node {$z_{\dlv}$};
\draw[->](-1.,-0.35) -- (-1.85,-0.85);
\draw[->](0,-0.35) -- (-1.85,-0.85);
\draw[->](1.,-0.35) -- (-1.85,-0.85);
\draw (-2,-1) circle [radius=0.2];
\draw[->](-1.,-0.35) -- (-1.45,-0.85);
\draw (-1.5,-1) node {$\ldots$};
\draw[->](-1.,-0.35) -- (-1,-0.8);
\draw[->](0,-0.35) -- (-1,-0.8);
\draw[->](1.,-0.35) -- (-1,-0.8);
\draw (-1,-1) circle [radius=0.2] node {\footnotesize $x'$};
\draw[->](-1.,-0.35) -- (-0.38,-0.85);
\draw[->](0,-0.35) -- (-0.38,-0.85);
\draw[->](1.,-0.35) -- (-0.38,-0.85);
\draw (-0.5,-1) circle [radius=0.2];
\draw[->](0,-0.35) -- (0,-0.85);
\draw (0,-1) node {$\ldots$};
\draw[->](-1.,-0.35) -- (0.38,-0.85);
\draw[->](0,-0.35) -- (0.38,-0.85);
\draw[->](1.,-0.35) -- (0.38,-0.85);
\draw (0.5,-1) circle [radius=0.2];
\draw[->](-1.,-0.35) -- (1,-0.8);
\draw[->](0,-0.35) -- (1,-0.8);
\draw[->](1.,-0.35) -- (1,-0.8);
\draw (1,-1) circle [radius=0.2];
\draw[->](1.,-0.35) -- (1.45,-0.85);
\draw (1.5,-1) node {$\ldots$};
\draw[->](-1.,-0.35) -- (1.85,-0.85);
\draw[->](0,-0.35) -- (1.85,-0.85);
\draw[->](1.,-0.35) -- (1.85,-0.85);
\draw (2,-1) circle [radius=0.2];
\end{tikzpicture}
 \end{center}
\ifthenelse{\equal{\version}{arxiv}}{\vspace{-0.5cm}}{\vspace{-0.2cm}}
\caption{The latent variable interpretation of low rank MDPs, where
  $(x,a)$ induces a distribution over latent variable $z$. Left:
  in block MDPs, latent variables induce a partition over the next
  state $x'$. Center: simplex features have embedding dimension equal
  to the number of latent variables. Right: low rank
  MDPs can have exponentially more latent variables than the
  dimension, $\dlv \gg d$.}
\label{fig:separation}
\ifthenelse{\equal{\version}{arxiv}}{}{\vspace{-0.5cm}}
\end{figure}

\textbf{Block MDPs.}
A block MDP~\citep{du2019provably,misra2019kinematic} is clearly a
latent variable model with $\Zcal$ corresponding to the latent state
space $\Scal$ and the additional restriction that two latent variables
$z$ and $z'$ have disjoint supports in their respective emissions
$\nu(\cdot\mid z)$ and $\nu(\cdot\mid z')$ (see the left panel
of~\pref{fig:separation}). Therefore, a block MDP is a low rank MDP
with rank $d \leq |\Scal|$, but the next result shows that a low rank
MDP is significantly more expressive.

\begin{proposition}
\label{prop:block_separation}
For any $d \geq 2$ and any $M \in \NN$ there exists an environment on
$|\Xcal| = M$ states, that can be expressed as a low rank MDP with
embedding dimension $d$, but for which any block MDP representation
must have $M$ latent states.
\end{proposition}
In fact, the MDP that we construct for the proof, admits a latent
variable representation with $|\Zcal| = d$, but does not admit a
non-trivial block MDP representation. This separation exploits the
decodability restriction of block MDPs, which is indeed quite limiting
in terms of expressiveness.

\textbf{Simplex features.}
Given the latent variable representation and the fact that it
certifies a rank of at most $\dlv$, it is natural to ask if this
representation is canonical for all low rank MDPs. In other words, for
any transition operator with rank $d$, can we express it as a latent
variable model with $|\Zcal|=d$, or equivalently with simplex features
of dimension $d$?

As discussed above, this model is indeed more expressive than the
block MDP. However, the next result answers the above question
in the negative. The latent variable representation is
\emph{exponentially} weaker than the general low rank
representation in the following sense: 

\begin{proposition}
\label{prop:simplex_separation}
For any even $n \in \NN$, there exists an MDP that can be cast as a
low rank MDP with embedding dimension $O(n^2)$, but which has $\dlv
\geq 2^{\Omega(n)}$.
\end{proposition}

See the center and right panels
of~\pref{fig:separation}. The result is proved by recalling that the latent variable dimension
determines the non-negative rank of $T$, which can be much larger than
its rank~\citep{rothvoss2017matching,yannakakis1991expressing}. It
showcases how low rank MDPs are quite different from latent variable
models of comparable dimension and
demonstrates how embedding functions with negative values
can provide significant expressiveness.

\textbf{Bellman and Witness rank.}
As our last concrete connection,
we remark here that the low rank MDP with a function class
derived from $\Phi$ (and $\Upsilon$) admits low Bellman (resp.,
Witness) rank.

\begin{proposition}[Informal]
\label{prop:bellman_informal}
The low rank MDP model always has Bellman rank at most $d$.
Additionally, given $\Phi$ and assuming $\phi_{0:H-1}^\star\in \Phi$, we can
construct a function classes $(\Gcal,\Pi)$, so that \olive when run
with $(\Gcal,\Pi)$ has sample complexity $\otil\rbr{\poly(d, H, K,
  \log |\Phi|, \epsilon^{-1})}$.
\end{proposition}
See~\pref{prop:bellman} for a more precise statement. An analogous
result hold for the Witness rank notion of~\citet{sun2018model}
(see~\pref{prop:witness} in the appendix). Unfortunately both \olive,
and the algorithm of~\citet{sun2018model} are not computationally
tractable, as they involve enumeration of the employed function
class. We turn to the development of computationally tractable
algorithms in the next section.

\section{Main results}
\label{sec:results}

We now turn to the design of algorithms for representation learning
and exploration in low rank MDPs. As a computational abstraction, we
consider the following optimization and sampling oracles.

\begin{definition}[Computational oracles]
 Define the following oracles for the classes $\Phi, \Upsilon$:
\begin{enumerate}[leftmargin=*]
\item
The \emph{maximum likelihood oracle}, \mle, takes a dataset $D$ of $(x,a,x')$ triples, and returns
\begin{center}
$\mle(D) \defeq \argmax_{\phi\in\Phi,\mu \in \Upsilon} \sum_{(x,a,x') \in D} \log( \inner{\phi(x,a)}{\mu(x')} )$.
\end{center}
\item The \emph{sampling oracle}, \samp, is a subroutine which, for any $(\phi,\mu) \in \Phi\times\Upsilon$ and any $(x,a)$, returns a sample $x'\sim \inner{\phi(x,a)}{\mu(\cdot)}$. Multiple calls to the procedure result in independent samples.
\end{enumerate}
\label{def:oracles}
\end{definition}

We assume access to both oracles as a means towards practical
algorithms that avoid explicitly enumerating over all functions in $\Phi$ and $\Upsilon$.
Note that related assumptions are quite common in the
literature~\citep{misra2019kinematic,du2019provably,agarwal2014taming},
and in practice, both oracles can be reasonably approximated whenever
optimizing over $\Phi,\Upsilon$ is feasible (e.g., neural
networks). Regarding \mle, other optimization oracles are possible,
and in the appendix (\pref{rem:fgan}) we sketch how our proof can
accommodate a generative adversarial oracle as a
replacement~\citep{goodfellow2014generative,arora2017generalization}.
While the sampling oracle is less standard, one might implement \samp
via optimization methods like the Langevin
dynamics~\citep{welling2011bayesian} or through reparametrization
techniques such as the Gumbel-softmax
trick~\citep{jang2017categorical, figurnov2018implicit}.\footnote{We
do not explicitly consider approximate oracles, but additive
approximations can be accommodated in our proof. In particular,
if \samp returns a sample from a distribution that is
$\epssamp$ close in total variation to the target distribution in
$\poly(1/\epssamp)$ time, then we retain computational efficiency.}
\ifthenelse{\equal{\version}{arxiv}}{
In addition, the sampling oracle can be avoided at the cost of
additional real world experience, an approach we describe formally
in~\pref{thm:real_world_main} below. }{}

\subsection{Algorithm description}
The algorithm is called \ouralg, for ``Feature Learning And
Model-Based Exploration.'' Pseudocode is displayed
in~\pref{alg:main}. \ouralg is iterative in nature, where in iteration
$j$, we use an exploratory policy $\rho_{j-1}$ to collect a dataset of
transitions, and then we pass all previously collected transitions to
the MLE oracle.
The optimization oracle returns embedding functions
$(\hat{\phi}_h,\hat{\mu}_h)$ for each $h$ which define transition
operators $\hat{T}_h$ for our updated learned model
$\Mhat$. Then \ouralg calls a planning sub-routine to compute the
exploratory policy $\rho_{j}$ that we use in the next iteration. After $J_{\max}$ iterations, we simply output our current model $\Mhat$.

For the planning step, intuitively we seek an exploratory policy
$\rho$ that induces good coverage over the state space when executed
in the model. We do this by solving one planning problem per time step $h$
in~\pref{alg:general_planner} using a technique inspired by elliptical
potential arguments from linear bandits~\citep{dani2008stochastic}.
Using the $h$-step model $\hat{T}_{0:h-1}$, we iteratively maximize
certain quadratic forms of our learned features $\hat{\phi}_{h-1}$ to
find new directions not covered by the previously discovered policies,
and we update the exploratory policy to include the maximizer.
The planning algorithm terminates when no policy can achieve large quadratic form, which
implies that we have found all reachable directions in
$\hat{\phi}_{h-1}$. 
This yields a \emph{mixture policy} $\rhopre$ that is executed by
sampling one of the mixture components and executing that policy for
the entire episode. The component policies are linear in the learned
features $\hat{\phi}_{0:h-1}$. The challenge in our analysis is to
relate this coverage in the model to that in the true environment as
we discuss in the next section.

\pref{alg:general_planner} is a model-based planner, so it requires no interaction with the environment.
The main computational step is the optimization
problem~\pref{eq:planning_optimization}, which can be solved
efficiently with access to the sampling oracle, essentially by running
the algorithm of~\citet{jin2019provably}
(See~\pref{lem:sampling_planner} in the appendix). Note that we are
optimizing over \emph{all policies}, which is possible because the
Bellman backups in a low rank MDP are linear functions of the features
(c.f.,~\pref{lem:representation}). The sampling oracle can also be
used to approximate all expectations, and, with sufficient accuracy,
this has no bearing on the final results. Our proofs do account for
the sampling errors.

\setlength{\textfloatsep}{4pt plus 0.0pt minus 2.0pt}\setlength{\floatsep}{4pt plus 0.0pt minus 2.0pt}
\begin{algorithm}[t]
\begin{algorithmic}
\STATE \textbf{Input:} Environment $\Mcal$, function classes $\Phi, \Upsilon$, subroutines \mle and \samp, parameters $\beta, n$.
\STATE Set $\rho_0$ to be the random policy, which takes all actions uniformly at random. 
\STATE Set $D_h = \emptyset$ for each $h \in \{0,\ldots,H-1\}$. 
\FOR{$j=1,\ldots,J_{\max}$}
\FOR{$h=0,\ldots,H-1$}
\STATE Collect $n$ samples $(x_h,a_h,x_{h+1})$ by rolling into $x_h$ with $\rho_{j-1}$ and taking $a_h \sim \unif(\Acal)$. 
\STATE Add these samples to $D_h$.
\STATE Solve maximum likelihood problem: $(\hat{\phi}_h,\hat{\mu}_h) \gets \mle(D_{h})$.
\STATE Set $\hat{T}_h(x_{h+1} \mid x_h,a_h) = \inner{\hat{\phi}_h(x_h,a_h)}{\hat{\mu}_h(x_{h+1})}$.
\ENDFOR
\STATE For each $h$, call planner (\pref{alg:general_planner}) with $h$ step model $\hat{T}_{0:h-1}$ and $\beta$ to obtain $\rhopre$.
\STATE Set $\rho_j = \unif(\{\rhopre \circ \random \}_{h=0}^{H-1})$, to be uniform over the discovered $h$-step policies, augmented with random actions. 
\ENDFOR
\end{algorithmic}
\caption{\ouralg: Feature Learning And Model-Based Exploration}
\label{alg:main}
\end{algorithm}

\begin{algorithm}[t]
\begin{algorithmic}
\STATE \textbf{Input:} MDP $\Mtil =
(\phi_{0:\tilde{h}}, \mu_{0:\tilde{h}})$, subroutine \samp, parameter $\beta > 0$. Initialize $\Sigma_0 = I_{d \times d}$.
\FOR{$t = 1,2,\ldots,$}
\STATE Compute (see text for details)
\vspace{-0.2cm}
\begin{align}
\pi_t = \argmax_\pi \EE\sbr{ \phiht^\top \Sigma_{t-1}^{-1} \phiht \mid \pi,\Mtil}. \label{eq:planning_optimization}
\end{align}
\vspace{-0.2cm}
\STATE If the objective is at most $\beta$, halt and output $\rho = \unif(\{\pi_\tau\}_{\tau < t})$.
\STATE Compute $\Sigma_{\pi_t} = \EE\sbr{ \phiht\phiht^\top \mid \pi,\Mtil}$. Update $\Sigma_t \gets \Sigma_{t-1} + \Sigma_{\pi_t}$.
\ENDFOR
\end{algorithmic}
\caption{Elliptical planner}
\label{alg:general_planner}
\end{algorithm}

\subsection{Theoretical Results}

\setlength{\textfloatsep}{20.0pt plus 2.0pt minus 4.0pt}
\setlength{\floatsep}{12.0pt plus 2.0pt minus 2.0pt}

We now state the main guarantee.
\begin{theorem}
\label{thm:main}
Fix $\delta \in (0,1)$. If $\Mcal$ is a low rank MDP with dimension
$d$ and horizon $H$ and~\pref{assum:realizability} holds, then \ouralg
with subroutine~\pref{alg:general_planner} and appropriate
settings\footnote{The precise settings for $\beta,J_{\max}$, and $n$
are given in the appendix. } of $\beta, J_{\max}$, and $n$,
computes a model $\Mhat$ such that~\pref{eq:sys_id} holds with
probability at least $1-\delta$. The total number of trajectories
collected is 
\begin{align*}
\tilde{\order}\rbr{ \frac{H^{22}K^9d^7\log(|\Phi||\Upsilon|/\delta)}{\veps^{10}}},
\end{align*}
and the algorithm runs in polynomial time with
polynomially many calls to \mle and \samp (\pref{def:oracles}).
\end{theorem}

Thus, \ouralg provably learns low rank MDP models in a statistically
and computationally efficient manner, under~\pref{assum:realizability}
and~\pref{assum:reachability}.  
While the result
is comparable to prior work in the dependencies on $d$, $H$, $K$ and
$\varepsilon$, we instead highlight the more conceptual advances over
prior work. \begin{itemize}[leftmargin=*]
\item The key advancement over the block MDP
  algorithms~\citep{du2019provably,misra2019kinematic} is that \ouralg
  applies to a significantly richer class of models with comparable
  function approximation assumptions. A secondary, but important,
  improvement is that \ouralg does not require any reachability
  assumptions, unlike these previous results.
We remark that~\citet{feng2020provably} avoid reachability
  restrictions in block MDPs, but their function approximation/oracle
  assumptions are much stronger than ours.
\item Over~\citet{jin2019provably,yang2019reinforcement}, the key advancement is that we
  address the representation learning setting where the embeddings
  $\phi_{0:H-1}^\star$ are not known a priori. On the other hand, our
  bound scales 
  with the number of actions $K$. We believe that additional structural
  assumptions on $\Phi$ are required to avoid the dependence on $K$ in the representation learning setting.
    \item Over~\citet{jiang2017contextual,sun2018model}, the key
  advancement is computational efficiency. However, the low rank MDP
  is less general than what is covered by their theory, and our sample
  complexity is worse in the polynomial factors. 
\end{itemize}
As remarked earlier, the logarithmic dependence
on the sizes of $\Phi, \Upsilon$ can be relaxed to alternative notions
of capacity for continuous classes.

We also state a sharper bound for a version of \ouralg that operates
directly on the simplex factorization.  The main difference is that we
use a conceptually simpler planner (See~\pref{alg:simplex_planner} in
the appendix) and the sample complexity bound scales with $\dlv$.

\begin{theorem}
\label{thm:improved}
Fix $\delta \in (0,1)$. If $\Mcal$ admits a simplex factorization with
embedding dimension $\dlv$,~\pref{assum:realizability} holds, and all $\phi\in\Phi$ satisfy
$\phi(x,a)\in \Delta([\dlv])$, then \ouralg
with~\pref{alg:simplex_planner} as the subroutine and appropriate
setting\footnote{This version does not require the parameter $\beta$.}
of $J_{\max}$ and $n$ computes a model $\Mhat$ such that~\pref{eq:sys_id} holds with
probability at least $1-\delta$. The total number of trajectories
collected is
\begin{align*}
\tilde{\order}\rbr{\frac{H^{11}K^5\dlv^5\log(|\Phi||\Upsilon|/\delta)}{\veps^3}}.
\end{align*}
The algorithm runs in polynomial time with polynomially many calls to \mle and \samp (\pref{def:oracles}).
\end{theorem}
This bound scales much more favorably with $H,K$ and $\veps$, but
incurs a polynomial dependence on the latent variable dimension
$\dlv$, instead of the embedding dimension $d$. For many problems,
including block MDPs, we expect that $\dlv \approx d$, in which case
using this version of \ouralg may be
preferable. However,~\pref{thm:improved} requires that we encode
simplex constraints into our function class $\Phi$, for example using
the softmax. When $\dlv$ is small, this may be a practically useful
design choice.

\ifthenelse{\equal{\version}{arxiv}}{
\paragraph{Planning in the real world. }
As our final result, we consider replacing the model-based planning
subroutine with one that collects trajectories from the
environment. This allows us to avoid using the sampling oracle, but
our analysis requires one additional assumption. Recall the latent
variable representation of~\pref{def:latent_var} and the fact that we
can augment the trajectories with the latent variables, and let
$\Zcal_h$ denote the latent state space for $T_h$, i.e., the values
that $z_{h+1}$ can take.  Our
\emph{reachability} assumption posits that for the MDP $\Mcal$, the
latent variables can be reached with non-trivial probability.

\begin{assum}[Reachability]
\label{assum:reachability}
There exists a constant $\etamin > 0$ such that
\begin{align*}
\forall h \in \{0,\ldots,H-1\}, z \in \Zcal_h:~~ \max_\pi \PP\sbr{z_{h+1} = z \mid \pi,\Mcal} \geq \etamin.
\end{align*}
\end{assum}

The assumption generalizes prior reachability assumptions in block
MDPs~\citep{du2019provably, misra2019kinematic}, where the latent
variables are referred to as ``latent states.''  Note that
reachability does not eliminate the exploration problem, as a random
walk may still visit a latent variable with exponentially small
probability.
However, reachability is a limitation that unfortunately imposes an
upper bound on the latent variable dimension $\dlv$, as formalized in
the next proposition.
\begin{proposition}
\label{prop:reachability}
If MDP $\Mcal$ has rank $d$ and satisfies~\pref{assum:reachability},
then for each $h$, the latent variable dimension of $T_h$ satisfies
$\dlv \leq \nicefrac{dK^2}{\etamin^2}$.
\end{proposition}

With this new assumption, we have the following theorem, which is proved in~\pref{app:real_world}.
\begin{theorem}
\label{thm:real_world_main}
Fix $\delta \in(0,1)$. In the setup of~\pref{thm:improved} with~\pref{assum:reachability}, a version
of \ouralg (described in~\pref{app:real_world}) computes a model
$\Mhat$ such that~\pref{eq:sys_id} holds with probability
$1-\delta$. The algorithm collects
$\poly(\dlv,H,K,\nicefrac{1}{\etamin},\nicefrac{1}{\varepsilon},\log(|\Phi||\Upsilon|/\delta))$
trajectories and runs in polynomial time with $H$ calls to \mle.
\end{theorem}

Importantly, this instantiation of \ouralg does not require that the
function classes support efficient sampling, i.e., we do not
use \samp. On the other hand, the sample complexity degrades in
comparison with our results using model-based planners. We also note
that the result considers simplex representations as
in~\pref{thm:improved} because (a) the calculations are considerably
simpler, and (b) in light of~\pref{prop:reachability} handling general
representations under~\pref{assum:reachability} only incurs a
polynomial overhead. We believe that extending the result to
accommodate general representations directly is possible.  }{}

\paragraph{Challenges in the analysis.}
We highlight three main challenges in the analysis.  The first
challenge arises in the analysis of the model learning step, where we want to show
that our model $\hat{T}_h$, learned by maximum likelihood estimation,
accurately approximates the true dynamics $T_h$ on the data
distributions induced by the exploratory policies $\rho_{0:j-1}$. This
requires a generalization argument, but both the empirical MLE
objective and population version --- the KL divergence --- are
unbounded, so we cannot use standard uniform convergence
techniques. Instead, we employ (and slightly adapt) results from the
statistics literature~\citep{geer2000empirical,zhang2006from} to show
that MLE yields convergence in the Hellinger and total variation
distances. While these arguments are well-known in the statistics
community, we highlight them here because we believe they may be
broadly useful in the context of model-based RL.

The second challenge is to transfer the model error guarantee from the
exploratory distributions to distributions induced by other
policies. Intuitively, error transfer should be possible if the
exploratory policies cover the state space, but we must determine how
we measure and track coverage. Leveraging the low rank dynamics, we
show that if a policy $\pi$ induces a distribution over true features
$\phi^\star_{h-1}$ that is in the span of the directions visited by
the exploratory policies, then we can transfer the MLE error guarantee
for $\hat{T}_h$ to $\pi$'s distribution. This suggests measuring
coverage for time $h$ in terms of the second moment matrix of the true
features at the previous time induced by the exploratory policies
$\rho$, that is $\Sigma_{h-1}
= \EE_{\rho}\phi_{h-1}^\star \phi_{h-1}^{\star,\top}$. Using these
matrices, we prove a sharp \emph{simulation lemma} that bounds the
model error observed by a policy $\pi$ in terms of the model error on
the exploratory distribution and the probability that $\pi$ visits
features for which
$\phi_{h-1}^{\star,\top}\Sigma_{h-1}^{-1} \phi_{h-1}^\star$ is large.

The simulation lemma suggests a planning strategy and an
``explore-or-terminate'' argument: the next exploratory policy should
maximize $\phi_{h-1}^{\star,\top}\Sigma_{h-1}^{-1} \phi_{h-1}^\star$,
in which case either we visit some new direction and make progress, or
we certify~\pref{eq:sys_id}, since no policy can make this quantity
large. However, we cannot maximize this objective directly, since we
do not know $\phi_{h-1}^\star$ or $\Sigma_{h-1}$! Moreover, we cannot
use $\hat{\phi}_{h-1}$ to approximate $\phi^\star_{h-1}$ in the
objective, since even if the model is accurate, the features may not
be. Instead we plan to find a policy that induces a well-conditioned
covariance matrix in terms of $\hat{\phi}_{h-2}$, the learned features
at time $h-2$. By composing this policy with a random action and
applying our simulation lemma, we can show that this policy either
explores at some previous time or approximately maximizes
$\phi_{h-1}^{\star,\top}\Sigma_{h-1}^{-1} \phi_{h-1}^\star$, which
allows us to apply the explore-or-terminate argument. Combining this
reasoning with an elliptical potential argument, we can bound the
number of iterations in which exploration can happen, which leads to
the final result.

\section{Discussion}
\label{sec:discussion}
This paper studies representation learning and exploration for low
rank MDPs. We provide an intuitive interpretation of these models in
terms of a latent variable representation, and we prove a number of
structural results certifying that low rank MDPs are significantly
more expressive than models studied in prior work. We also develop
\ouralg, a computationally and statistically efficient model based
algorithm for system identification in low rank MDPs. Policy
optimization follows as a corollary.

Our results raise a number of promising directions for future work. On
the theoretical side, 
can we develop provably efficient model-free algorithms for representation
learning in the low rank MDP? On the empirical side, can we leverage
the algorithmic insights of~\ouralg to develop practically effective
representation learning algorithms for complex reinforcement learning
tasks? We look forward to answering these questions in future work.

\subsection*{Acknowledgements}
We thank Ruosong Wang for insightful discussions regarding the
reachability assumption.

\appendix
\section{Proofs for the structural results}
\label{app:structural}

In this appendix we provide proofs for the structural results in the
paper. We first provide the proof
of~\pref{lem:representation}. In~\pref{app:structural_1} we focus on
results relating to realizability and reachability. Then
in~\pref{app:structural_2} we turn to the separation results
of~\pref{prop:block_separation}
and~\pref{prop:simplex_separation}. Finally
in~\pref{app:structural_3} we provide details about the connection to
the Bellman and Witness rank.

\subsection{Proof of~\pref{lem:representation}}

\begin{proof}[Proof of~\pref{lem:representation}]
Fix $h$ and $V:\Xcal \to [0,1]$. We drop the dependence on $h$ from
the notation, with $x,a$ always corresponding to states and actions at
time $h$ and $x'$ corresponding to an action at time $h+1$. Observe
that as $\Mhat$ is a low rank MDP, we have
\begin{align*}
\forall x,a:~~ \EE\sbr{V(x') \mid x,a,\Mhat} = \inner{\hat{\phi}(x,a)}{\int \hat{\mu}(x')V(x')} =: \inner{\hat{\phi}(x,a)}{\theta}
\end{align*}
Combining with~\pref{eq:sys_id}, we have that for any policy $\pi$:
\begin{align*}
& \EE\sbr{ \abr{\inner{\hat{\phi}(x,a)}{\theta} - \EE\sbr{V(x') \mid x,a,\Mcal}} \mid \pi,\Mcal}\\
& = \EE\sbr{ \abr{\EE\sbr{V(x') \mid x,a,\Mhat} - \EE\sbr{V(x') \mid x,a,\Mcal}} \mid \pi,\Mcal}\\
& \leq \EE\sbr{ \nbr{\inner{\hat{\phi}(x,a)}{\hat{\mu}(\cdot)} - T(\cdot \mid x,a)}_{\tv} \mid \pi,\Mcal} \leq \varepsilon.\tag*\qedhere
\end{align*}
\end{proof}

\subsection{On realizability and reachability}
\label{app:structural_1}

\begin{proposition}
\label{prop:realizability}
Fix $M\in \NN$, $n \leq \nicefrac{M}{2}$, and any algorithm. There exists a low rank MDP over
$M$ states with rank $2$ and horizon $2$ such that, if the algorithm
collects $n$ trajectories and outputs a policy $\hat{\pi}$, then with
probability at least $\nicefrac{1}{8}$, $\hat{\pi}$ is at least $\nicefrac{1}{8}$-suboptimal for the MDP.
\end{proposition}
The result shows that if the low rank MDP has $M$ states, then we
require $n = \Omega(M)$ samples to find a near-optimal policy with
moderate probability. Thus low rank structure alone is not sufficient
to obtain sample complexity guarantees that are independent of the
number of states.

\begin{proof}[Proof of~\pref{prop:realizability}]
The result is obtained by embedding a binary classification problem
into a low rank MDP and appealing to a standard binary classification
lower bound argument.  We construct a family of one-step transition
operators, all of which have rank $2$. The state space at the current
time is of size $M$ and there are two actions $\Acal \defeq \{0,1\}$. From
each $(x,a)$ pair we transition deterministically either to $x_{g}$ or
$x_{b}$, and we receive reward $1$ from $x_g$ and reward $0$ from
$x_b$.

Formally, we denote the states as $\{x_1,\ldots,x_M\}$ and index each
instance by a binary vector $v \in \{0,1\}^M$, which specifies the
good action for each state. The transition operator is
\begin{align*}
T_v(\cdot \mid x_j,a) = \left\{\begin{aligned}
x_g &\textrm{ if } a = v_j\\
x_b &\textrm{ if } a \ne v_j
\end{aligned}\right.
\end{align*}
There are therefore $2^M$ instances. Note that as there are only two
states at the next time, we trivially see that that transition
operator of each instance is rank $2$ and the linear MDP representation is:
\[\phi_v^\star(x_j,a) = (\one\{a=v_j\},\one\{a \ne v_j\})\quad\mbox{and}\quad \mu_v^\star(x') = (\one\{x'=x_g\},\one\{x'=x_b\}).\]

The starting distribution is uniform over $[M]$, so that in $n$
episodes, the agent collects a dataset
$\{(x^{(i)},a^{(i)},y^{(i)}\}_{i=1}^n$ where $x^{(i)} \sim
\unif(x_1,\ldots,x_M)$, $a^{(i)}$ is chosen by the agent and $y^{(i)}$
denotes whether the agent transitions to $x_g$ or $x_b$. Information
theoretically, this is equivalent to obtaining $n$ samples from the
following data generating process: sample $j \in \unif([M])$
and reveal $v_j$.

In this latter process, we can apply a standard binary classification lower bound argument.
Let $P_v$ denote the data distribution where indices
$j$ are sampled uniformly at random and labeled by $v_j$. Let
$P_v^{(n)}$ denote the product measure where $n$ samples are generated
iid from $P_v$. By randomizing the instance, for any example that
does not appear in the sample, the probability of error is
$\nicefrac{1}{2}$. Therefore the probability of error for any
classifier is
\begin{align*}
& \max_{v} \EE_{S \sim P_v^{(n)}} \PP_{j}[\hat{f}(j) \ne v_j] \ge \EE_{v \sim \textrm{Unif}(\{0,1\}^M)} \EE_{S \sim P_v^{(n)}} \PP_{j}[\hat{f}(j) \ne v_j]\\
& = \frac{1}{M}\sum_{j=1}^M \EE_{v} \EE_{S \sim P_v^{(n)}} \one\{\hat{f}(j) \ne v_j\} \geq \frac{1}{M}\sum_{j=1}^{M} \frac{1}{2} \PP_{S}[j \notin S]\\
& = \frac{1}{2} \rbr{1 - \frac{1}{M}}^n  \geq \frac{1}{2} \rbr{1 - \nicefrac{n}{M}}.
\end{align*}
The second inequality uses the fact that if $j$ does not appear in the
sample then $v_j\sim \textrm{Ber}(\nicefrac{1}{2})$. Equivalently, we
can first sample $n$ unlabeled indices, then commit to the label just
on these indices, so that the label for any index not in the sample
remains random. Thus for any classifier, there exists some instance
for which on average over the sample, the probability of error is at
least $\nicefrac{1}{4}$ as long as $n \leq \nicefrac{M}{2}$. This also
implies that with constant probability over the sample the error rate
is at least $\nicefrac{1}{8}$, since for any random variable $Z$ taking values in $[0,1]$, we have
\begin{align*}
\EE[Z] \leq \nicefrac{1}{8}\rbr{1 - \PP[Z > \nicefrac{1}{8}]} + \PP[Z \geq \nicefrac{1}{8}] \leq \nicefrac{1}{8} + \PP[Z \geq \nicefrac{1}{8}].
\end{align*}
Taking $Z = \PP_j[\hat{f}(j) \ne v_j]$, we have
\begin{align*}
\PP_{S\sim P_v^{(n)}}\sbr{ \PP_j[\hat{f}_j \ne v_j] \geq \nicefrac{1}{8}} \geq \frac{1}{2}\rbr{1 - \nicefrac{n}{M}} - \nicefrac{1}{8} \geq \nicefrac{1}{8},
\end{align*}
where the last inequality holds with $n \leq \nicefrac{M}{2}$.

Now, notice that we can identify any predictor with a policy in the
obvious way and also that the suboptimality for a policy is precisely
the classification error for the predictor. With this correspondence,
we obtain the result.
\end{proof}

\begin{proof}[Proof of~\pref{prop:reachability}]
Fix stage $h$. Assume that $X \defeq |\Xcal|, \dlv \defeq |\Zcal_h|$
are finite, where $\Zcal_h$ is the latent state space associated with
$T_h$. Recall that $\psi_h(x,a) \in \Delta(\Zcal)$ maps each state
action pair to a distribution over latent states. As $X, |\Acal|,
\dlv$ are all finite, we may collect these vectors as a matrix $\Psi
\in \RR^{XK \times \dlv}$ with rows corresponding to $\psi_h(x,a)$. A
policy $\pi$ induces a distribution over $(x,a)$ pairs, which we call
$p_{\pi,h} \in \Delta(\Xcal\times\Acal)$. The corresponding
distribution over latent variables at stage $h$ is therefore
$p_{\pi,h}^\top \Psi$.

We re-express $p_{\pi,h}$ in two steps. First, we can write $p_{\pi,h}
= A_{\pi,h}\times\tilde{p}_{\pi,h}$ where $A_{\pi,h} \in \RR^{XK\times
  X}$ is a matrix where the $x^{\textrm{th}}$ column describes the
distribution $\pi(\cdot \mid x) \in \Delta(\Acal)$, and
$\tilde{p}_{\pi,h} \in \Delta(\Xcal)$ is the distribution over $x_h$
induced by policy $\pi$. Note that $A_{\pi,h}$ is column stochastic
(it is non-negative with each column summing to $1$). In fact it has
additional structure, since in column $x$, only the rows corresponding
to $(x,\cdot)$ are non-zero, but this will not be essential for our
arguments. As $A_{\pi,h}^\top$ is therefore row-stochastic and the
product of two row-stochastic matrices is also row-stochastic, we have
that $A_{\pi,h}^\top \Psi \in \RR^{X \times \dlv}$ is also
row-stochastic.

Next, we use the dynamics at stage $h-1$ to re-write
$\tilde{p}_{\pi,h}$, which is the state distribution induced by policy
$\pi$ at stage $h$. As $T_{h-1}$ is also rank $d$, we can write $T(x_h
\mid x_{h-1},a_{h-1}) =
\inner{\phi_{h-1}(x_{h-1},a_{h-1})}{\mu_{h-1}(x_h)}$, and we can
collect the embeddings $\mu_{h-1}(x_h)$ as columns of a $d \times X$
matrix $U_{h-1}$. With these definitions, 
\begin{align*}
\tilde{p}_{\pi,h} = \EE\sbr{ U_{h-1}^\top \phi_{h-1}(x_{h-1},a_{h-1}) \mid \pi,\Mcal} = U_{h-1}^\top v_{\pi,h-1}.
\end{align*}
Here $\Mcal$ is the MDP in consideration.  In summary, for any policy
$\pi$, the distribution over latent variable $z_{h+1}$ (which
generates $x_{h+1}$) induced by policy $\pi$ can be written as
\begin{align*}
\PP\sbr{z_{h+1} = \cdot \mid \pi,\Mcal} = v_{\pi,h-1}^\top U_{h-1} A_{\pi,h}^\top\Psi \in \RR^{\dlv}.
\end{align*}
Now, let us use our linear-algebraic re-writing to express the
reachability condition. If a latent variable $z \in \Zcal_h$ is
reachable, then there exists some policy $\pi_z$ such that $\PP[z_{h+1} =
  z \mid \pi_z,\Mcal] \geq \etamin$. First of all, by importance
weighting on the last action $a_h$, we have:
\begin{align*}
\PP\sbr{z_{h+1} = z \mid a_{0:h-1} \sim \pi_z, a_h \sim \unif(\Acal),\Mcal} \geq \etamin/K.
\end{align*}
The normalization condition on $\phi_{h-1}$ leads to the upper bound
\begin{align*}
\frac{\etamin}{K} \leq \abr{ v_{\pi_z,h-1}^\top U_{h-1} A_{h}^\top \Psi e_z } \leq \max_{v: \nbr{v}_2 \leq 1} \abr{ v^\top U_{h-1} A_{h} \Psi e_z} = \nbr{U_{h-1}A_h \Psi e_z}_2.
\end{align*}
where we use $A_h$ to denote the action-selection matrix induced by
the uniform policy.

Next, consider some $\ell_{\infty}$-bounded vector $w \in \RR^{\Zcal}$
with $\nbr{w}_{\infty} \leq 1$. The fact that $A_{h}^\top\Psi$ is
row-stochastic implies that
\begin{align*}
\nbr{A_{h}^\top \Psi w}_{\infty} \leq \max_{x,a} \abr{\psi(x,a)^\top w} \leq \nbr{w}_{\infty}.
\end{align*}
Therefore, using the normalization condition on $U_{h-1}$ we have
\begin{align*}
\max_{w: \nbr{w}_{\infty} \leq 1} \nbr{U_{h-1} A_{h}^\top \Psi w}^2_{2} \leq \max_{w: \nbr{w}_{\infty} \leq 1} \nbr{U_{h-1} w}_2^2 \leq d.
\end{align*}
We will select a vector $w \in \{\pm 1\}^{\dlv}$, for which we know this
upper bound holds. We select the vector iteratively, peeling off
latent variables that are reachable. For brevity, define $B \defeq
U_{h-1}A_h^\top\Psi$ and observe that
\begin{align*}
\nbr{B w}^2_2  &= \nbr{B e_{z_1} w[z_1]}_2^2 + 2 \langle B e_{z_1} w[z_1], \sum_{z \ne z_1}B e_{z} w[z]\rangle + \| \sum_{z \ne z_1}B e_{z} w[z] \|_2^2.
\end{align*}
If $z_1$ is $\etamin$ reachable, then the first term is at least
$\rbr{\nicefrac{\etamin}{K}}^2$ by the above calculation and the fact that we take $w[z_1]
\in \{\pm 1\}$. Then we ensure that the cross-term is non-negative by
setting $w[z_1]$ appropriately. Note that $w[z_1]$ is formally a
function of the remaining coordinates of $w$, but we have not
introduced any constraint on these remaining coordinates. Therefore,
for $z_1$ is $\etamin$ reachable, we get (for this partially specified $w$)
\begin{align*}
\nbr{Bw}_2^2 \geq \rbr{\nicefrac{\etamin}{K}}^2 + \| \sum_{z \ne z_1}B e_{z} w[z]\|_2^2
\end{align*}
Continuing in this way, we iteratively peel off latent variables that
are $\etamin$ reachable, and for each we gain $\rbr{\nicefrac{\etamin}{K}}^2$
in the lower bound. Therefore, if all $\dlv$ latent variables are
reachable, there exists some $w \in \{\pm 1\}^{\dlv}$ such that
\begin{align*}
\dlv\cdot \rbr{\nicefrac{\etamin}{K}}^2 \leq \nbr{U_{h-1}^\top A_h^\top \Psi w}_2^2 \leq d,
\end{align*}
which implies that we must have $\dlv \leq \nicefrac{dK^2}{\etamin^2}$.
\end{proof}

\subsection{Separation results}
\label{app:structural_2}
\begin{proof}[Proof of~\pref{prop:block_separation}]
Fix $N$ and consider a MDP with horizon $2$, where at stage $1$ there
is only one state $x$ and two actions $a_1,a_2$. At stage $2$ there
are $N$ possible states, so that $T(\cdot \mid x,a_i) \in \Delta([N])$
for each $i \in \{1,2\}$.
We define the transition operator for stage $1$, called $T$ for
brevity, explicitly in terms of its factorization. Let $\phi(x,a_1) =
e_1$ and $\phi(x,a_2) = e_2$ where $e_1,e_2 \in \RR^2$ denotes the two
standard basis elements in two dimensions. We define $\mu_1(i) =
\nicefrac{1}{N}$, $\mu_2(i) = i/(\sum_{j=1}^N j)$ and $\mu(i) =
(\mu_1(i), \mu_2(i)) \in \RR^2$. Thus $T(x'=i \mid x,a) =
\inner{\phi(x,a)}{\mu(i)}$, which can be easily verified to be a valid
transition operator. By construction $T$ has rank $2$.

For clarity we express $T$ as the $2 \times N$
matrix.
\begin{align*}
T \defeq \rbr{\begin{matrix}
\nicefrac{1}{N} & \nicefrac{1}{N} & \ldots & \nicefrac{1}{N}\\
\nicefrac{1}{(\sum_{j=1}^N j)} & \nicefrac{2}{(\sum_{j=1}^N j)} & \ldots & \nicefrac{N}{(\sum_{j=1}^N j)}
\end{matrix}}.
\end{align*}

We now show that the block MDP representation must have $N$ latent
states. Suppose the block MDP representation is $T(x'=i \mid x,a) =
\inner{\phi_B(x,a)}{\mu_B(i)}$. The block MDP representation requires
that for each index $i$ the vector $\mu_B(i)$ is one-sparse. From this, we
deduce a constraint that arises when two states belong to the same
block. If $i,j$ belong to the same block, say block $b$, then for
each $(x,a) \in \Xcal\times\Acal$, we have
\begin{align*}
T(x'=i \mid x,a) &= \phi_B(x,a)[b]\mu_B(i)[b] = \frac{\mu_B(i)[b]}{\mu_B(j)[b]} \cdot \phi_B(x,a)[b] \mu_B(j)[b] \\
&= \frac{\mu_B(i)[b]}{\mu_B(j)[b]} \cdot T(x'=j \mid x,a)
\end{align*}
In words, if states $i,j$ at stage 2 belong to the same
block, then the vectors $T(x'=i\mid \cdot), T(x'=j\mid \cdot)$ must be
pairwise linearly dependent.\footnote{Note that this is equivalent to
  the notion of backward kinematic
  inseparability~\citep{misra2019kinematic}.} Based on our
construction, $T(x' = i \mid\cdot) = \mu(i)$, which is just the
$i^{\textrm{th}}$ column of the matrix $T$. By inspection, all $N$
vectors are pairwise linearly independent, and so we can
  conclude that the block MDP representation must have $N$ latent
  states.
\end{proof}

\begin{proof}[Proof of~\pref{prop:simplex_separation}]
We consider a one step transition operator $T$ that we instantiate to
be the slack matrix describing a certain polyhedral set. Let $n$ be
even and let $K_n$ be the complete graph on $n$ vertices.  To set up
the notation we will work with vectors $x \in \RR^{{n \choose 2}}$
that associate a weight to each edge. We index the vectors as
$x_{u,v}$ where $u\ne v \in [n]$ correspond to vertices.

A result of~\citet{edmonds1965maximum} states that the perfect
matching polytope, which is the convex hull of all edge-indicator
vectors corresponding to perfect matchings, can be explicitly written
in terms of ``odd-cut'' constraints:
\begin{align*}
\Pcal_n &\defeq \textrm{conv}\cbr{\one_{M} \in \RR^{{n \choose 2}} \mid M \textrm{ is a perfect matching in } K_n}\\
& = \cbr{x \in \RR^{{n \choose 2}}: x \succeq 0, \forall v: \sum_u x_{u,v} = 1, \forall U \subset [n], |U| \textrm{ odd } \sum_{v \notin U}\sum_{u \in U} x_{u,v} \geq 1}.
\end{align*}
This polytope has exponentially many vertices and exponentially many
constraints. Formally, there are $V \defeq
\frac{n!}{2^{\nicefrac{n}{2}}(\nicefrac{n}{2})!}$ vertices,
corresponding to perfect matchings in $K_n$, and the number of
constraints is $C \defeq 2^{\Omega(n)}$ corresponding to the number of
odd-sized subsets of $[n]$. By adding one dimension to account for the
offsets in the inequality constraints, we can enumerate the vertices
$v_1,\ldots,v_V \in \RR^{{n \choose 2}+1}$ and the constraints
$c_1,\ldots,c_C \in \RR^{{n \choose 2}+1}$, such that
$\inner{c_i}{v_j} \geq 0$ for all $i,j$. Then, we define the
\emph{slack matrix} for this polytope to be $Z \in \RR_+^{C \times V}$
with entries $Z_{i,j} = \inner{c_i}{v_j}$.

This slack matrix clearly has rank ${n \choose 2} + 1 = O(n^2)$. On
the other hand, we claim that the non-negative rank is at least
$2^{\Omega(n)}$. This follows from (a) the fact that $\Pcal_n$ has
extension complexity $2^{\Omega(n)}$~\citep{rothvoss2017matching}, (b)
the extension complexity of a polytope is exactly the non-negative
rank of its slack
matrix~\citep{yannakakis1991expressing,fiorini2013combinatorial}.

Next, we define the transition operator $T$. We associate each $(x,a)$
pair with a constraint $c_i$ and each $x'$ with a vertex $v_j$. Then
we define
\begin{align*}
T(x' \mid x,a) = \frac{\inner{c_i}{v_j}}{\sum_{k=1}^{V} \inner{c_i}{v_k}}
\end{align*}
This is easily seen to be a distribution for each $(x,a)$ pair. We can
represent $T$ as a $C \times V$ matrix $T = DZ$ where $D$ is a
diagonal matrix (with strictly positive diagonal) and $Z$ is the slack
matrix defined above.

We conclude the proof with two facts
from~\citet{cohen1993nonnegative}. First, the non-negative rank is
preserved under positive diagonal rescaling, and so the non-negative
rank of $T$ is also $2^{\Omega(n)}$. Second, for a row-stochastic
matrix $P$, the non-negative rank is equal to the smallest number of
factors we can use to write $P = RS$ where both $R$ and $S$ are
row-stochastic (here factors refers to the internal dimension). It is
immediate that the simplex features representation corresponds to such
a row-stochastic factorization, and so we see that any simplex
features representation of $T$ must have embedding dimension at least
$2^{\Omega(n)}$.
\end{proof}

\subsection{On Bellman and Witness rank}
\label{app:structural_3}
We now state the formal version of~\pref{prop:bellman_informal}. We
consider the value-function/policy decomposition studied
by~\citet{jiang2017contextual} where we approximate the value
functions with a class $\Gcal: \Xcal \to [0,H]$ and the policies with
a class $\Pi: \Xcal \to \Acal$. Given an explicit reward
function $R$ with range $[0,1]$ and the function class $\Phi$ of
candidate embeddings, we define these two classes as:
\begin{align*}
\Pi(\Phi) &\defeq \cbr{\pi: x_h \mapsto \argmax_{a \in \Acal} \inner{\phi_h(x_h,a)}{\theta_h} + R(x_h,a_h): \theta_{0:H-1} \in B_{d}(H\sqrt{d}), \phi_{0:H-1} \in \Phi},\\
\Gcal(\Phi) &\defeq\cbr{g: x_h \mapsto \max_a\inner{\phi_h(x_h,a)}{\theta_h} + R(x_h,a_h): \theta_{0:H-1} \in B_{d}(H\sqrt{d}), \phi_{0:H-1} \in \Phi}.
\end{align*}
Here $B_d(\cdot)$ is the Euclidean ball in $d$ dimensions with the specified radius.
We have the following proposition:
\begin{proposition}
\label{prop:bellman}
The low rank MDP model with \emph{any} function classes $\Gcal \subset
\Xcal \to [0,B]$ and $\Pi\subset \Xcal \to \Delta(\Acal)$ has bellman
rank at most $d$ with normalization parameter $O(B\sqrt{d})$. Additionally,
for any known reward function $R$ with range $[0,1]$ and assuming
$\phi_{0:H-1}^\star \in \Phi$, the optimal policy and value function
lie in $(\Gcal(\Phi),\Pi(\Phi))$, and so \olive has sample complexity
$\otil\rbr{\poly(d, H, K, \log |\Phi|, \epsilon^{-1})}$.
\end{proposition}
\begin{proof}[Proof of~\pref{prop:bellman}]
The result is essentially Proposition 9
in~\citet{jiang2017contextual}, who address the simplex representation
case. We address the general case and also verify the
realizability assumption.

Consider any explicitly specified reward function $R: \Xcal \times
\Acal \times \{0,\ldots,H-1\} \to [0,1]$ and any low rank MDP with
embedding functions $\phi_{0:H-1}^\star,\mu_{0:H-1}^\star$ and
embedding dimension $d$.  For any policy $\pi,\pi'$ and any value
function $g: \Xcal \to \RR$ we define the \emph{average Bellman
  error}~\citep{jiang2017contextual} as
\begin{align*}
\Ecal(\pi,(g,\pi'),h) \defeq \EE\sbr{g(x_h) - R_h(x_h,a_h) - g(x_{h+1}) \mid a_{0:h-1}\sim\pi, a_h = \pi'(x_h), \Mcal},
\end{align*}
We also introduce the shorthand
\begin{align*}
\Delta((g,\pi'),x_h) \defeq \EE\sbr{g(x_h) - R_h(x_h,a_h) - g(x_{h+1}) \mid x_h, a_h = \pi'(x_h)}.
\end{align*}
Then, in the low rank MDP, the average Bellman error admits a factorization as follows
\begin{align*}
\Ecal(\pi,(g,\pi'),h) &= \EE\sbr{\Delta((g,\pi'),x_h) \mid x_h \sim \pi} \\
&= \inner{\EE\sbr{\phi_{h-1}^\star(x_{h-1},a_{h-1}) \mid \pi}}{\int \mu_{h-1}^\star(x_h)\Delta((g,\pi'),x_h)d(x_h)}\\
& =: \inner{\nu_h(\pi)}{\xi_h((g,\pi'))}
\end{align*}
We also have the normalization $\nbr{\nu_h(\pi)}_2 \leq 1$ and
$\nbr{\xi_h((g,\pi))}_2 \leq (2B+1)\sqrt{d}$. This final calculation
is based on the triangle inequality, the bounds on $g$ and $R$ and the
normalization of $\mu_{h-1}^\star$. Thus for any low rank
MDP and \emph{any} (bounded) function class $\Gcal,\Pi$, the Bellman
rank is at most $d$ with norm parameter $O(B\sqrt{d})$.

To prove that \olive has low sample complexity, we need to verify that
the optimal policy and optimal value function lie in $\Pi(\Phi)$ and
$\Gcal(\Pi)$ respectively. Then we must calculate the statistical
complexity of these two classes. Observe that we can express the
Bellman backup of any function $V:\Xcal \to \RR$ as a linear function
in the optimal embedding $\phi^\star$:
\begin{align*}
(\Tcal_h V)(x,a) &\defeq \EE[ R_h(x,a) + V(x') \mid x,a,h] = R_h(x,a) + \inner{\phi_h^\star(x,a)}{\int\mu_h^\star(x')V(x')d(x')}\\
& = R_h(x,a) + \inner{\phi^\star_h(x,a)}{w}.
\end{align*}
for some vector $w$. Moreover, if $V: \Xcal \to [0,H]$, we know that
$\nbr{w} \leq H\sqrt{d}$. In particular, this implies that the optimal
$Q$ function is a linear function in the true embedding functions
$\phi_{0:H-1}^\star$, and so realizability holds for
$\Gcal(\Phi),\Pi(\Phi)$. These function classes have range $B =
O(H\sqrt{d})$ so the normalization parameter in the Bellman rank
definition is $O(Hd)$.

Finally, we must calculate the statistical complexity of these two
classes. For $\Pi(\Phi)$ the Natarajan dimension is at most
$\otil\rbr{H (d + \log |\Phi|)}$, since for each $h$, we choose
$\phi_h$ and a $d$-dimensional linear classifier. Analogously the
pseudo-dimension of $\Gcal(\Phi)$ is $\otil\rbr{H(d +\log |\Phi|)}$.
Formally, we give a crude upper bound on the growth function, focusing on $\Pi(\Phi)$.
Fix $h$, let $S$ be
a sample of $n$ pairs $(x,a)$, and let $h_1,h_2: S \to \{0,1\}$
such that $h_1(x,a) \ne h_2(x,a)$ for all points in the sample. Since
once we fix $\phi \in \Phi$, we have a linear class, we can vary
$\theta$ to match $h_1,h_2$ on at most $(n+1)^d$ subsets $T \subset
S$. Then by varying $\phi \in \Phi$ we can match $h_1,h_1$ in total on
$|\Phi| (n+1)^d \leq n^{\order(d + \log |\Phi|)}$ subsets. If $S$ is
shattered, this means that $2^n \leq n^{\order(d + \log |\Phi|)}$,
which means that the Natarajan dimension is $O((d + \log|\Phi|)\log(d
+ \log |\Phi|))$. This calculation is for a fixed $h$, but the same
argument yields the bound of $\tilde{O}(H(d + \log |\Phi|))$.
Instantiating, we obtain the sample complexity bound for \olive.
\end{proof}

For the model-based version using the witness rank, the arguments are
more straightforward.
\begin{proposition}
\label{prop:witness}
The low rank MDP model with any candidate model class $\Pcal$ has
witness rank at most $d$, with norm parameter
$O(\sqrt{d})$. Additionally, for any explicitly specified reward
function $R$ with range $[0,1]$ and under~\pref{assum:realizability},
the algorithm of~\citet{sun2018model} (with witness class of all
bounded functions) has sample complexity $\otil\rbr{\poly(d,K,H,\log
  |\Phi||\Upsilon|,\varepsilon^{-1})}$.
\end{proposition}
\begin{proof}
Given a model $M$ and an explicit reward function $R$, we use $\pi_M$
to denote the optimal policy for $R$ with transitions governed by
$M$. Then, for two models $M_1,M_2$ and a time step $h$ the witness
model misfit, when instantiated with the test function class as all
bounded functions, is defined as
\begin{align*}
\Wcal(M_1,M_2, h) \defeq \EE\sbr{ \nbr{M_2(\cdot \mid x_h,a_h) - \Mcal(\cdot \mid x_h,a_h)}_{\tv} \mid a_{0:h-1} \sim \pi_{M_1}, a_h = \pi_{M_2},\Mcal}.
\end{align*}
Here we use the notation $M(\cdot \mid x_h,a_h)$ to denote the
transition operator implied by $M$ at stage $h$. Recall that $\Mcal$
is the true MDP. In words, the witness model misfit is the one-step TV
error between candidate model $M_2$ and the true environment $\Mcal$
on the data distribution induced by executing policy $\pi_{M_1}$ for
$h$ steps.

Using the backing up argument from the proof of~\pref{prop:bellman},
it is easy to see that the witness model misfit admits a factorization as
\begin{align*}
\Wcal(M_1,M_2,h) = \inner{\EE\sbr{\phi_{h-1}^\star(x_{h-1},a_{h-1}) \mid \pi_{M_1},\Mcal}}{\int \nu_{h-1}^\star(x_{h})\Delta(x_h,M_2)}
\end{align*}
where $\Delta(x_h,M_2)$ is the expected total variation distance
between $M_2$ and $\Mcal$ on $(x_h,\pi_{M_2}(x_h))$. Based on this
calculation, the witness rank is at most $d$ and the normalization
parameter is at most $O(\sqrt{d})$. It is more straightforward to see
that realizability holds here, and so the algorithm
of~\citet{sun2018model} has the stated sample complexity.
\end{proof}

\section{Analysis of \ouralg}
\label{app:algos}
As a reminder, \ouralg interacts with a low rank MDP $\Mcal$, with
time horizon $H$ and with non-stationary dynamics $T_h(x_{h+1} \mid
x_h,a_h) = \inner{\phi_h^\star(x_h,a_h)}{\mu_h^\star(x_{h+1})}$. We
assume that for each $h$ the operators $\phi_h^\star, \mu_h^\star$
embed into $\RR^d$.  We use the shorthand $\EE_{\pi}\sbr{\cdot} =
\EE\sbr{\cdot \mid \pi,\Mcal}$ to denote expectations when policy
$\pi$ interacts with the real MDP $\Mcal$ and
$\hat{\EE}_\pi\sbr{\cdot} = \EE\sbr{\cdot \mid \pi,\Mhat}$ for
expectations when the policy interacts with the estimated MDP $\Mhat$,
which has dynamics $\hat{T}_{0:H-1}$. Note that this MDP model changes
from iteration to iteration. When necessary we will use
$\hat{\EE}_{j,\pi}\sbr{\cdot}$ to denote the MDP model learned in the
$j^{\textrm{th}}$ iteration of \ouralg.

The analysis of \ouralg is based on a potential function argument. The
key quantities are the second moment matrices of the real features
induced by the policies $\rho_0,\rho_1,\ldots$ at each time
$h$. Formally, for $h \in \{0,\ldots,H-1\}$ and $j \in
[J_{\max}]$ we define
\begin{align*}
\Sigma_{h,j} \defeq \lambda I_{d \times d} + \sum_{i=0}^{j-1} \EE_{\rho_i}\sbr{\phi^\star_h(x_h,a_h)\phi^\star_h(x_h,a_h)^\top},
\end{align*}
where $\lambda >0$ is a small constant we will set towards the end of
the proof. Note that $\Sigma_{h,j} \succ 0$ for all $h,j$.

The importance of $\Sigma_{h,j}$ is demonstrated in the next result,
which establishes an accuracy guarantee for the model
$\hat{T}_{0:H-1}$ learned in iteration $j$. The result is a corollary
of~\pref{thm:mle}.
\begin{corollary}
\label{corr:mle}
Fix $j \geq 1$, $h \in \{1,\ldots H-1\}$, $\delta \in (0,1)$, and let
$\rho_0,\ldots,\rho_{j-1}$ be any (possibly data-dependent) policies,
with $\Sigma_{h,j}$ defined accordingly. Let $D_h$ be a dataset of
$nj$ examples where for each $0 \leq i < j$ we collect $n$ triples
$(x_h,a_h,x_{h+1})$ by rolling in with $\rho_i$ to $x_h$ and taking
$a_h$ uniformly at random. Then with probability $1-\delta$ the output
$(\hat{\phi}_h,\hat{\mu}_h)$ of $\mle(D_h)$ satisfies
\begin{align*}
\nbr{ \int \mu_{h-1}(x_{h}) \unif(a_h)\nbr{\inner{\hat{\phi}_h(x_h,a_h)}{\hat{\mu}_h(\cdot)} - T_h(\cdot \mid x_h,a_h)}_{\tv}}_{\Sigma_{h-1,j}}^2 \leq \lambda d + \frac{2 \log (|\Phi||\Upsilon|/\delta)}{n}.
\end{align*}
Additionally, for any $j \geq 1$, with probability at least $1-\delta$
we have
\begin{align*}
\nbr{\inner{\hat{\phi}_0(x_0,a_0)}{\hat{\mu}_0(\cdot)} - T(\cdot \mid x_0,a_0)}_{\tv}^2 \leq \frac{2 \log (|\Phi||\Upsilon|/\delta)}{n}.
\end{align*}
\end{corollary}
\begin{proof}
For shorthand, we use $v_h$ to denote the $d$-dimensional vector on
the left hand side of the desired bound. Then, the left hand side is
\begin{align*}
\nbr{v_h}_{\Sigma_{h-1},j}^2 &= \lambda \nbr{v_h}_2^2 + \sum_{i=0}^{j-1} \EE_{\rho_i} \sbr{\rbr{\phi_{h-1}^\star(x_{h-1},a_{h-1})^\top v_h}^2}\\
 &= \lambda \nbr{v_h}_2^2 + \sum_{i=0}^{j-1} \EE_{\rho_i} \sbr{\rbr{\EE\sbr{ \nbr{\inner{\hat{\phi}_h(x_h,a_h)}{\hat{\mu}_h(\cdot)} - T_h(\cdot \mid x_h,a_h)}_{\tv} \mid x_{h-1},a_{h-1}}}^2}\\
& \leq \lambda d + \sum_{i=0}^{j-1} \EE_{\rho_i} \sbr{\nbr{\inner{\hat{\phi}_h(x_h,a_h)}{\hat{\mu}_h(\cdot)} - T_h(\cdot \mid x_h,a_h)}^2_{\tv} \mid a_h\sim\unif(\Acal)}.
\end{align*}
The first term appears in the desired bound, so now we focus on the
second term. We have $nj$ total examples that form a martingale
process, since $\rho_i$ depends on all of the data collected
in previous iterations. Applying~\pref{thm:mle}, we see that with probability
$1-\delta$:
\begin{align*}
\sum_{i=0}^{j-1} n\cdot \EE_{\rho_i} \sbr{\nbr{\inner{\hat{\phi}_h(x_h,a_h)}{\hat{\mu}_h(\cdot)} - T_h(\cdot \mid x_h,a_h)}^2_{\tv} \mid a_h\sim\unif(\Acal)} \leq 2 \log (|\Phi||\Upsilon/\delta),
\end{align*}
where the factor of $n$ arises since we collect $n$ examples from
$\rho_i$. Re-arranging we obtain the first bound. The bound for $h=0$
is a direct application of~\pref{thm:mle}, since we assume there is a
fixed starting state $x_0$ with a single available action.
\end{proof}

Now that we have established an accuracy guarantee in terms of the
previous exploratory policies, we state and prove the main technical
``simulation'' lemma. The following notation is helpful. Given an MDP
model $\hat{\phi}_{0:H-1},\hat{\mu}_{0:H-1}$ and positive definite matrices
$\Sigma_0,\ldots,\Sigma_{H-1}$, define
\begin{align*}
\forall h \geq 1: \err_h(\Sigma_{h-1}) &\defeq \nbr{\int \mu_{h-1}(x_h)\unif(a_h) \nbr{\inner{\hat{\phi}_h(x_h,a_h)}{\hat{\mu}_h(\cdot)} - T_h(\cdot \mid x_h,a_h)}_{\tv}}_{\Sigma_{h-1}}^2,\\
\err_0 &\defeq \nbr{\inner{\hat{\phi}_0(x_0,a_0)}{\hat{\mu}_0(\cdot)} - T_0(\cdot \mid x_0,a_0)}^2_{\tv}.
\end{align*}
Further, for each $h\geq 1$, define $\Kcal_h(\Sigma_h) \defeq \cbr{
  (x,a) \in \Xcal \times \Acal:
  \nbr{\phi_h^\star(x_h,a_h)}_{\Sigma_h^{-1}}^2 \leq 1}$. Let
$M_\Kcal$ be the MDP with non-stationary transition operator
$T_{h,\Kcal}$ defined as
\begin{align*}
T_{h,\Kcal}(x_{h+1} \mid x_h,a_h) = \left\{\begin{aligned}
\inner{\phi_h^\star(x_h,a_h)}{\mu_h^\star(x_{h+1})} & \textrm{ if } (x_h,a_h) \in \Kcal_h(\Sigma_h)\\
\one\{x_{h+1} = \xabsorb\} & \textrm{ if } (x_h,a_h) \notin \Kcal_h(\Sigma_h)
\end{aligned}\right.,
\end{align*}
where $\xabsorb$ is a special self-looping absorbing state with a
single action $\aabsorb$ such that $T(\xabsorb \mid \xabsorb,\aabsorb)
= 1$ always. The initial transition $T_{0,\Kcal}$ is identical to
$T_0$. The intuition is that $\Kcal$ denotes the set of ``known''
state-action pairs, and the MDP $M_\Kcal$ terminates any episode that
escapes the known set.  In all of these definitions, we suppress the
dependence on $\Sigma_h$ when it is clear from context. We always
consider $(\xabsorb,\aabsorb)$ to be \emph{known}.

\begin{lemma}
\label{lem:clipped_simulation}
Let $\hat{\phi}_{0:H-1},\hat{\mu}_{0:H-1}$ be an MDP model and let $\Sigma_{0:H-1}$ be positive definite matrices. Assume that
\begin{align*}
\forall h \in \{0,\ldots,H-1\}: ~ \err_h(\Sigma_{h-1}) \leq \epstv.
\end{align*}
Let $f: \Xcal\times\Acal \to [0,1]$ be any function such that $f(\xabsorb,\aabsorb) = 0$,
and let $\pi$ be any policy. Then for any $h \in\{0,\ldots,H-1\}$
\begin{align*}
\EE_\pi\sbr{f(x_h,a_h) \mid \Mcal_\Kcal} - HK\sqrt{\epstv} \leq \hat{\EE}_\pi\sbr{f(x_h,a_h)} &\leq \EE_\pi\sbr{f(x_h,a_h) \mid \Mcal_\Kcal} + HK \sqrt{\epstv}\\
& ~~~~~~~~~~~ + \sum_{h'=0}^{h-1} \PP\sbr{(x_{h'},a_{h'}) \notin \Kcal_{h'} \mid \pi,\Mcal_{\Kcal}}.
\end{align*}
\end{lemma}
This lemma establishes a sharp relationship between the learned MDP
$\Mhat$ and an \emph{absorbing} MDP $\Mcal_\Kcal$, defined in terms of
the matrices $\Sigma_h$, which also governs the estimation error for
$\Mhat$. Intuitively the error guarantee implies that $\Mhat$ closely
approximates $\Mcal_\Kcal$ provided we stay within the known set
$\Kcal_{0:H-1}$. Conversely the difference in value between the two
MDPs can be bounded in terms of the escaping probability, which is the
third term on the right hand side of the bound.

Note also that the above lemma, with $\epstv=0$, can be used to compare
$\Mcal$ with $\Mcal_{\Kcal}$, which yields that for any non-negative function
$f$ and any policy $\pi$:
\begin{align}
\EE_\pi\sbr{f(x_h,a_h) \mid \Mcal_\Kcal}\leq \EE_\pi\sbr{f(x_h,a_h)} \leq \EE_\pi\sbr{f(x_h,a_h) \mid \Mcal_\Kcal} + \sum_{h=0}^{h-1} \PP\sbr{(x_{h'},a_{h'})\notin\Kcal_{h'} \mid \pi,\Mcal_\Kcal}. \label{eq:rmax_lemma}
\end{align}
This bound actually holds for any sets $\Kcal_h$. We now
turn to the proof of~\pref{lem:clipped_simulation}.

\begin{proof}
Let $\hat{V}_{h'}(x) \defeq \hat{\EE}_\pi[f(x_h,a_h) \mid x_{h'} = x]$
denote the value function (relative to $f$) in the model and let $V_{h',\Kcal}(x)$
denote the analogous quantity in the absorbing MDP. We have the
following telescoping identity:
\begin{align*}
& \hat{\EE}_\pi [f(x_h,a_h)] - \EE_\pi\sbr{f(x_h,a_h) \mid M_\Kcal} = \int \rbr{\hat{T}_0(x_1 \mid x_0,a_0) - T_{0,\Kcal}(x_1 \mid x_0,a_0)} \hat{V}_1(x_1) \\
& ~~~~~~~~~~ + \EE_\pi\sbr{\hat{V}_1(x_1) - V_{1,\Kcal}(x_1) \mid M_\Kcal}\\
& = \sum_{h'=0}^{h-1}\EE_\pi\sbr{\int (\hat{T}_{h'}(x_{h'+1} \mid x_{h'},a_{h'}) - T_{h',\Kcal}(x_{h'+1} \mid x_{h'},a_{h'})) \hat{V}_{h'+1}(x_{h'+1}) \mid M_\Kcal}.
\end{align*}

Note that $\hat{V}_h(x) = V_{h,\Kcal}(x)$ at the time $h$ where we
apply function $f$. This means that we only accumulate errors up to
time $h-1$. We now work with one of these terms. In the remainder of
the proof, unless otherwise specified, all expectations are taken by
executing $\pi$ in $M_\Kcal$. By adding and subtracting $T_{h'}$, we
get two terms
\begin{align*}
\mathrm{Term 1}_{h'} \defeq \EE\sbr{ \int (\hat{T}_{h'}(x_{h'+1} \mid x_{h'},a_{h'}) - T_{h'}(x_{h'+1} \mid x_{h'},a_{h'}) )\hat{V}_{h'+1}(x_{h'+1})}\\
\mathrm{Term 2}_{h'} \defeq \EE\sbr{\int (T_{h'}(x_{h'+1} \mid x_{h'},a_{h'}) - T_{h',\Kcal}(x_{h'+1} \mid x_{h'},a_{h'})) \hat{V}_{h'+1}(x_{h'+1})}.
\end{align*}
For $\mathrm{Term 1}_{h'}$, note that the expression evaluates to zero
if $x_{h'} = \xabsorb$ since both $\hat{T}_{h'}$ and $T_{h'}$ agree that
$\xabsorb$ has a single self-looping action. We now bound
$\mathrm{Term 1}$ at time $h'=1$, although exactly the same argument
applies to $h' > 0$.  Defining $\err(x_1,a_1) \defeq
\nbr{\hat{T}_1(\cdot \mid x_1,a_1) - T_1(\cdot \mid x_1,a_1)}_{\tv}$
and by applying Holder's inequality, we have
\begin{align*}
\mathrm{Term 1}_1 &\leq \EE\sbr{\one\{x_1 \ne \xabsorb\} \nbr{\hat{T}_1(\cdot \mid x_1,a_1) - T_1(\cdot \mid x_1,a_1)}_{\tv}}\\
& \leq K \cdot \EE\sbr{\one\{x_1 \ne \xabsorb\} \unif(a_1) \err(x_1,a_1)}\\
& = K \cdot \EE\sbr{\phi_{0}^\star(x_{0},a_{0})\one\cbr{\nbr{\phi_{0}^\star(x_{0},a_{0})}_{\Sigma_{0}^{-1}}^2\leq 1}} \cdot \int \mu_{0}^\star(x_1)\unif(a_1)\err(x_1,a_1)\\
& \leq K \cdot \EE \sbr{\nbr{\phi_{0}^\star(x_{0},a_{0})}_{\Sigma_{0}^{-1}}\one\cbr{\nbr{\phi_{0}^\star(x_{0},a_{0})}_{\Sigma_{0}^{-1}}^2\leq 1}} \cdot \sqrt{\err_1(\Sigma_{0})}\\
& \leq K\sqrt{\epstv}
\end{align*}
The first inequality is Holder's inequality, while the second is an
importance weighting argument to replace $a_1 \sim \pi(x_1)$ with the
uniform distribution. Next we re-write the expectation using the low
rank dynamics, and also use the fact that $x_1 \ne \xabsorb$ implies
that the previous transition was non-absorbing, which yields the
indicator. Finally, we use the Cauchy-Schwarz inequality in the
$\Sigma_{0}$ norm, along with the implication of the indicator and the
assumed bound on $\err_1(\Sigma_{0})$. This argument applies as is to
all indices $h' > 0$ and for $h' = 0$ we simply apply Holder's
inequality and the definition of $\err_0$ to obtain the upper bound
$\sqrt{\epstv}$. In total, these terms account for the
$HK\sqrt{\epstv}$ terms on both sides of the lemma statement.

Next we turn to $\mathrm{Term 2}_{h'}$. For the first inequality in
the lemma statement, we need to upper bound $-\mathrm{Term 2}_{h'}$,
but this term is easily seen to be non-positive, since
$\hat{V}(\xabsorb) = 0$ always. So this proves the first
inequality. For the second inequality, we have (again focusing on time
$1$)
\begin{align*}
\mathrm{Term 2}_1 = \EE\sbr{ \int T_1(x_{2} \mid x_1,a_1) \one\{(x_1,a_1) \notin \Kcal_1\} \hat{V}_{2}(x_2)} \leq \PP\sbr{(x_1,a_1) \notin \Kcal_1 \mid \pi,M_\Kcal}.
\end{align*}
The same argument applies for all $h \geq 1$.
\end{proof}

In the next lemma, we consider the case where $\rho_j$ has large
escaping probability, measured with respect to the known sets
$\Kcal_h(\Sigma_{h,j})$. Recall that $\Sigma_{h,j}$ is the second
moment matrix of the true features $\phi_h^\star$ at time $h$ induced
by the previous roll-in policies $\rho_0,\ldots,\rho_{j-1}$.
\begin{lemma}
\label{lem:case_1}
Consider iteration $j$ of \ouralg and assume that
$\err_h(\Sigma_{h-1,j}) \leq \epstv$ for each $h$ with our current model
$\Mhat$. Define $R_h(x,a) \defeq \one\{(x,a) \notin
\Kcal_h(\Sigma_{h,j})\}$ for each $h$. Then,
\begin{align*}
\max_{h} \tr\rbr{\EE_{\rho_j}\sbr{\phi_{h}^\star(x_{h},a_{h})\phi_{h}^\star(x_{h},a_{h})} \Sigma_{h,j-1}^{-1}} \geq \frac{1}{H}\max_{h}\cbr{\hat{\EE}_{\rho_j}\sbr{R_h(x_h,a_h)} - HK\sqrt{\epstv}}
\end{align*}
\end{lemma}
\begin{proof}
For shorthand let $\Kcal_h \defeq \Kcal_h(\Sigma_{h,j})$ denote the
known set at round $j$, and let $\Mcal_\Kcal$ denote the corresponding
absorbing MDP.  Then, applying the second inequality
in~\pref{lem:clipped_simulation} we have
\begin{align*}
 \hat{\EE}_{\rho_j}\sbr{R_h(x_h,a_h)}
& \leq \EE\sbr{R_h(x_h,a_h) \mid \rho_j,\Mcal_\Kcal} + HK\sqrt{\epstv} + \sum_{h'=0}^{h-1} \PP\sbr{(x_{h'},a_{h'}) \notin \Kcal_{h'} \mid \rho_j,\Mcal_\Kcal}\\
& \leq HK\sqrt{\epstv} + \sum_{h'=0}^h \PP\sbr{(x_{h'},a_{h'}) \notin \Kcal_{h'} \mid \rho_j,\Mcal_\Kcal}\\
& \leq HK\sqrt{\epstv} + \sum_{h'=0}^h \PP\sbr{(x_{h'},a_{h'}) \notin \Kcal_{h'} \mid \rho_j,\Mcal}\\
& = HK\sqrt{\epstv} + \sum_{h'=0}^h \PP\sbr{\nbr{\phi_{h'}^\star(x_{h'},a_{h'})}_{\Sigma_{h',j}^{-1}} \geq 1 \mid \rho_j,\Mcal}\\
& \leq HK\sqrt{\epstv} + \sum_{h'=0}^h \EE_{\rho_j}\sbr{\tr\rbr{\phi_{h'}^\star(x_{h'},a_{h'})\phi_{h'}^\star(x_{h'},a_{h'})^\top\Sigma_{h',j}^{-1}}}
\end{align*}
The last step follows from Markov's inequality.  Since both matrices
are positive semidefinite, the trace terms are all
non-negative. Therefore, by the pigeonhole principle, there exists
some $h' \in \{0,\ldots,h\}$ for which
\begin{align*}
\EE_{\rho_j}\sbr{\tr\rbr{\phi_{h'}^\star(x_{h'},a_{h'})\phi_{h'}^\star(x_{h'},a_{h'})^\top\Sigma_{h',j}^{-1}}} \geq \frac{1}{H}\rbr{\hat{\EE}_{\rho_j}\sbr{R_h(x_h,a_h)} - HK\sqrt{\epstv}}.
\end{align*}
This argument applies for all $R_h$, and so we obtain the lemma.
\end{proof}

Next we argue that there cannot be too many iterations for which
$\max_h \hat{\EE}_{\rho_j}\sbr{R_h(x_h,a_h)}$ is large. For notation,
here we use $\Mhat^{(j)}$ to denote the MDP model in iteration $j$ and
we use $R_h^{(j)}$ to denote the reward functions in~\pref{lem:case_1}
derived from the known sets in iteration $j$.
\begin{corollary}
\label{corr:case_1}
Assume that for each round $j \in [J_{\max}]$ and for all $h$ we have
$\err_h(\Sigma_{h-1,j}) \leq \epstv$. Set
\begin{align}
J_{\max} \defeq \frac{4Hd}{\lambda K\sqrt{\epstv}}\cdot \log\rbr{1 + \frac{4H}{\lambda K\sqrt{\epstv}}}.\label{eq:t_max}
\end{align}
Then there exists some $j \in [J_{\max}]$ for which $\max_h
\EE\sbr{R^{(j)}_h(x_h,a_h) \mid \rho_j, \Mhat^{(j)}} \leq
2HK\sqrt{\epstv}$.
\end{corollary}
\begin{proof}
Suppose that in round $j$, it holds that $\max_h
\EE\sbr{R^{(j)}_h(x_h,a_h) \mid \rho_j, \Mhat^{(j)}} \geq
2HK\sqrt{\epstv}$. Then, by~\pref{lem:case_1}, there exists some time
step $h$ for which
\begin{align*}
\tr\rbr{\EE_{\rho_j}\sbr{\phi_h^\star(x_h,a_h) \phi_h^\star(x_h,a_h)^\top}\Sigma_{h,j}^{-1}} \geq K\sqrt{\epstv}.
\end{align*}
Note that we also have $\Sigma_{h,j+1} = \Sigma_{h,j} +
\EE_{\rho_j}\sbr{\phi_h^\star(x_h,a_h) \phi_h^\star(x_h,a_h)^\top}$,
so we are in a position to apply the elliptical potential
argument. Specifically if $J$ is the number of iterations for which
the above inequality holds for some $h$, then
applying~\pref{lem:log_det} for each $h$ and summing across $h$ yields
\begin{align*}
JK\sqrt{\epstv} \leq (1 + \nicefrac{1}{\lambda})dH\log(1+\nicefrac{J_{\max}}{d})
\end{align*}
Plugging in our choice of $J_{\max}$, and using the fact that $\lambda < 1$ we have
\begin{align*}
J &< \frac{2Hd}{\lambda K\sqrt{\epstv}} \log\rbr{1 + \frac{4H}{\lambda K\sqrt{\epstv}} \log\rbr{1 + \frac{4H}{\lambda K\sqrt{\epstv}}}} \\
& \leq \frac{2Hd}{\lambda K\sqrt{\epstv}} \log\rbr{1 + \rbr{\frac{4H}{\lambda K\sqrt{\epstv}}}^2} \leq J_{\max}.
\end{align*}
This means that in $J_{\max}$ iterations, we can have $\max_h
\EE\sbr{R_h^{(j)}(x_h,a_h) \mid \rho_j,\Mhat^{(j)}}\geq
2HK\sqrt{\epstv}$ in at most $J<J_{\max}$ of them. Thus we must have
one where this quantity is small, which proves the lemma.
\end{proof}

Next, we state a guarantee provided by~\pref{alg:general_planner},
which is a more convenient form of~\pref{lem:general_planner}.
\begin{lemma}
\label{lem:planner_intermediate}
Fix any iteration $j$, time $h$, function $f:\Xcal\times\Acal \to [0,1]$,
policy $\pi$, any $\alpha > 0$. Then
\begin{align*}
\EE\sbr{ f(x_h,a_h) \mid \pi,\Mhat^{(j)}} \leq \frac{T\beta}{2\alpha} + \frac{\alpha d}{2T} + \frac{\alpha KH}{2} \EE\sbr{f(x_h,a_h) \mid \rho_j,\Mhat^{(j)}},
\end{align*}
where $T \leq 4d\log(1+\nicefrac{4}{\beta})/\beta$ and $\beta > 0$ is
the parameter to~\pref{alg:general_planner}.
\end{lemma}
\begin{proof}
We suppress the dependence on $j$. Let us first focus on $\rhopre$,
which is output of~\pref{alg:general_planner} for some time step
$h$. $\rhopre$ induces a distribution over states at time step $h$,
and we argue that this distribution adequately covers all possible
roll-in distributions in the model $\Mhat = \Mhat^{(j)}$. Consider any
function $f: \Xcal\times\Acal \to [0,1]$, any policy $\pi$, any
$\Sigma \succ 0$, and $\alpha > 0$. Calling $f_\pi(x_h) = \int \pi(a_h
\mid x_h) f(x_h)$, we have
\begin{align*}
\hat{\EE}_\pi f(x_h,a_h) &= \hat{\EE}_\pi \inner{\hat{\phi}_{h-1}(x_{h-1},a_{h-1})}{\int \hat{\mu}_{h-1}(x_h)f_\pi(x_h)}\\
& \leq \hat{\EE}_\pi \nbr{\hat{\phi}_{h-1}(x_{h-1},a_{h-1})}_{\Sigma^{-1}}\cdot \nbr{\int \hat{\mu}_{h-1}(x_h)f_\pi(x_h)}_{\Sigma}\\
& \leq \frac{1}{2\alpha}\hat{\EE}_\pi \nbr{\phi_{h-1}(x_{h-1},a_{h-1})}_{\Sigma^{-1}}^2 + \frac{\alpha}{2}\nbr{\int \hat{\mu}_{h-1}(x_h)f_\pi(x_h)}_{\Sigma}^2.
\end{align*}
Here we expand $\hat{T}_{h-1}$ in terms of its low rank representation
and then apply the Cauchy-Schwarz inequality in the norm induced by
$\Sigma$. Finally we use the AM-GM inequality which holds for any
non-negative $\alpha$.

We instantiate $\Sigma$ to be the covariance matrix induced by
$\rhopre$. First, for any policy $\pi$ we define
the $h-1$ step model covariance as $\Sigma_\pi \defeq \hat{\EE}_\pi
\hat{\phi}_{h-1}(x_{h-1},a_{h-1})
\hat{\phi}_{h-1}(x_{h-1},a_{h-1})^\top$, where the dependence on $h-1$
is suppressed in the notation. Note that both the expectation and the
embedding are taken with respect to the model $\Mhat$. Then, the
output of~\pref{alg:general_planner} is a $h$-step policy $\rhopre$
that is defined as a mixture over $T$ policies
$\pi_1,\ldots,\pi_T$. Using these policies, we define $\Sigma$ as
follows:
\begin{align*}
\Sigma = \Sigma_{\rhopre} + \frac{I_{d \times d}}{T} = \frac{1}{T} \sum_{t=1}^T \Sigma_{\pi_t} + \frac{I_{d \times d}}{T}.
\end{align*}
As we run~\pref{alg:general_planner} using $\hat{T}_{0:h-1}$ we can
apply~\pref{lem:general_planner} on the $h$ step MDP
$\hat{T}_{0:h-1}$. In other words, in~\pref{lem:general_planner}, we
set $H \gets h$ and $\Mtil \gets \Mhat$. The conclusion is that $T
\leq 4d \log (1 + \nicefrac{4}{\beta})/\beta$, where $\beta$ is the
parameter to the subroutine, and we can also bound the first term
above:
\begin{align*}
\hat{\EE}_\pi \nbr{\hat{\phi}_{h-1}(x_{h-1},a_{h-1})}^2_{\Sigma^{-1}}
 = \hat{\EE}_\pi \phi_{h-1}(x_{h-1},a_{h-1})^\top \rbr{\Sigma_{\rhopre} + \frac{I_{d\times d}}{T}}^{-1} \phi_{h-1}(x_{h-1},a_{h-1}) \leq T\beta.
\end{align*}
Next, we turn to the second term. Expanding the definition of $\Sigma$, we have
\begin{align*}
& \nbr{\int \hat{\mu}_{h-1}(x_h)f_\pi(x_h)}_{\Sigma}^2 \\
& = \hat{\EE}_{\rhopre} \rbr{\inner{\hat{\phi}_{h-1}(x_{h-1},a_{h-1})}{\int \hat{\mu}_{h-1}(x_h)f_\pi(x_h)}}^2 + \frac{\nbr{\int \hat{\mu}_{h-1}(x_h)f_\pi(x_h)}_2^2}{T}\\
& = \hat{\EE}_{\rhopre} \rbr{\hat{\EE}\sbr{f_\pi(x_h) \mid x_{h-1},a_{h-1}}}^2 + \frac{\nbr{\int \hat{\mu}_{h-1}(x_h)f_\pi(x_h)}_2^2}{T}\\
& \leq \hat{\EE}_{\rhopre} f_\pi(x_h) + \frac{\nbr{\int \hat{\mu}_{h-1}(x_h)f_\pi(x_h)}_2^2}{T} \leq \hat{\EE}_{\rhopre} f_\pi(x_h) + \nicefrac{d}{T}.
\end{align*}
The first inequality is Jensen's inequality along with the fact that
$f(x_h)^2 \leq f(x_h)$ since $f: \Xcal \to [0,1]$. The second
inequality is based on our normalization assumptions on $\mu_{h-1}$,
which we also impose on $\hat{\mu}_{h-1}$. Finally, collecting all the
terms and importance weighting the last action, we obtain the bound
\begin{align*}
\hat{\EE}_\pi f(x_h) \leq \frac{T\beta}{2\alpha} + \frac{\alpha K}{2} \hat{\EE}_{\rhopre\circ\unif(\Acal)}f(x_h,a_h) + \frac{\alpha d}{2T}.
\end{align*}
This bound applies to $\rhopre$. As $\rho_j$ is a uniform mixture of
these policies and as $f$ is non-negative, we see that
$\hat{\EE}_{\rhopre\circ\unif(\Acal)}f(x_h,a_h) \leq H \cdot \hat{\EE}_{\rho_j}
f(x_h,a_h)$, which proves the lemma.
\end{proof}

Finally, we use the guarantee for~\pref{alg:general_planner}, to prove
that our model $\Mhat$ universally approximate the true MDP as soon as
$\max_h \EE\sbr{R_h^{(j)}(x_h,a_h) \mid \rho_j,\Mhat^{(j)}} \leq
2HK\sqrt{\epstv}$. For the lemma, we use the concept of a sparse
reward function. $R: \Xcal \times \Acal \to [0,1]$ is called
\emph{sparse} if all value functions are in $[0,1]$. For example, this
holds if $R$ is only associated with state-action pairs at a single
time point.
\begin{lemma}
\label{lem:case_2}
Assume that for each round $j \in [J_{\max}]$ and for all $h$, we have
$\err_h(\Sigma_{h-1,j}) \leq \epstv$, and set $J_{\max}$ as
in~\pref{eq:t_max}. Then the final MDP model $\Mhat$ satisfies the
following guarantee: For any sparse reward function $R: \Xcal \times \Acal
\to [0,1]$, any policy $\pi$, and any $\alpha > 0$ we have
\begin{align*}
\abr{V(\pi;R,\Mhat) - V(\pi;R,\Mcal)} \leq H K\sqrt{\epstv} + H \epsescape,
\end{align*}
where $\epsescape \defeq \alpha H^2K^2\sqrt{\epstv} +
\frac{T\beta}{2\alpha} + \frac{\alpha d}{2T} + HK\sqrt{\epstv}$, $T
\leq 4d\log(1+\nicefrac{4}{\beta})/\beta$ and $\beta>0$ is the
parameter to~\pref{alg:general_planner}.
\end{lemma}
\begin{proof}
Via~\pref{corr:case_1}, there must be some round $j$ for which
\begin{align*}
\max_h \PP\sbr{(x_h,a_h)\notin \Kcal_h(\Sigma_{h,j}) \mid \rho_j,\Mhat^{(j)}} = \max_h\EE\sbr{R_h^{(j)}(x_h,a_h) \mid \rho_j,\Mhat^{(j)}}\leq
2HK\sqrt{\epstv}.
\end{align*}
We will prove the guarantee for this round $j$, and at the end of the
proof argue that this also applies to the final learned model.

Combining the lower bound of the simulation lemma
(\pref{lem:clipped_simulation}) with the planning guarantee
(\pref{lem:planner_intermediate}) we see that, for any policy $\pi$
\begin{align*}
& \PP\sbr{(x_h,a_h) \notin \Kcal_h(\Sigma_{h,j}) \mid \pi,\Mcal^{(j)}_{\Kcal}} \leq \PP\sbr{(x_h,a_h)\notin\Kcal_h(\Sigma_{h,j} \mid \pi,\Mhat^{(j)}} + HK\sqrt{\epstv}\\
& \leq \frac{\alpha KH}{2} \PP\sbr{(x_h,a_h)\notin\Kcal_h(\Sigma_{h,j}) \mid \rho_j,\Mhat^{(j)}} + \frac{T\beta}{2\alpha} + \frac{\alpha d}{2T} + HK\sqrt{\epstv}\\
& \leq \alpha K^2H^2\sqrt{\epstv} + \frac{T\beta}{2\alpha} + \frac{\alpha d}{2T} + HK\sqrt{\epstv} =: \epsescape
\end{align*}
Now that we have upper bounded the escaping probability, we can turn
to the approximation guarantee. While we are not in the exact setting
of~\pref{lem:clipped_simulation}, since we have a sparse reward
function, all values are in $[0,1]$ so the same argument applies. 
For one side of the error guarantee, since we assume
that $\err_h(\Sigma_{h-1,j}) \leq \epstv$ for all iterations, we have
\begin{align*}
V(\pi;R,\Mhat^{(j)}) &\leq V(\pi;R,\Mcal_{\Kcal}^{(j)}) + HK\sqrt{\epstv} + \sum_{h=0}^{H-1}\PP\sbr{(x_h,a_h)\notin \Kcal_h(\Sigma_{h,j}) \mid \pi,\Mcal^{(j)}_{\Kcal}}\\
&\leq V(\pi;R,\Mcal) + HK\sqrt{\epstv} + \sum_{h=0}^{H-1}\PP\sbr{(x_h,a_h)\notin \Kcal_h(\Sigma_{h,j}) \mid \pi,\Mcal^{(j)}_{\Kcal}}\\
& \leq V(\pi;R,\Mcal) + HK\sqrt{\epstv} + H\epsescape.
\end{align*}
Here the first inequality is~\pref{lem:clipped_simulation}, while the second is due to~\pref{eq:rmax_lemma}.
For the other direction, we first use~\pref{eq:rmax_lemma} and then~\pref{lem:clipped_simulation}:
\begin{align*}
V(\pi;R,\Mcal) &\leq V(\pi;R,\Mcal_{\Kcal}^{(j)}) + \sum_{h=0}^{H-1}\PP\sbr{(x_h,a_h)\notin \Kcal_h(\Sigma_{h,j}) \mid \pi,\Mcal^{(j)}_{\Kcal}}\\
& \leq V(\pi;R,\Mhat^{(j)}) + H K\sqrt{\epstv} + H \epsescape.
\end{align*}
This proves the result for the MDP model $\Mhat^{(j^\star)}$ at the
time $j^\star$ where the exploratory policy $\rho_{j^\star}$ fails to
achieve large reward on $R_h^{(j^\star)}$. We now claim this applies
for all iterations after $j^\star$, and in particular it holds at the
end of the algorithm. To see why, observe that $\Sigma_{h,j+1} \succeq
\Sigma_{h,j}$ for all $j$, and so $\Kcal_h(\Sigma_{h,j}) \subset
\Kcal_h(\Sigma_{h,j+1})$ for all rounds. Since the known set increases
with $j$, the escaping probability is decreasing, so it is upper
bounded by $\epsescape$ for all rounds after $j\geq
j^\star$. Additionally, we assume that $\err_h(\Sigma_{h-1,j}) \leq
\epstv$ for all $j \in [J_{\max}]$, so the total variation term
in~\pref{lem:clipped_simulation} remains bounded. Working through the
proof of~\pref{lem:clipped_simulation}, we can see that the bound
continues to hold for all $j^\star \leq j \leq J_{\max}$, which
proves the result.
\end{proof}

\paragraph{Final steps.}
Let us collect all of the conditions and bounds here. At the end of the algorithm, we have
\begin{align}
\max_{\pi,R} \abr{V(\pi;R,\Mhat) - V(\pi;R,\Mcal)} \leq H K\sqrt{\epstv} + H \epsescape, \label{eq:value_approx}
\end{align}
where we may set
\begin{align*}
\epsescape \defeq \min_{\alpha > 0} \cbr{\alpha H^2K^2\sqrt{\epstv} + \frac{T\beta}{2\alpha} + \frac{\alpha d}{2T} + HK\sqrt{\epstv}}, \qquad
T \leq \frac{4d\log(1+\nicefrac{4}{\beta})}{\beta},
\end{align*}
and $\beta > 0$ is the parameter to~\pref{alg:general_planner}. Applying~\pref{corr:mle} and taking
a union bound over all iterations $j \in [J_{\max}]$ and all times $h$, we can set
\begin{align*}
\epstv \defeq \lambda d + \frac{2 \log(J_{\max}H|\Phi||\Upsilon|/\delta)}{n},
\end{align*}
where $\lambda > 0$ is a parameter in the analysis. Finally, the total number of samples collected is
\begin{align*}
nHJ_{\max}, \qquad \textrm{where}, \qquad J_{\max} \defeq \frac{4Hd}{\lambda K\sqrt{\epstv}}\cdot \log\rbr{1 + \frac{4H}{\lambda K\sqrt{\epstv}}}
\end{align*}

We start by optimizing for $\alpha$ in the definition of $\epsescape$,
which yields $\alpha = \sqrt{\frac{T\beta}{2(H^2K^2\sqrt{\epstv} +
    d/(2T))}}$. Plugging into $\epsescape$ and using the bound on $T$,
we get
\begin{align*}
\epsescape & \leq \sqrt{2T\beta} \cdot \rbr{ HK\epstv^{1/4} + \sqrt{d/(2T)}} + HK\sqrt{\epstv}\\
& \leq 2 \sqrt{8 d \log(1+\nicefrac{4}{\beta})}HK\epstv^{1/4} + \sqrt{d\beta}.
\end{align*}
Here we are using the fact that $\epstv \leq 1$, which is without loss
of generality, since~\pref{eq:value_approx} is trivial when $\epstv
\geq 1$.  Now we set $\beta = H^2K^2\sqrt{\epstv}$ so that
\begin{align*}
\epsescape \leq 16 \sqrt{d \log (1+\nicefrac{4}{\epstv})} HK\epstv^{1/4}.
\end{align*}
Thus, we may restate the final accuracy guarantee as
\begin{align*}
\max_{\pi,R} \abr{V(\pi;R,\Mhat) - V(\pi;R,\Mcal)} &\leq H K\sqrt{\epstv} + 16 \sqrt{d\log(1+\nicefrac{4}{\epstv})}H^2 K\epstv^{1/4}\\
& \leq 17 \sqrt{d\log(1+\nicefrac{4}{\epstv})}H^2 K\epstv^{1/4}.
\end{align*}
We want this to be upper bounded by $\veps$, the final accuracy parameter, which means we can take
\begin{align*}
\epstv = c \frac{\veps^{4}H^{-8}K^{-4} d^{-2}}{\log^2(1 + \nicefrac{1}{\veps})},
\end{align*}
where $c>0$ is a universal constant. Looking at the definition of
$\epstv$ and $T_{\max}$, we set
\begin{align*}
\lambda = c \frac{\veps^4 H^{-8}K^{-4}d^{-3}}{\log^2(1+\nicefrac{1}{\veps})}, \quad T_{\max} = \tilde{O}\rbr{ \frac{H^{13}d^5 K^5}{\veps^6}}, \quad n = \tilde{O}\rbr{ \frac{H^{8}K^4d^2 \log(|\Phi||\Upsilon|/\delta)}{\veps^4}},
\end{align*}
where we are ignoring logarithmic factors. This gives the final sample complexity of
\begin{align*}
\tilde{O}\rbr{ \frac{H^{22} K^{9} d^{7} \log (|\Phi||\Upsilon|/\delta)}{\veps^{10}}}.
\end{align*}
Finally, note that~\pref{eq:sys_id} is implied by the final accuracy
guarantee, since we may choose $R$ to be the total variation distance
between our model and the true transition dynamics at time $h$, which
is clearly a sparse reward function.

\paragraph{Analysis with the sampling oracle.}
With the sampling oracle, the argument is very similar. The main
difference is in~\pref{lem:planner_intermediate}, which is the only
place where we use the exact planner. Instead, we modify the proof
of~\pref{lem:planner_intermediate} to instead
use~\pref{lem:gen_sampling_planner} to obtain $\hat{\Sigma}$ and
$\rhopre$, and we do the Cauchy-Schwarz step using $\hat{\Sigma}$.
By~\pref{lem:gen_sampling_planner} the first term is still $O(T\beta)$
and for the second term we pay an additive $O(\beta)$ to translate
from $\hat{\Sigma}$ to $\Sigma$ (since the spectral norm error is
$O(\nicefrac{\beta}{d})$ and Euclidean norm of the term involving
$\hat{\mu}_{h-1}$ is at most $d$). This we have an additional
$O(\alpha\beta)$ term in the sample-based analog
of~\pref{lem:planner_intermediate}. However, as we use the bound $T
\leq O(d\log(1+\nicefrac{1}{\beta})/\beta)$ in the remaining
calculations, the new $O(\alpha \beta)$ term is only larger than the
$O(\alpha d/T)$ term by a logarithmic factor. In particular, above we
have a $\sqrt{d\beta}$ term in the bound for $\epsescape$, but with
the sampling oracle, we will additionally have a $O(\sqrt{T}\beta) =
O(\sqrt{d\beta\log(1+\nicefrac{1}{\beta})})$ term in this bound.
Ultimately, this only affects the final sample complexity bound in
logarithmic factors.  We adjust the failure probability accordingly,
using $\delta/2$ probability for invocations
of~\pref{lem:gen_sampling_planner}, and $\delta/2$ for the invocations
of~\pref{corr:mle}. As we invoke the planner polynomially many times,
the total number of calls to the sampling oracle is polynomial in all
parameters.

\subsection{Refined analysis for simplex representations.}
Here we prove~\pref{thm:improved} by considering a different potential
function argument and a different instantiation of the planning
algorithm that directly attempts to visit each latent state. In
particular, we instantiate~\pref{alg:main} with the planning routine
presented in~\pref{alg:simplex_planner}. Note that this planner does
not require the parameter $\beta$, but it does assume that
$\hat{\phi}(x,a) \in \Delta([\dlv])$ for each $(x,a)$.

The analog of $\Sigma_{h,j}$ is the cumulative probability of hitting
latent variables $z_{h} \in \Zcal_{h}$. Formally, we define
\begin{align*}
p_{h,j}(z) \defeq \sum_{i=0}^{j-1} \PP\sbr{z_{h} = z \mid \rho_i,\Mcal}.
\end{align*}
We make two remarks. First $p_{h,j}$ is not a distribution, rather it
is the sum of $j$ probability distributions. Second, we use $p_{h,j}$
to measure the coverage at time $h$, since $z_{h}$ is the latent
variable that generates $x_{h}$. This indexing is different from how
we use $\Sigma_{h,j}$ to measure coverage at time $h+1$ in the general
case.

We now state the analog of~\pref{corr:mle}.
\begin{corollary}
\label{corr:simplex_mle}
For $j \geq 1, h \in \{0,\ldots,H-1\}, \delta \in (0,1)$ and let
$\rho_0,\ldots,\rho_{j-1}$ be any (possibly data-dependent) policies,
with $p_{h,j}$ defined accordingly. Let $D_h$ be a dataset of $nj$
examples, where for each $0\leq i < j$, we collect $n$ triples
$(x_h,a_h,x_{h+1})$ by rolling in with $\rho_i$ to $x_h$ and taking
$a_h$ uniformly at random. Then, with probability at least $1-\delta$
the output $(\hat{\phi}_h,\hat{\mu}_h)$ of $\mle(D_h)$ satisfies
\begin{align*}
\sum_{z \in \Zcal_h} p_{h,j}(z) \EE\sbr{\nbr{\inner{\hat{\phi}(x_h,a_h)}{\hat{\mu}_h(\cdot)} - T_h(\cdot \mid x_h,a_h)}^2_{\tv} \mid z_h = z, a_h \sim \unif(\Acal)} \leq \frac{2 \log(|\Phi||\Upsilon|/\delta)}{n}.
\end{align*}
\end{corollary}
\begin{proof}
This is an immediate consequence of~\pref{thm:mle}, using the
definition of $p_{h,j}$. 
\end{proof}

We denote the LHS of the above lemma as $\err_h(p_{h,j})$.  Now we
define the known set $\Kcal_h$ and the absorbing MDP. In the simplex
features setting, the known set $\Kcal_h$ is instead defined in terms
of latent variables. Recall that we can augment every trajectory
$\tau$ with the latent variables generated along the trajectory that
is $\tau=(z_0,x_0,a_0,
z_1,x_1,a_1,\ldots,z_{H-1},x_{H-1},a_{H-1})$. We therefore define
$\Kcal_{h,j} \defeq \{z \in \Zcal_h: p_{h,j}(z) \geq \Delta\}$ where
$\Delta$ is some parameter we will set towards the end of the
proof. The absorbing MDP $\Mcal_{\Kcal}$ in iteration $j$ is defined
to have transition operator that, for each $h$, transitions from
$z_{h}$ to $\xabsorb$ if $z_h \notin \Kcal_{h,j}$ and otherwise
transitions as in $\Mcal$. As in the more general analysis, $\xabsorb$
is an absorbing state with a single self-looping action $\aabsorb$ and
we always consider $(\xabsorb,\aabsorb)$ to be known.

We now state the analog of~\pref{lem:clipped_simulation}.
\begin{lemma}
\label{lem:simplex_clipped_simulation}
Let $\hat{\phi}_{0:H-1},\hat{\mu}_{0:H-1}$ be an MDP model with
simplex features and let $p_{0:H-1}$ be non-negative vectors. Assume
that $\err_h(p_{h-1}) \leq \epstv$ for each $h$. Let $f:\Xcal
\times \Acal \to [0,1]$ be any function such that
$f(\xabsorb,\aabsorb) = 0$ and let $\pi$ be any policy. Then, for any $h$
\begin{align*}
\EE_\pi\sbr{f(x_h,a_h) \mid \Mcal_{\Kcal}} - HK\sqrt{\nicefrac{\epstv}{\Delta}} \leq \hat{\EE}_\pi\sbr{f(x_h,a_h)} &\leq \EE_\pi\sbr{f(x_h,a_h)\mid \Mcal_\Kcal} + HK\sqrt{\nicefrac{\epstv}{\Delta}} \\
& ~~~~~~~~~~~~~~ + \sum_{h'=1}^{h}\PP\sbr{z_{h'}\notin \Kcal_{h'} \mid \pi,\Mcal_\Kcal}
\end{align*}
\end{lemma}
\begin{proof}
As in the proof of~\pref{lem:clipped_simulation}, we must control two terms for each $h' < h$:
\begin{align*}
\mathrm{Term 1}_{h'} \defeq \EE\sbr{ \int (\hat{T}_{h'}(x_{h'+1} \mid x_{h'},a_{h'}) - T_{h'}(x_{h'+1} \mid x_{h'},a_{h'})) \hat{V}_{h'+1}(x_{h'+1}) \mid \pi,\Mcal_{\Kcal}}\\
\mathrm{Term 2}_{h'} \defeq \EE\sbr{ \int (T_{h'}(x_{h'+1} \mid x_{h'},a_{h'}) - T_{h',\Kcal}(x_{h'+1} \mid x_{h'},a_{h'})) \hat{V}_{h'+1}(x_{h'+1}) \mid \pi,\Mcal_{\Kcal}}.
\end{align*}
For $\mathrm{Term 1}_{h'}$, as we are in $\Mcal_{\Kcal}$ we can ignore
the trajectories where $z_{h'} \notin \Kcal_{h'}$. Thus considering
$h'=1$
\begin{align*}
\mathrm{Term 1}_{1} & =\sum_{z \in \Kcal_{1}} \PP_\pi[z_{1}=z \mid \Mcal_\Kcal]\cdot\EE_\pi\sbr{ \int (\hat{T}_{1}(x_{2} \mid x_{1},a_{1}) - T_{1}(x_{2} \mid x_{1},a_{1})) \hat{V}_{2}(x_{2}) \mid z_{1}=z}\\
& \leq K \sum_{z \in \Kcal_{1}} \PP_\pi[z_{1}=z \mid \Mcal_\Kcal]\cdot \EE\sbr{ \nbr{ \hat{T}_{1}(\cdot \mid x_{1},a_{1}) - T_{1}(\cdot \mid x_{1},a_{1})}_{\tv}  \mid z_{1}=z,a_{1}\sim\unif(\Acal)}\\
& \leq K \sqrt{\sum_{z \in \Kcal_{1}} \PP_\pi[z_{1}=z \mid \Mcal_\Kcal]\cdot\EE\sbr{ \nbr{ \hat{T}_{1}(\cdot \mid x_{1},a_{1}) - T_{1}(\cdot \mid x_{1},a_{1})}^2_{\tv}  \mid z_{1}=z,a_{1}\sim\unif(\Acal)}}\\
& \leq K \sqrt{\nicefrac{\err_{1}(p_{1})}{\Delta}} \leq K\sqrt{\nicefrac{\epstv}{\Delta}}. 
\end{align*}
Here we are using that the total variation term is non-negative, and
that $\PP_\pi[z_{1}=z \mid \Mcal_\Kcal] \leq 1$, while $p_{1}(z) \geq
\Delta$ by the fact that $z \in \Kcal_{1}$.  This argument applies for
all $h'$ and yields the $HK\sqrt{\nicefrac{\epstv}{\Delta}}$ term on
both sides of the statement. For $\mathrm{Term 2}_{h'}$, we clearly have
\begin{align*}
\mathrm{Term 2}_{h'} \leq \PP\sbr{z_{h'+1}\notin\Kcal_{h'+1} \mid \pi,\Mcal_{\Kcal}}.
\end{align*}
As in the proof of~\pref{lem:clipped_simulation}, $\mathrm{Term 2}_{h'} \geq 0$ which yields the lower bound. 
\end{proof}

Next we argue that if the exploratory policy $\rho_j$ that we find has
large escaping probability then we will add some latent variable to
the known set in the next iteration.
\begin{lemma}
\label{lem:simplex_case_1}
Consider iteration $j$ of \ouralg and assume that $\err_h(p_{h,j})
\leq \epstv$ for each $h$. Define $R_h(x,a) \defeq \sum_{z \notin
  \Kcal_{h,j}} \phi^\star_h(x,a)[z]$. Then
\begin{align*}
\max_h \PP\sbr{z_{h} \notin \Kcal_{h,j} \mid \rho_j} \geq \frac{1}{H} \max_h \cbr{\hat{\EE}_{\rho_j}[R_h(x_h,a_h)] - HK\sqrt{\nicefrac{\epstv}{\Delta}}}.
\end{align*}
In particular, if there exists some $h$ such that
$\hat{\EE}_{\rho_j}[R_h(x_h,a_h)] \geq HK\sqrt{\nicefrac{\epstv}{\Delta}} + H\dlv
\Delta$, then there exists some $h', z \notin \Kcal_{h'}$ such that
$\PP\sbr{z_{h'} = z \mid \rho_j} \geq \Delta$.
\end{lemma}
\begin{proof}
Observe that by the definition of $R_h$, we have
\begin{align*}
\hat{\EE}_{\rho_j}\sbr{R_h(x_h,a_h)} &\leq \EE_{\rho_j}[R_h(x_h,a_h) \mid \Mcal_{\Kcal}] + HK\sqrt{\nicefrac{\epstv}{\Delta}} + \sum_{h'=1}^h \PP\sbr{ z_{h'} \notin \Kcal_{h',j}\mid \rho_j,\Mcal_{\Kcal}}\\
& = HK\sqrt{\nicefrac{\epstv}{\Delta}} + \sum_{h'=1}^{h+1} \PP\sbr{z_{h'} \notin \Kcal_{h',j} \mid \rho_j,\Mcal_{\Kcal}}.
\end{align*}
Both statements now follow from the pigeonhole principle.
\end{proof}

Next we prove the analog of~\pref{lem:case_2}. For this, we compute
$\rhopre$ using the planning routine in~\pref{alg:simplex_planner},
with the planning guarantee in~\pref{lem:simplex_planner}.
\begin{lemma}
\label{lem:simplex_case_2}
Assume that for each round $j \in [J_{\max}]$ and for all $h$, we have
$\err_h(p_{h,j}) \leq \epstv$ and set $J_{\max} = H\dlv+1$. Then the
final MDP model $\Mhat$ satisfies the following guarantee: For any
sparse reward function $R:\Xcal \times \Acal \to [0,1]$ and any policy
$\pi$, we have
\begin{align*}
\abr{V(\pi;R,\Mhat) - V(\pi;R,\Mcal)} \leq H K\sqrt{\nicefrac{\epstv}{\Delta}} + H \epsescape,
\end{align*}
where $\epsescape \defeq H^2K\dlv^2\Delta + \rbr{H^2K^2\dlv + HK\dlv}\sqrt{\nicefrac{\epstv}{\Delta}}$.
\end{lemma}
\begin{proof}
First observe that by~\pref{lem:simplex_case_1}, in every iteration
where $\rho_j$ satisfies $\max_h \hat{\EE}_{\rho_j}\sbr{R_h(x_h,a_h)}
\geq HK\sqrt{\nicefrac{\epstv}{\Delta}} + H\dlv\Delta$, we add some
latent variable at some time point to the known set. This means that
this can only happen for at most $H\dlv$ iterations, and so, by the
setting of $J_{\max}$, there must be some iteration $j$ in which
$\max_h\hat{\EE}_{\rho_j}\sbr{R_h(x_h,a_h)} \leq
HK\sqrt{\nicefrac{\epstv}{\Delta}} + H\dlv\Delta$. In this iteration
$j$, we have
\begin{align*}
& \PP\sbr{z_{h+1} \notin \Kcal_{h+1,j} \mid \pi,\Mcal_{\Kcal}^{(j)}} \leq \PP\sbr{z_{h+1} \notin \Kcal_{h+1,j} \mid \pi, \Mhat^{(j)}} + HK\sqrt{\nicefrac{\epstv}{\Delta}}\\
& = \sum_{z \in \hat{\Zcal}_h} \hat{\EE}_\pi\sbr{\hat{\phi}_{h-1}(x_{h-1},a_{h-1})[z]} \cdot \PP\sbr{\bar{\Kcal}_{h+1,j} \mid \pi,\Mhat^{(j)}, \hat{z}_h = i} + HK\sqrt{\nicefrac{\epstv}{\Delta}}\\
& \leq K \sum_{z \in \hat{\Zcal}_h} \hat{\EE}_\pi\sbr{\hat{\phi}_{h-1}(x_{h-1},a_{h-1})[z]} \cdot \PP\sbr{\bar{\Kcal}_{h+1,j} \mid \Mhat^{(j)}, \hat{z}_h = z,a_h \sim \unif(\Acal)} + HK\sqrt{\nicefrac{\epstv}{\Delta}}\\
& \leq K\dlv \sum_{z \in \hat{\Zcal}_h} \hat{\EE}_{\rhopre}\sbr{\hat{\phi}_{h-1}(x_{h-1},a_{h-1})[z]} \cdot \PP\sbr{\bar{\Kcal}_{h+1,j} \mid \Mhat^{(j)}, \hat{z}_h = z,a_h \sim \unif(\Acal)} + HK\sqrt{\nicefrac{\epstv}{\Delta}}\\
& \leq HK\dlv \hat{\PP}_{\rho_j}\sbr{\bar{\Kcal}_{h+1,j}} + HK\sqrt{\nicefrac{\epstv}{\Delta}} \leq HK\dlv \hat{\EE}_{\rho_j}\sbr{R_h(x_h,a_h)} + HK\sqrt{\nicefrac{\epstv}{\Delta}}\\
& \leq HK\dlv \rbr{HK\sqrt{\nicefrac{\epstv}{\Delta}} + H\dlv\Delta} + HK\sqrt{\nicefrac{\epstv}{\Delta}} =: \epsescape
\end{align*}
Here the first inequality is~\pref{lem:simplex_clipped_simulation},
while the first equality re-writes the expectation in terms of the latent
variable $z_h$. In the second inequality we translate to taking $a_h$
uniformly via importance weighting, while in the third, we
apply~\pref{lem:simplex_planner}, which lets us translate to
$\rhopre$. Finally, we use that $\rho_j$ uses $\rhopre$ with
probability $\nicefrac{1}{H}$ and the definition of $R_h$.  The result
now follows from~\pref{lem:simplex_clipped_simulation}, along with the
analog of~\pref{eq:rmax_lemma}. As in the general case, this bound
applies for all iterations after the first one where the
escaping probability for $\rho_j$ is small.
\end{proof}

\paragraph{Final steps.}
The final steps with simplex features are much more straightforward
than in the general case. First we choose $\Delta$ to balance the two
terms in $\epsescape$. We set $\Delta =
(\nicefrac{K}{\dlv})^{2/3}\epstv^{1/3}$ which yields
\begin{align*}
\epsescape \leq 3H^2K\dlv\rbr{K^{2/3}(\dlv\epstv)^{1/3}},
\end{align*}
where we are also using the fact that $\epstv \leq
1$. Via~\pref{lem:simplex_case_2}, after $J_{\max} = H\dlv + 1$
iterations, we are guaranteed that
\begin{align*}
\max_{\pi,R}\abr{V(\pi;R,\Mhat) - V(\pi;R,\Mcal)} \leq H K\sqrt{\nicefrac{\epstv}{\Delta}} + H \epsescape \leq \order\rbr{H^3K^{5/3}\dlv^{4/3}\epstv^{1/3} }.
\end{align*}
For this to be at most $\veps$ we should set $\epstv \leq
\order\rbr{\veps^3 H^{-9} K^{-5}
  \dlv^{-4}}$. Applying~\pref{corr:simplex_mle} and taking a union
over all $J_{\max}$ rounds, we want to set
\begin{align*}
n = \frac{2\log(J_{\max} |\Phi||\Upsilon|/\delta)}{\epstv} = \tilde{\order}\rbr{\frac{H^{9}K^5 \dlv^4\log(|\Phi||\Upsilon|/\delta)}{\veps^3} }.
\end{align*}
The total sample complexity is $nHJ_{\max} =
\tilde{\order}\rbr{\frac{H^{11}K^5\dlv^5\log(|\Phi||\Upsilon|/\delta)}{\veps^3}}$.
As in the general setting, the value function guarantee
implies~\pref{eq:sys_id}, which yields the result. 

\paragraph{Analysis with a sampling oracle.}
With a sampling oracle, the only difference is in the proof
of~\pref{lem:simplex_case_2}. Here, we can only apply the second
statement of~\pref{lem:simplex_planner}, which yields an additive
$O(K\dlv\epsopt)$ term. By taking $\epsopt =
(\nicefrac{H}{\dlv})\sqrt{\nicefrac{\epstv}{\Delta}}$, this additional
term can be absorbed into the other additive term at the expense of a
constant. Thus we obtain the same guarantee, up to constants, with
polynomially many calls to the sampling oracle.

\section{Planning Algorithms}
\label{app:planners}
In this section, we present exploratory planning algorithms for low
rank models, assuming that the dynamics are known. Formally, we
consider an $H$ step low rank MDP $\Mtil$ with deterministic start
state $x_0$, fixed action $a_0$, and transition matrices
$T_0,\ldots,T_{H-1}$. Each transition operator $T_h$ factorizes as
$T_h(x_{h+1} \mid x_h,a_h) = \inner{\phih}{\muh}$ and we assume
$\phi_{0:H-1},\mu_{0:H-1}$ are \emph{known}. To compartmentalize the
results, we focus on exploratory planning at time $H$, but we will
invoke these subroutines with MDP models that have horizon $h \leq
H$. This simply requires rebinding variables.

We present two types of results. One style assumes that all
expectations are computed exactly. As we are focusing purely on
planning with known dynamics and rewards, this imposes a computational
burden, but not a statistical one, while leading to a more transparent
proof. To address the computational burden, we also consider
algorithms that approximate all expectations with samples. For this,
we assume that we can obtain sample transitions from the MDP model
$\Mtil$ in a computationally efficient manner. Formally, the
\emph{sampling oracle} allows us to sample $x' \sim T_h(\cdot \mid
x,a)$ for any $x,a$.

\subsection{Planning with a sampling oracle}
For the computational style of result, it will be helpful to first
show how to optimize a given reward function whenever the model admits
a sampling oracle. As notation, we always consider an explicitly
specified non-stationary reward function $R:\Xcal \times
\Acal\times\{0,\ldots,H-1\} \to [0,1]$. Then, we define
\begin{align*}
V(\pi,R) = \EE\sbr{\sum_{h=0}^{H-1}R(x_h,a_h,h) \mid \pi,\Mtil}.
\end{align*}

The next lemma is a simple application of the result
of~\cite{jin2019provably}.
\begin{lemma}
\label{lem:sampling_planner}
Suppose that the reward function
$R:\Xcal\times\Acal\times\{0,\ldots,H-1\} \to [0,1]$ is explicitly
given and that $T_{0:H-1}$ is a known low rank MDP that enables
efficient sampling. Then for any $\epsilon > 0$ there is an algorithm
for finding a policy $\hat{\pi}$ such that with probability at least
$1-\delta$, $V(\hat{\pi},R) \geq \max_\pi V(\pi,R) - \epsilon$ in
polynomial time with $\poly(d,H,\nicefrac{1}{\epsilon},\log(\nicefrac{1}{\delta}))$ calls to the
sampling routine.
\end{lemma}
\begin{proof}
As we have sampling access to the MDP, we can execute the \ucbvi
algorithm of~\citet{jin2019provably}. For any $n$, if we execute the algorithm
for $n$ episodes, it produces $n$ policies $\pi_1,\ldots,\pi_n$ and
guarantees
\begin{align*}
\max_\pi V(\pi,R) - \frac{1}{n}\sum_{i=1}^nV(\pi_i,R) \leq c\sqrt{\frac{d^3 H^3\log(ndH/\delta)}{n}}
\end{align*}
with probability at least $1-\delta$ where $c>0$ is a universal
constant.
We are assured that one of the policies $\pi_1,\ldots,\pi_n$ is at
most $\nicefrac{\epsilon}{2}$-suboptimal by taking $n = \order\rbr{
  d^3H^3\log(dH/(\epsilon\delta))/\epsilon^{2}}$.

We find this policy via a simple policy evaluation step. For each
policy $\pi_i$, we collect $O(H^2 \log(n/\delta)/\epsilon^2)$
roll-outs using the generative model, where we take actions according
to $\pi_i$. Via a union bound, this guarantees that for each $i$ we
have $\hat{V}_i$ such that with probability at least $1-\delta$
\begin{align*}
\max_i \abr{\hat{V}_i - V(\pi_i,R)} \leq \nicefrac{\epsilon}{4}.
\end{align*}
Therefore, if we take $\hat{i} = \argmax_{i \in [n]} \hat{V}_i$ we are
assured that $V(\pi_{\hat{i}},R) \geq \max_\pi V(\pi,R) - \epsilon$ with
probability at least $1-2\delta$.
The total number of samples required from the model are
\begin{align*}
nH \rbr{1 + \frac{H^2\log(n/\delta)}{\epsilon^2}} = \otil\rbr{\frac{d^3 H^6 \log(1/\delta)}{\epsilon^4}}.\tag*\qedhere
\end{align*}
\end{proof}

\subsection{Planning with simplex features}
We first consider a simpler planning algorithm that is adapted to the
simplex features representation. The pseudocode is displayed
in~\pref{alg:simplex_planner}. The planner computes a mixture policy
$\rho$, where component $\pi_i$ of the mixture focuses on activating
coordinate $i$ of the feature map $\phiH$. Each mixture component can
be computed in a straightforward manner using a dynamic programming
approach, such as LSVI. The basic guarantee for this algorithm is the
following lemma.

\begin{algorithm}[t]
\begin{algorithmic}
\STATE \textbf{Input:} MDP $\Mtil =
(\phi_{0:{H-1}}, \mu_{0:H-1})$ with $\phih \in \Delta([\dlv])$,
$\mu_{h}[z] \in \Delta(\Xcal)$.
\FOR{$z = 1,\ldots,\dlv$}
\STATE Compute
$\pi_z = \argmax_{\pi} \EE[ \phi_{H-1}(x_{H-1},a_{H-1})[z] \mid \Mtil,
  \pi]$
\ENDFOR
\STATE Output policy mixture $\rho \defeq \unif(\{\pi_z\}_{z=1}^{\dlv})$
\end{algorithmic}
\caption{Exploratory planner for simplex representations}
\label{alg:simplex_planner}
\end{algorithm}

\begin{lemma}[Guarantee for~\pref{alg:simplex_planner}]
\label{lem:simplex_planner}
If $\Mtil$ is an $H$-step low rank MDP with simplex features of
dimension $\dlv$, then the output of~\pref{alg:simplex_planner}, $\rho$, satisfies
\begin{align*}
\forall \pi, z \in [\dlv]: \EE\sbr{ \phiH[z] \mid \Mtil,\pi} \leq \dlv \EE\sbr{\phiH[z] \mid \Mtil,\rho}.
\end{align*}
Given a sampling oracle for $\Mtil$, the algorithm runs
in polynomial time with
$\poly(\dlv,H,\nicefrac{1}{\epsopt},\log(1/\delta))$ calls to \samp, and
with probability at least $1-\delta$, $\rho$ satisfies
\begin{align*}
\forall \pi, z \in [\dlv], \EE\sbr{ \phiH[z] \mid \Mtil,\pi} \leq \dlv \EE\sbr{\phiH[z] \mid \Mtil,\rho} + \epsopt.
\end{align*}
\end{lemma}
\begin{proof}
The first result follows immediately from the non-negativity of
$\phiH[i]$, the optimality property of $\pi_i$ and the definition of
$\rho$.

For the second result, by~\pref{lem:sampling_planner} we can optimize
any explicitly specified reward function using a polynomial number of
samples. If we call this sampling-based planner for each of the $d$
reward functions, with high probability (via a union bound) the
policies $\hat{\pi}_i$ are near-optimal for their corresponding reward
functions. By appropriately re-scaling the accuracy parameter
in~\pref{lem:sampling_planner} we obtain the desired guarantee.
\end{proof}

\subsection{Elliptical planner}

The next planning algorithm applies to general low rank MDP, and it is
more sophisticated. It proceeds in iterations, where in iteration $t$
we maintain a covariance matrix $\Sigma_{t-1}$ and,
in~\pref{eq:planning_optimization}, we search for a policy that
maximizes quadratic forms with the inverse covariance
$\Sigma_{t-1}^{-1}$. With a sampling oracle this optimization can be done via a call to~\pref{lem:sampling_planner}.
If this maximizing policy $\pi_t$ cannot achieve large quadratic forms
against $\Sigma_{t-1}^{-1}$, then we halt and output the mixture of
all previous policies. Otherwise, we mix $\pi_t$ into our candidate
solution, update the covariance matrix accordingly, and advance to the
next iteration. The performance guarantee for this algorithm is as
follows.
\begin{lemma}[Guarantee for~\pref{alg:general_planner}]
\label{lem:general_planner}
If $\Mtil$ is an $H$-step low rank MDP with embedding dimension $d$
then for any $\beta > 0$,~\pref{alg:general_planner} terminates after
at most $T+1$ iterations where $T \leq 4d
\log(1+\nicefrac{4}{\beta})/\beta$. Upon termination, $\rho$ guarantees
\begin{align*}
\forall \pi: \EE\sbr{ \phiH^\top\rbr{\Sigma_{\rho} + I/T}^{-1} \phiH \mid \Mtil,\pi} \leq T\beta.
\end{align*}
where $\Sigma_{\rho} = \frac{1}{T}\sum_{t=1}^T \Sigma_{\pi_t}$.
\end{lemma}
\begin{proof}
The performance guarantee is immediate from the termination condition,
using the fact that $\Sigma_T = T\cdot (\Sigma_{\rho} + I/T)$.

For the iteration complexity bound, we condense the notation and omit
the dependence on $H-1$, $x_{H-1},a_{H-1}$ in all terms. We have
\begin{align*}
\beta T \leq \sum_{t=1}^T \EE\sbr{ \phi^\top \Sigma_{t-1}^{-1} \phi \mid \Mtil, \pi_t} = \sum_{t=1}^T\tr (\Sigma_{\pi_t} \Sigma_{t-1}^{-1}) \leq 2d \log(1+\nicefrac{T}{d}),
\end{align*}
where the first inequality is based on the fact that we did not
terminate at each iteration $t \in [T]$ and the last inequality
follows from a standard elliptical potential argument (e.g., Lemma 11
in~\citet{dani2008stochastic}; see~\pref{lem:log_det} for a precise
statement and proof). This gives an upper bound on $T$ that implies
the one in the lemma statement, via~\pref{corr:log_det}.
\end{proof}

With the sampling oracle, we modify the algorithm slightly and obtain
a qualitatively similar guarantee. The modifications are discussed in
the proof.
\begin{lemma}
\label{lem:gen_sampling_planner}
The sample-based version of~\pref{alg:general_planner} has the
following guarantee. Assume $\Mtil$ is an $H$-step low rank MDP with
embedding dimension $d$ and fix $\beta > 0$, $\delta \in (0,1)$. Then
the algorithm terminates after at most $T+1$ iterations, where $T \leq
O(d \log (1+\nicefrac{1}{\beta})/\beta)$. Upon termination, it ouputs
a matrix $\hat{\Sigma}$ and a policy $\rho$ such that with probability
at least $1-\delta$:
\begin{align*}
\forall \pi: \EE\sbr{ \phiH^\top\rbr{\hat{\Sigma} + I/T}^{-1} \phiH \mid \Mtil,\pi} \leq O(T\beta),\\
\nbr{ \hat{\Sigma} - \rbr{\EE\sbr{ \phiH\phiH^\top\mid \rho,\Mtil} + I/T}}_{\mathrm{op}} \leq O(\nicefrac{\beta}{d}).
\end{align*}
The algorithm runs in polynomial time with
$\poly(d,H,\nicefrac{1}{\beta},\log(1/\delta))$ calls to the sampling
oracle.
\end{lemma}
\begin{proof}
The algorithm is modified as follows. We replace all covariances with
empirical approximations, obtained by calls to the sampling
subroutine. We call the empirical versions
$\hat{\Sigma}_t,\hat{\Sigma}_{\pi_t}$, etc. Then, the policy
optimization step~\pref{eq:planning_optimization} is performed via an
application of~\pref{lem:sampling_planner} and so we find an
$\epsopt$-suboptimal policy $\pi_t$ for the reward function induced by
$\hat{\Sigma}_{t-1}$. Then we use the sampling subroutine to estimate
the value of this policy, which we denote $\hat{V}_t(\pi_t)$. As
before, we terminate if $\hat{V}_t(\pi_t) \leq \beta$. If we terminate
in round $t$, we output $\rho = \unif(\{\pi_i\}_{i=1}^{t-1})$ and we
also output $\hat{\Sigma} =
\frac{1}{t-1}\sum_{i=1}^{t-1}\hat{\Sigma}_{\pi_i}$. As
notation, we use $V_t(\pi)$ to denote the value for policy $\pi$ on
the reward function used in iteration $t$, which is induced by
$\hat{\Sigma}_{t-1}$.

With $\poly(d,H,T,\nicefrac{1}{\epsopt},\log(\nicefrac{1}{\delta}))$
calls to the sampling subroutine and assuming the total number of
iterations of the algorithm $T$ is polynomial, we can verify that with
probability $1-\delta$
\begin{align*}
\max_{t \in [T]}\max\cbr{ d \cdot \nbr{\hat{\Sigma}_{\pi_t} - \Sigma_{\pi_t}}_{\mathrm{op}}, \abr{\hat{V}_t(\pi_t) - V_t(\pi_t)}, \max_\pi V_t(\pi) - V_t(\pi_t)} \leq \epsopt.
\end{align*}
The first two bounds follow from standard concentration of measure
arguments. The final one is based on an application
of~\pref{lem:sampling_planner}.

Now, if we terminate in iteration $t$, we know that $\hat{V}_t(\pi_t) \leq \beta$. This implies
\begin{align*}
\max_\pi V_t(\pi) \leq V_t(\pi_t) + \epsopt \leq \hat{V}_t(\pi_t) + 2\epsopt \leq \beta + 2\epsopt.
\end{align*}
As we are interested in the reward function induced by
$\hat{\Sigma}_{t-1}$, this verifies the quality guarantee, provided $\epsopt = O(\beta)$.

Finally, we turn to the iteration complexity. Similarly to above, we have
\begin{align*}
T\rbr{\beta - 2\epsopt} &\leq \sum_{t=1}^T \hat{V}_t(\pi_t) - 2\epsopt \leq \sum_{t=1}^T V_t(\pi_t) - \epsopt\\
& = \sum_{t=1}^T \EE\sbr{\phi^\top \hat{\Sigma}_{t-1}^{-1}\phi \mid \Mtil,\pi_t} - \epsopt = \sum_{t=1}^T \tr(\Sigma_{\pi_t}\hat{\Sigma}_{t-1}^{-1}) - \epsopt\\
&\leq \sum_{t=1}^T \tr(\hat{\Sigma}_{\pi_t} \hat{\Sigma}_{t-1}^{-1}) \leq 2d \log(1 + \nicefrac{T}{d}).
\end{align*}
In other words, if we set $\epsopt = O(\beta)$ then both the iteration
complexity and the performance guarantee are unchanged. The accuracy
guarantee for the covariance matrix $\hat{\Sigma}_{t-1}$ is
straightforward, since each $\hat{\Sigma}_{\pi_t}$ is $\epsopt$
accurate and $\hat{\Sigma}$ is the average of such matrices.
\end{proof}

\ifthenelse{\equal{\version}{arxiv}}{
\section{Planning in the environment}
\label{app:real_world}
In this section, we prove~\pref{thm:real_world_main}, which is based
on planning in the environment, rather than in the model. Recall that
the advantage of this approach is that we do not need the sampling
oracle, $\samp$, but the downside is that we
require~\pref{assum:reachability}. As~\pref{assum:reachability}
implies a polynomial upper bound on $\dlv$, we focus on the simplex
representation, and remark that this also accommodates general
representations with polynomial overhead in sample
complexity. Planning in the environment with general representations
(and~\pref{assum:reachability}) is possible, but the arguments and
calculations are much simpler in the simplex case.

The algorithm here is a ``forward'' version of \ouralg, that we call
\ouralgf. The pseudocode is displayed in~\pref{alg:forward_alg}.  The
algorithm learns the dynamics one time step at a time, starting from
$h=0$ to $h=H-1$. In iteration $h$, we use an exploratory policy
$\rho_h$ to collect a dataset of triples $(x_h,a_h,x_{h+1})$ that we
pass to \mle to obtain an estimate $\hat{T}_h$ of the dynamics at time
$h$. Then we pass the previously computed policies and feature maps to
a new planning algorithm (\pref{alg:real_world_planner}), which yields
the exploratory policy $\rho_{h+1}$ for the next iteration.

\begin{algorithm}[t]
\begin{algorithmic}
\STATE \textbf{Input:} Environment $\Mcal$, function classes $\Phi, \Upsilon$, subroutine \mle, parameter $n$.
\STATE Set $\rho_0$ to be the null policy, which takes no actions.
\FOR{$h=0,\ldots,H-1$}
\STATE Set $\rhotrain \gets \rho_h \circ \unif(\Acal)$. ~~~~~~~~ \COMMENT{Uniform over available actions.}
\STATE Collect $n$ triples $D_h \gets \{(x^{(i)}_h,a^{(i)}_h,x^{(i)}_{h+1})\}_{i=1}^n$ by executing $\rhotrain$ in $\Mcal$.
\STATE Solve maximum likelihood problem: $(\hat{\phi}_h,\hat{\mu}_h) \gets \mle(D_h)$.
\STATE Set $\hat{T}_h(x_{h+1} \mid x_h,a_h) = \inner{\hat{\phi}_h(x_h,a_h)}{\hat{\mu}_h(x_{h+1})}$.
\STATE Call planner (\pref{alg:real_world_planner}) with policies $\rho_{0:h}$ and feature maps $\hat{\phi}_{0:h-1}$ to obtain $\rhopre$.
\STATE Set $\rho_{h+1} = \rhopre\circ\unif(\Acal)$.
\ENDFOR
\end{algorithmic}
\caption{\ouralgf: Forward version of \ouralg with simplex features}
\label{alg:forward_alg}
\end{algorithm}

The main argument is based on inductively establishing two facts. 
\begin{align}
\forall h' < h, \forall \pi:&~~ \EE\sbr{ \nbr{\hat{T}_{h'}(\cdot \mid x_{h'},a_{h'}) - T_{h'}(\cdot \mid x_{h'}, a_{h'})}_{\tv}^2 \mid \pi, \Mcal} \leq \epstv \label{eq:induction_tv}\\
\forall z \in \Zcal_h:&~~ \max_\pi \PP\sbr{z_h = z \mid \pi, \Mcal} \leq \kappa \cdot \PP\sbr{z_h = z \mid \rho_h, \Mcal}, \label{eq:induction_coverage}
\end{align}
where $\kappa>0$ is a constant we will set towards the end of the proof.
We do this in two parts. We first focus on optimizing a fixed given
reward function by collecting experience from the environment,
analogously to the sampling approach
in~\pref{lem:sampling_planner}. For this part, we will assume
that~\pref{eq:induction_tv} and~\pref{eq:induction_coverage} hold with
$h$ being the current planning horizon. In the next subsection we
choose the reward functions carefully to establish the guarantees
required by \ouralg. Since we are considering simplex representations,
this second part is very similar to~\pref{alg:simplex_planner}.

\begin{algorithm}[t]
\begin{algorithmic}
\STATE \textbf{input:} Exploratory policies $\rho_{0:H-1}$, feature maps $\hat{\phi}_{0:H-1}$ with $\hat{\phi}_h(x,a) \in \Delta([\dlv])$.
\FOR{$i=1,\ldots,\dlv$}
\STATE Compute $\hat{\pi}_i = \textsc{Linear-Fqi}(n,\rho_{0:H-1},\hat{\phi}_{0:H-1}, R_{H-1} \defeq \hat{\phi}_{H-1}(x,a)[i])$ ~~~~~~\COMMENT{$R_{0:H-2} \equiv 0$}
\ENDFOR
\RETURN policy mixture $\rho \defeq \unif(\{\hat{\pi}_i\}_{i=1}^{\dlv})$.
\STATE
\STATE \textbf{function} \textsc{Linear-Fqi}($n, \rho_{0:H-1}, \hat{\phi}_{0:H-1}, R_{0:H-1}$)
\STATE \textbf{input:} Sample size $n$, policies $\rho_{0:H-1}$, feature maps $\hat{\phi}_{0:H-1}$, rewards $R_{0:H-1}:\Xcal\times\Acal \to [0,1]$.
\STATE Set $\hat{V}_H(x) = 0$
\FOR{$h= H-1,\ldots,0$}
\STATE Collect $n$ samples $\{(x^{(i)}_h,a^{(i)}_h,x^{(i)}_{h+1})\}_{i=1}^n$ by following $\rho_h\circ\unif(\Acal)$ in $\Mcal$.
\STATE Solve least squares problem:
\begin{align*}
\hat{\theta}_h \gets \argmin_{\theta \in \RR^d: \nbr{\theta}_2 \leq H\sqrt{d}} \sum_{i=1}^n\rbr{ \inner{\theta}{\hat{\phi}_h(x_h^{(i)},a_h^{(i)})} - \hat{V}_{h+1}(x_{h+1}^{(i)})}^2.
\end{align*}
\STATE Define $\hat{Q}_h(x,a) =  R_h(x,a)+\inner{\hat{\theta}_h}{\hat{\phi}_h(x,a)}$.
\STATE Define $\hat{\pi}_h(x) = \argmax_{a} \hat{Q}_h(x,a), \hat{V}_h(x) = \min\{\max_a \hat{Q}_h(x,a), H\}$.
\ENDFOR
\RETURN $\hat{\pi} = (\hat{\pi}_0,\ldots,\hat{\pi}_{H-1})$.
\end{algorithmic}
\caption{Planning in the environment with simplex features}
\label{alg:real_world_planner}
\end{algorithm}

\subsection{Optimizing a fixed reward function}
To optimize a fixed reward function, the high level idea is that,
via~\pref{lem:representation}, we can approximate any Bellman backup
using our features $\hat{\phi}$, and via~\pref{eq:induction_coverage},
we can collect a dataset with good coverage. Using these two
properties the planning algorithm, \textsc{Linear-Fqi}, displayed as a
subroutine in~\pref{alg:real_world_planner} is quite natural.  The
algorithm is a least squares dynamic programming algorithm (FQI stands
for ``Fitted Q Iteration''). For each $h$, working from $H-1$ down to
$0$, we collect a dataset of $n$ samples by following $\rho_h$. Then,
we solve a least squares regression problem to approximate the Bellman
backup of the value function estimate $\hat{V}_{h+1}$ for the next
time. We use this to define the value function and the policy for the
current time in the obvious way.

Note that we index policies in two different ways: $\rho_h$ is the
exploratory policy that induces a distribution over $x_h$, while
$\hat{\pi}_h$ is the one-step policy that we acts on $x_h$.  As with
the other planning lemmas, we apply the next lemma with a value of $H$
that is not necessarily the real horizon in the environment. In
particular, we will use this lemma in the $h^{\textrm{th}}$ iteration
of \ouralg, with planning horizon $h-1$ and with reward functions
specified in the next subsection. By induction, we can assume
that~\pref{eq:induction_tv} and~\pref{eq:induction_coverage} hold for
the planning horizon.

\begin{lemma}
\label{lem:real_world_optimization}
Assume that~\pref{eq:induction_tv} and~\pref{eq:induction_coverage}
hold for all $h \in [H]$. Then for any reward functions $R_{0:H-1}:
\Xcal \times \Acal \to [0,1]$ and any $\delta \in (0,1)$, if we set
\begin{align*}
n \geq \frac{2304 d^3}{\epstv^2}\log\rbr{1152d^3/\epstv^2} + \frac{2304d^2}{\epstv^2}\log(2H/\delta),
\end{align*}
then the policy $\hat{\pi}_{0:H-1}$ returned
by \textsc{Linear-Fqi} satisfies
\begin{align*}
\EE\sbr{\sum_{h=0}^{H-1} r_h \mid \hat{\pi}_{0:H-1},\Mcal} \geq \max_{\pi} \EE\sbr{\sum_{h=0}^{H-1}r_h \mid \pi,\Mcal} - 2H^3\sqrt{2\kappa K\epstv}.
\end{align*}
\end{lemma}
\begin{proof}
The analysis is similar to that of~\citet{chen2019information}, who
study a similar algorithm in the infinite-horizon discounted setting.
Let $\EE_h$ denote expectation induced by the distribution over
$(x_h,a_h,x_{h+1})$ obtained by following
$\rho_h\circ\unif(\Acal)$. For any function $f:\Xcal \to \RR$, let
$\Bcal_h f(x,a) \defeq \EE[f(x_{h+1}) \mid x_h,a_h]$ denote the
Bellman backup operator for time $h$ \emph{without the immediate
  reward}.  Let $\hat{\Bcal}_h$ denote the Bellman backup operator
induced by the learned model at time $h$, again without the immediate
reward.  We omit the dependence on $x,a$ in these operators when it is
clear from context.  Note that by the normalization assumptions, we
always have $\hat{V}_{h+1}(x_{h+1}) \in [0,H]$. Moreover,
$\hat{\Bcal}_h\hat{V}_{h+1}$ is a linear function in $\hat{\phi}_h$
where the coefficient vector has $\ell_2$ norm at most $H\sqrt{d}$.

We apply~\pref{lem:least_squares} with $B \defeq H\sqrt{d}$, and we take a
union bound over all $h \in [H]$. Defining $\Delta_n \defeq
24H^2d\sqrt{2(d\log n + \log(2H/\delta))/n}$, we have that with
probability at least $1-\delta$, for all $h \in [H]$:
\begin{align*}
\EE_{h} \sbr{\rbr{\inner{\hat{\theta}_h}{\hat{\phi}_h(x_h,a_h)} - \Bcal_h\hat{V}_{h+1}}^2} &\leq \min_{\theta: \nbr{\theta}_2 \leq B} \EE_{h} \sbr{\rbr{\inner{\theta}{\hat{\phi}_h(x_h,a_h)} - \Bcal_h \hat{V}_{h+1}}^2} + \Delta_n\\
& \leq \EE_h\sbr{\rbr{\hat{\Bcal}_h \hat{V}_{h+1} - \Bcal_h \hat{V}_{h+1}}^2} + \Delta_n\\
& \leq H^2 \cdot \EE_h \nbr{\hat{T}_h(\cdot \mid x_h,a_h) - T_h(\cdot \mid x_h,a_h)}_{\tv}^2 + \Delta_n.
\end{align*}
The first inequality here is the least squares generalization
analysis, additionally using that $\Bcal_h\hat{V}_{h+1}$ is the Bayes
optimal predictor. The second uses the fact that the Bellman backups
in the model are linear functions in $\hat{\phi}$ (with bounded
coefficient vector). Precisely, we have
\begin{align*}
\hat{\Bcal}_h\hat{V}_{h+1}(x_h,a_h) = \inner{\hat{\phi}_h(x_h,a_h)}{\int \hat{\mu}_h(x_{h+1}) \hat{V}_{h+1}(x_{h+1}) d(x_{h+1})}.
\end{align*}
Setting $\theta$ to be the second term, we obtain the second
inequality. Finally, we apply Holder's inequality and use the fact
that $\hat{V}_{h+1}$ is bounded in $[0,H]$ by construction. Appealing
to~\pref{eq:induction_tv} we have
\begin{align*}
\EE_{h} \sbr{\rbr{\inner{\hat{\theta}_h}{\hat{\phi}_h(x_h,a_h)} - \Bcal_h\hat{V}_{h+1}}^2} \leq H^2 \epstv + \Delta_n.
\end{align*}
Now applying~\pref{eq:induction_coverage}, we transfer this squared
error to the distribution induced by any other policy. This
calculation is exactly the same as in the main induction argument (see~\pref{eq:simplex_transfer}), and
it yields
\begin{align*}
\EE_\pi \sbr{\rbr{\inner{\hat{\theta}_h}{\hat{\phi}_h(x_h,a_h)} - \Bcal_h\hat{V}_{h+1}}^2}  \leq \kappa K\rbr{H^2\epstv + \Delta_n}.
\end{align*}
Next, we bound the difference in cumulative rewards between $\hat{\pi}
\defeq \hat{\pi}_{0:H-1}$ and the optimal policy $\pi^\star$ for the
reward function. For this, recall that we define $\hat{Q}_0(x,a) =
R_0(x,a) + \langle\hat{\theta}_0,\hat{\phi}_0(x,a)\rangle$ and also
that $\hat{\pi}_0$ is greedy with respect to this $Q$ function, which
implies that $\hat{Q}_0(x,\hat{\pi}_0(x)) \geq
\hat{Q}_0(x,\pi^\star(x))$ for all $x$. Therefore,
\begin{align*}
V^\star - V^{\hat{\pi}} &= \EE\sbr{R(x_0,a_0) + V^{\star}(x_1) \mid \pi^\star} - \EE\sbr{R(x_0,a_0) + V^{\hat{\pi}}(x_1) \mid \hat{\pi}}\\
& \leq \EE\sbr{R(x_0,a_0) + V^\star(x_1) - \hat{Q}_0(x_0,a_0) \mid \pi^\star} - \EE\sbr{R(x_0,a_0) + V^{\hat{\pi}}(x_1) - \hat{Q}_0(x_0,a_0) \mid \hat{\pi}}\\
& = \EE\sbr{V^\star(x_1) - \inner{\hat{\theta}_0}{\hat{\phi}_0(x_0,a_0)} \mid \pi^\star} - \EE\sbr{V^{\hat{\pi}}(x_1) - \inner{\hat{\theta}_0}{\hat{\phi}_0(x_0,a_0)} \mid \hat{\pi}}\\
& = \EE\sbr{V^\star(x_1) - \inner{\hat{\theta}_0}{\hat{\phi}_0(x_0,a_0)} \mid \pi^\star} - \EE\sbr{V^{\star}(x_1) - \inner{\hat{\theta}_0}{\hat{\phi}_0(x_0,a_0)} \mid \hat{\pi}} \\
& ~~~~~~~~ + \EE\sbr{V^\star(x_1) - V^{\hat{\pi}}(x_1) \mid \hat{\pi}}.
\end{align*}
Continuing, we find
\begin{align*}
V^\star - V^{\hat{\pi}} & \leq \sum_{h=0}^{H-1}\EE\sbr{V^\star(x_{h+1}) - \inner{\hat{\theta}_h}{\hat{\phi}_h(x_h,a_h)} \mid \hat{\pi}_{0:h-1}\circ\pi^\star} \\
& ~~~~~~~~ - \sum_{h=0}^{H-1} \EE\sbr{V^\star(x_{h+1}) - \inner{\hat{\theta}_h}{\hat{\phi}_h(x_h,a_h)} \mid \hat{\pi}_{0:h}}.
\end{align*}
Next, we bound each of these terms. Let us focus on just one of them,
call the roll-in policy $\pi$ and drop the dependence on $h$. Then,
\begin{align*}
\EE_\pi\sbr{\abr{\EE\sbr{V^\star(x')\mid x,a} - \inner{\hat{\theta}}{\hat{\phi}(x,a)}}} &\leq \EE_\pi\sbr{\abr{\Bcal V^\star(x,a) - \Bcal\hat{V}(x,a)} + \abr{\Bcal\hat{V}(x,a) - \inner{\hat{\theta}}{\hat{\phi}(x,a)}}}\\
& \leq \EE_\pi\sbr{\abr{V^\star(x') - \hat{V}(x')}+ \abr{\Bcal\hat{V}(x,a) - \inner{\hat{\theta}}{\hat{\phi}(x,a)}}},
\end{align*}
where the second inequality is Jensen's inequality.
By definition
\begin{align*}
& \EE_\pi\sbr{\abr{V^\star(x') - \hat{V}(x')}} = \EE_\pi\sbr{\abr{\max_{a} Q^\star(x',a) - \min\{H,\max_{a'}R(x',a') + \inner{\hat{\phi}(x',a')}{\hat{\theta}}\}}}\\
& \leq \EE_\pi\sbr{\abr{\max_{a} Q^\star(x',a) - \max_{a'}R(x',a') + \inner{\hat{\phi}(x',a')}{\hat{\theta}}}}
 \leq \EE_{\pi \circ \tilde{\pi}} \sbr{ \abr{\Bcal V^\star(x',a') - \inner{\hat{\phi}(x',a')}{\hat{\theta}}}}.
\end{align*}
Here in the last inequality, we define $\tilde{\pi}$ to choose the
larger of the two actions, that is we set $\tilde{\pi}(x') =
\argmax_{a \in \Acal} \max\{Q^\star(x',a),
R(x',a)+ \langle\hat{\theta},\hat{\phi}(x',a)\rangle\}$. This expression has
the same form as the initial one, but at the next time point, so unrolling, we get
\begin{align*}
\EE\sbr{\abr{\Bcal V^\star(x_{h},a_h) - \inner{\hat{\theta}_h}{\hat{\phi}_h(x_h,a_h)}} \mid \pi} & \leq \sum_{\tau=h}^{H-1} \max_{\pi_\tau} \EE_{\pi_\tau}\sbr{\abr{\Bcal\hat{V}_{\tau+1}(x_\tau,a_\tau) - \inner{\hat{\theta}_\tau}{\hat{\phi}_\tau(x_\tau,a_\tau)}}}\\
& \leq H\sqrt{\kappa K(H^2\epstv + \Delta_n)}.
\end{align*}
Plugging this into the overall value difference, the final bound is
\begin{align*}
V^\star - V^{\hat{\pi}} \leq 2H^2\sqrt{\kappa K(H^2\epstv + \Delta_n)}.
\end{align*}
To wrap up, we want the term involving $\Delta_n$ to be at most
$H^2\epstv$, so the term involving $\Delta_n$ is of the same order as
the term involving $\epstv$. By our definition of $\Delta_n$, this
requires
\begin{align*}
n \geq \frac{24^2d^2}{\epstv^2}\cdot 2\rbr{d \log n + \log(2H/\delta)}.
\end{align*}
A sufficient condition here is
\begin{align*}
n \geq \frac{2304 d^3}{\epstv^2}\log\rbr{1152d^3/\epstv^2} + \frac{2304d^2}{\epstv^2}\log(2H/\delta),
\end{align*}
which yields the result.
\end{proof}

\akshay{I guess I thought we should be able to get $d/n$ rate here, but it seems more complicated.}
\begin{lemma}
\label{lem:least_squares}
Let $\{\phi_i,y_i\}_{i=1}^n$ be $n$ samples drawn iid from some
distribution where $\phi \in \RR^d$ satisfies $\nbr{\phi}_2 \leq 1$
and $y \in [0,H]$ almost surely. Let $\hat{\theta} \in \RR^d$ denote
the constrained square loss minimizer, constrained so that
$\| \hat{\theta} \|_2 \leq B$, where $B \geq H$. Then for any $\delta
\in (0,1)$, with probability at least $1-\delta$ we have
\begin{align*}
\EE\sbr{ (\langle\hat{\theta}, \phi\rangle - y)^2} \leq \min_{\theta: \nbr{\theta}_2 \leq B}\EE\sbr{\rbr{\inner{\theta}{\phi} -y}^2} + 24B^2\sqrt{\frac{2}{n}\rbr{d\log(n) + \log(\nicefrac{2}{\delta})}}.
\end{align*}
\end{lemma}
\begin{proof}
Fix $\theta$ with $\nbr{\theta}_2 \leq B$. We will apply Hoeffding's
inequality on this $\theta$ and then use a covering argument for
uniform convergence. Let $R(\theta)$ denote the expected square loss,
with $\hat{R}(\theta)$ as the empirical counterpart. Using the bounds
on all quantities, the square loss has range $(B+H)^2 \leq 4B^2$, and
so Hoeffding's inequality yields that with probability at least
$1-\delta$
\begin{align*}
\abr{R(\theta) - \hat{R}(\theta)} \leq 4B^2\sqrt{\frac{2}{n}\log(\nicefrac{2}{\delta})}.
\end{align*}
Let $V_\gamma$ denote a covering of $\{\theta: \nbr{\theta}_2
\leq B\}$ in the $\ell_2$ norm at scale $\gamma$, which has $\log
|V_\gamma| \leq d \log(2B/\gamma)$ via standard arguments. Taking a
union bound, the above inequality holds for all $\theta \in V_\gamma$
with probability $1-|V_\gamma| \delta$. By direct calculation, we
see that $\hat{R}(\theta)$ and $R(\theta)$ are both
$2(B+H)$-Lipschitz. Therefore, we have that with probability
$1-|V_\gamma| \delta$, for all $\theta$ with $\nbr{\theta}_2 \leq B$
\begin{align*}
\abr{R(\theta) - \hat{R}(\theta)} \leq 4B\gamma + 4B^2\sqrt{\frac{2}{n}\log(\nicefrac{2}{\delta})}.
\end{align*}
Taking $\gamma = 2B/\sqrt{n}$ we can rebind $\delta$ and absorb the
first term into the second. Thus, with probability at least $1-\delta$, for all $\theta$, we have
\begin{align*}
\abr{R(\theta) - \hat{R}(\theta)} \leq  12B^2\sqrt{\frac{2}{n}\rbr{d\log(n) + \log(\nicefrac{2}{\delta})}}.
\end{align*}
Now by the standard ERM analysis, we have
\begin{align*}
R(\hat{\theta}) & \leq \hat{R}(\hat{\theta}) + 12B^2\sqrt{\frac{2}{n}\rbr{d\log(n) + \log(\nicefrac{2}{\delta})}} \leq \min_{\theta} \hat{R}(\theta) + 12B^2\sqrt{\frac{2}{n}\rbr{d\log(n) + \log(\nicefrac{2}{\delta})}}\\
& \leq \min_{\theta}R(\theta) + 24B^2\sqrt{\frac{2}{n}\rbr{d\log(n) + \log(\nicefrac{2}{\delta})}}.\tag*\qedhere
\end{align*}
\end{proof}

\subsection{Instantiating the reward functions.}
We now use~\pref{alg:real_world_planner} in \ouralgf, to establish the
induction. Assume that $\hat{\phi}_h(x,a) \in \Delta([\dlv])$ for all
$x,a,h$, analogously to in~\pref{thm:improved}.  Recall that $\rho_h$
is our exploratory policy that induces a distribution over $x_h$. We
augment $\rho_h$ with an action taken uniformly at random to obtain
the ``training policy'' $\rhotrain$. Via an application
of~\pref{thm:mle} (using~\pref{assum:realizability}), we know that
with probability at least $1-\delta$ we learn
$\hat{\phi}_h,\hat{\mu}_h$ such that (with $\hat{T} =
\langle\hat{\phi}_h,\hat{\mu}_h\rangle$):
\begin{align}
\EE_{(x_h,a_h) \sim \rhotrain} \nbr{ \hat{T}_h(\cdot \mid x_h,a_h) - T_h(\cdot \mid x_h,a_h)}_{\tv}^2 \leq \frac{2 \log (|\Phi| |\Upsilon|/\delta)}{n} =: \epssup. \label{eq:supervised_error}
\end{align}
This is the only step where we use the optimization oracle, \mle, and
similar guarantee can also be obtained by other means. As one example,
in~\pref{rem:fgan}, we discuss a generative adversarial approach.

We now use this bound and~\pref{eq:induction_coverage} to
establish~\pref{eq:induction_tv} for time $h$. Considering any policy
$\pi$, we define the ``error function'' $\err_\pi(x_h) \defeq \int
\pi(a_h \mid x_h) \nbr{\hat{T}_h(\cdot \mid x_h,a_h) - T_h(\cdot \mid
  x_h,a_h)}_{\tv}^2$.
\begin{align}
& \EE_\pi\nbr{\hat{T}_h(\cdot\mid x_h,a_h) - T_h(\cdot \mid x_h,a_h)}^2_{\tv} = \EE_\pi\sbr{\err_\pi(x_h)}\notag\\
& = \sum_{z \in \Zcal_h} \PP\sbr{z_h = z \mid \pi,\Mcal} \cdot \int \err_\pi(x_h) \nu_{h-1}^\star(x_h \mid z) d(x_h)\notag\\
& \leq \kappa \cdot \sum_{z \in \Zcal_h} \PP\sbr{z_h = z \mid \rho_h,\Mcal}\cdot \int \err_\pi(x_h) \nu_{h-1}^\star(x_h \mid z) d(x_h)\notag\\
& = \kappa \cdot \EE_{\rho_h}\sbr{\err_\pi(x_h)}
 \leq \kappa \cdot \EE_{\rhotrain} \sbr{\nbr{\hat{T}_h(\cdot\mid x_h,a_h) - T_h(\cdot \mid x_h,a_h)}^2_{\tv}} \cdot \sup_{x_{h},a_h} \abr{ \frac{\pi(a_h \mid x_h)}{\rhotrain(a_h \mid x_h)}}\notag\\
& \leq \kappa K \epssup =: \epstv \label{eq:simplex_transfer}
\end{align}
The first inequality is~\pref{eq:induction_coverage}, which allows us
to transfer from the distribution induced by $\pi$ to the distribution
induced by $\rho_h$. It is crucial that the pre-multiplier term
involving $\nu_{h-1}^\star$ and $\err_\pi$ is non-negative which
follows from the fact that $\err_\pi$ is non-negative and
$\nu_{h-1}^\star(\cdot)[i]$ is a (positive) measure. 
The
final two inequalities are based on importance weighting for the
action at time $h$, using the fact that $\rhotrain(\cdot \mid x_h) =
\unif(\Acal)$. This final expression is our choice of $\epstv$, which
establishes~\pref{eq:induction_tv} for time $h$.
For time $h=0$,~\pref{eq:induction_tv} follows immediately
from~\pref{eq:supervised_error}, since $(x_0,a_0)$ are fixed. In
particular all policies induce the same distribution over $(x_0,a_0)$
so transfering from $\pi$ to $\rho_0^{\mathrm{train}}$ is trivial. As
$K,\kappa \geq 1$, this gives the base case.

Next we turn to establishing~\pref{eq:induction_coverage} for the next
exploratory policy $\rho_{h+1}$. The planning algorithm
in~\pref{alg:real_world_planner} is analogous
to~\pref{alg:simplex_planner}, except that we perform the optimization
in the environment using \textsc{Linear-Fqi}, with parameter $n$ that
we will set subsequently. At iteration $h$ of \ouralgf, this yields
$\dlv$ policies $\hat{\pi}_1,\ldots,\hat{\pi}_{\dlv}$ where
$\hat{\pi}_i$ approximately maximizes the probability of reaching the
$i^{\textrm{th}}$ coordinate of $\hat{\phi}_{h-1}$ when executed in
the real world.

Defining $\epsstat$ to be the sub-optimality guaranteed
by~\pref{lem:real_world_optimization} (additionally taking a union
bound over all $H\dlv$ invocations), we have that at iteration $h$ of
\ouralg
\begin{align*}
\EE\sbr{\hat{\phi}_{h-1}(x_{h-1},a_{h-1})[i] \mid \hat{\pi}_i,\Mcal} \geq \max_\pi \EE\sbr{\hat{\phi}_{h-1}(x_{h-1},a_{h-1})[i] \mid \pi,\Mcal} - \epsstat.
\end{align*}
We define $\rhopre$ to be the uniform distribution over the
$\hat{\pi}_i$ policies, which induce a distribution over $x_h$.

Now for any function $f:\Xcal \to [0,1]$, appealing to~\pref{eq:induction_tv} at time $h$ we have
\begin{align*}
\EE_\pi f(x_h) & \leq \EE_{\pi} \inner{\hat{\phi}_{h-1}(x_{h-1},a_{h-1})}{\int\hat{\mu}_{h-1}(x_h)f(x_h)} + \sqrt{\epstv}\\
& = \sum_{i=1}^{\dlv} \rbr{\int\hat{\mu}_{h-1}(x_h)[i] f(x_h)}\cdot \EE_\pi\hat{\phi}_{h-1}(x_{h-1},a_{h-1})[i] + \sqrt{\epstv}\\
& \leq \sum_{i=1}^{\dlv} \rbr{\int\hat{\mu}_{h-1}(x_h)[i] f(x_h)}\cdot \rbr{\EE_{\hat{\pi}_i}\hat{\phi}_{h-1}(x_{h-1},a_{h-1})[i] +\epsstat} + \sqrt{\epstv}\\
& \leq \sum_{i=1}^{\dlv} \rbr{\int\hat{\mu}_{h-1}(x_h)[i] f(x_h)}\cdot \rbr{\dlv \EE_{\rhopre}\hat{\phi}_{h-1}(x_{h-1},a_{h-1})[i] +\epsstat} + \sqrt{\epstv}\\
& \leq \dlv \EE_{\rhopre}\inner{\hat{\phi}_{h-1}(x_{h-1},a_{h-1})}{\int\hat{\mu}_{h-1}(x_h)[i] f(x_h)} +\dlv \epsstat + \sqrt{\epstv}\\
& \leq \dlv \EE_{\rhopre}f(x_h) +\dlv \epsstat + (1+\dlv)\sqrt{\epstv}.
\end{align*}
The first and last inequalities here use~\pref{eq:induction_tv} on
$\hat{T}_{h-1}$, which holds by induction. The second inequality is
the optimality guarantee for $\hat{\pi}_i$, and the third is based on
the definition of $\rhopre$. For the fourth inequality, we collect
terms, and additionally use that $f$ is $\ell_{\infty}$ bounded and
$\hat{\mu}_{h-1}[i]$ is a measure, so $\abr{\int
  \hat{\mu}_{h-1}(x_h)[i] f(x_h)} \leq 1$.  Via importance weighting,
we have that for any latent variable $z\in\Zcal_{h+1}$
\begin{align*}
\max_\pi \PP[z_{h+1}=z] \leq \dlv K\cdot \PP_{\rho_{h+1}}[z_{h+1}=z] + \dlv\epsstat + (1+\dlv)\sqrt{\epstv}.
\end{align*}
We must set the additive error to be at most $\nicefrac{\etamin}{2}$,
which establish~\pref{eq:induction_coverage} with $\kappa = 2\dlv
K$. Unpacking the definition of $\epsstat$ in the simplex features
case this gives the constraint
\begin{align*}
2H^3\sqrt{4\dlv K^2 \epstv} + (1+\dlv)\sqrt{\epstv} \leq \nicefrac{\etamin}{2}.
\end{align*}
Therefore, we set
\begin{align*}
\epstv \leq \order\rbr{ \frac{\etamin^2}{H^6\dlv^2 K^2}}.
\end{align*}
With this choice we establish the inductive hypothesis for the next
time, and proceeding in this manner, we
establish~\pref{eq:induction_tv} for all time steps. 
As this is quite similar to our desired system identification
guarantee,~\pref{eq:sys_id}, we are just left to set the parameters
appropriately and calculate the sample complexity. For the calls to
\mle, we have that the $\epstv = 2\dlv K^2 \epssup$ which yields the
constraint
\begin{align*}
\epssup \leq \order\rbr{\min\cbr{\frac{\etamin^2}{H^6 \dlv^3 K^4},\frac{\varepsilon^2}{\dlv K^2}}}.
\end{align*}
This means that for the calls to \mle we may set $n$ as
\begin{align*}
n = \order\rbr{ \max\cbr{\frac{\dlv^2 K^2H^6}{\etamin^2},\frac{1}{\varepsilon^2}} \cdot \dlv K^2 \log (H|\Phi| |\Upsilon|/\delta)}.
\end{align*}
The calls to \mle incur a total sample complexity of $nH$.

We also have to collect trajectories to invoke
\textsc{Linear-Fqi}. For this, we must set $n$ as
\begin{align*}
n = \otil\rbr{\frac{\dlv^3}{\epstv^2}\log(\nicefrac{1}{\delta})} = \otil\rbr{\frac{\dlv^7 K^4 H^{12}}{\etamin^4}\log(\nicefrac{1}{\delta})},
\end{align*}
and the calls to \textsc{Linear-Fqi} require $nH\dlv$
samples in total. Therefore, the total sample complexity is
\begin{align*}
\otil\rbr{\max\cbr{\frac{\dlv^2 K^2 H^6}{\etamin^2},\frac{1}{\varepsilon^2}} \cdot H\dlv K^2 \log(|\Phi||\Upsilon|/\delta) + \frac{\dlv^8 K^4 H^{13}}{\etamin^4}\log(\nicefrac{1}{\delta})}.
\end{align*}

}{}

\section{Maximum Likelihood Estimation}

In this section we adapt classical results for maximum likelihood
estimation in general parametric models. We consider a sequential
conditional probability estimation setting with an instance space
$\Xcal$ and target space $\Ycal$ and with a conditional density
$p(y \mid x) = f^\star(x,y)$. We are given a function class
$\Fcal: (\Xcal \times \Ycal) \to \RR$ with which to model the
condition distribution $f^\star$, and we assume that $f^\star \in
\Fcal$, so that the problem is well-specified or realizable.  We are
given a dataset $D \defeq \{(x_i,y_i)\}_{i=1}^n$, where $x_i \sim
\Dcal_i = \Dcal_i(x_{1:i-1},y_{1:i-1})$ and $y_i \sim p(\cdot \mid
x_i)$. Note that $\Dcal_i$ depends on the previous examples, so this
is a martingale process. We optimize the maximum likelihood objective
\begin{align}
\hat{f} \defeq \argmax_{f \in \Fcal} \sum_{i=1}^n \log f(x_i,y_i).
\end{align}
The iid version of the following result is
classical~\citep[c.f.,][Chapter 7]{geer2000empirical}, but
under-utilized in machine learning and reinforcement learning in
particular. Our adaptation is inspired by~\citet{zhang2006from}.

\begin{theorem}
\label{thm:mle}
Fix $\delta \in (0,1)$, assume $|\Fcal| < \infty$ and $f^\star \in \Fcal$. Then with probability at least $1-\delta$
\begin{align*}
\sum_{i=1}^n\EE_{x \sim \Dcal_i} \nbr{\hat{f}(x,\cdot) - f^\star(x,\cdot)}_{\tv}^2 \leq 2 \log (|\Fcal|/\delta).
\end{align*}
\end{theorem}

\begin{remark}
\label{rem:fgan}
Given a class of discriminators $\Gcal:(\Xcal,\Ycal) \mapsto [-1,1]$,
an alternative is to consider the following (conditional) ``generative
adversarial'' objective:
\begin{align*}
\hat{f} = \argmin_{f \in \Fcal} \max_{g \in \Gcal}\frac{1}{n}\sum_{i=1}^n \rbr{ g(x_i,y_i) - \EE[g(x_i,y) \mid y \sim f(x,\cdot)]}.
\end{align*}
This is the natural objective associated with the distance function
induced by $\Gcal$~\citep{arora2017generalization}, and is also
related to other GAN-style approaches.  Owing to the realizability
assumption, $f^\star$ will always have low objective value, scaling
with the complexity of $\Gcal$. Additionally, if $\Gcal$ is expressive
enough, one can establish a guarantee similar to~\pref{thm:mle}, which
can then be used in the analysis of \ouralg. Formally, a sufficient
condition is that $\Gcal$ contains the indicators of the \emph{Scheffe
  sets} for all pairs $f,f'\in\Fcal$, in which case the total
variation guarantee can be obtained by standard uniform convergence
arguments. See~\citet{devroye2012combinatorial,sun2018model} for more
details.
\end{remark}

\begin{remark}
We also remark that the proof of~\pref{thm:mle} actually establishes
convergence in the squared Hellinger distance. We obtain the total
variation guarantee simply by observing that the squared Hellinger
distance dominates the squared total variation distance.
\end{remark}

We prove~\pref{thm:mle} in this section. We begin with a decoupling
inequality. Let $D$ denote the dataset and let $D'$ denote a
\emph{tangent sequence} $\{(x_i',y_i')\}_{i=1}^n$ where $x_i' \sim
\Dcal_i(x_{1:i-1},y_{1:i-1})$ and $y_i' \sim p(\cdot \mid x_i')$. Note
here that $x_i'$ depends on the original sequence, and so the tangent
sequence is independent conditional on $D$.
\begin{lemma}
\label{lem:decoupling}
Let $D$ be a dataset of $n$ examples, and let $D'$ be a tangent
sequence. Let $L(f,D) = \sum_{i=1}^n\ell(f,(x_i,y_i))$ be any function
that decomposes additively across examples where $\ell$ is any
function, and let $\hat{f}(D)$ be any estimator taking as input random
variable $D$ and with range $\Fcal$. Then
\begin{align*}
\EE_{D}\sbr{\exp\rbr{L(\hat{f}(D),D) - \log \EE_{D'}\exp(L(\hat{f}(D),D')) - \log |\Fcal|}} \leq 1
\end{align*}
\end{lemma}
Observe that in the second term, the ``loss function'' takes as input
$D'$, but the estimator takes as input $D$. As such, the above
inequality decouples the estimator from the loss.
\begin{proof}
Let $\pi$ be the uniform distribution over $\Fcal$ and let $g: \Fcal
\to \RR$ be any function. Define $\mu(f) \defeq
\frac{\exp(g(f))}{\sum_{f} \exp(g(f))}$, which is clearly a
probability distribution. Now consider any other probability
distribution $\hat{\pi}$ over $\Fcal$:
\begin{align*}
0 &\leq \kl{\hat{\pi}}{\mu} = \sum_{f}\hat{\pi}(f)\log(\hat{\pi}(f)) + \sum_{f} \hat{\pi}(f)\log\rbr{\sum_{f'} \exp(g(f'))} - \sum_f \hat{\pi}(f)g(f)\\
& = \kl{\hat{\pi}}{\pi} - \sum_f \hat{\pi}(f)g(f) + \log \EE_{f \sim \pi} \exp(g(f))\\
& \leq \log |\Fcal|- \sum_f \hat{\pi}(f)g(f) + \log \EE_{f \sim \pi} \exp(g(f)).
\end{align*}
Re-arranging, it holds that
\begin{align*}
\sum_f \hat{\pi}(f) g(f) - \log |\Fcal| \leq \log \EE_{f \sim \pi} \exp(g(f)).
\end{align*}
We instantiate this bound with $\hat{\pi} = \one\{\hat{f}(D)\}$ and
$g(f) = L(f,D) - \log \EE_{D'}\exp(L(f,D'))$ to obtain, for any $D$
\begin{align*}
L(\hat{f}(D),D) - \log \EE_{D'} \exp(L(\hat{f}(D),D')) - \log |\Fcal| \leq \log \EE_{f \sim \pi} \frac{\exp\rbr{L(f,D)}}{\EE_{D'}\exp(L(f,D'))}.
\end{align*}
Exponentiating both sides and then taking expectation over $D$, we obtain
\begin{align*}
& \EE_D\sbr{\exp( L(\hat{f}(D),D) - \log \EE_{D'} \exp(L(\hat{f}(D),D')) - \log |\Fcal|)} \\
& ~~~~~~~~\leq \EE_{f \sim \pi} \EE_D \frac{ \exp\rbr{L(f,D)}}{\EE_{D'}[\exp(L(f,D')) \mid D]} = 1.
\end{align*}
The last equality follows since, conditional on $D$, the tangent
sequence $D'$ is independent. Therefore,
\begin{align*}
\EE_{D'}\sbr{\exp(L(f,D')) \mid D} = \prod_{i=1}^n \EE_{(x_i',y_i')
  \sim \Dcal_i}\sbr{\exp(\ell(f,(x_i',y_i')))},
\end{align*}
which allows us to peel off terms starting from $n$ down to $1$ and
cancel them with those in the numerator.
\end{proof}

The next lemma translates from TV-distance to a loss
function that is closely related to the KL divergence.
\begin{lemma}
\label{lem:surrogate}
For any two conditional probability densities $f_1,f_2$ and any
distribution $\Dcal \in \Delta(\Xcal)$ we have
\begin{align*}
\EE_{x \sim \Dcal} \nbr{f_1(x,\cdot) - f_2(x,\cdot)}_{\tv}^2 \leq - 2\log \EE_{x\sim\Dcal,y \sim f_2(\cdot \mid x)} \exp\rbr{- \frac{1}{2}\log(f_2(x,y)/f_1(x,y))}
\end{align*}
\end{lemma}
\begin{proof}
Let us begin by relating the total variation distance, which appears
on the left hand side, to the (squared) Hellinger distance, which for densities
$p,q$ over a domain $\Zcal$ is defined as
\begin{align*}
\hell{q}{p} \defeq \int \rbr{\sqrt{p(z)} - \sqrt{q(z)}}^2 dz.
\end{align*}
Lemma 2.3 in~\citet{tsybakov2008introduction} asserts that
\begin{align*}
\nbr{p(\cdot) - q(\cdot)}^2_{\tv} \leq \hell{q}{p}\cdot \rbr{1 -
  \frac{\hell{q}{p}}{4}} \leq \hell{q}{p},
\end{align*}
where the final inequality uses non-negativity of the Hellinger distance. Next, note that we can also write
\begin{align*}
\hell{q}{p} &= \int p(z) + q(z) - 2\sqrt{p(z)q(z)} dz = 2\cdot\EE_{z \sim q}\sbr{1 - \sqrt{p(z)/q(z) }}\\
& \leq - 2\log \EE_{z \sim q}\sqrt{p(z)/q(z)} = -2\log \EE_{z \sim q}\exp\rbr{ - \frac{1}{2}\log (q(z)/p(z))}.
\end{align*}
Here the inequality follows from the fact that $1-x \leq - \log (x)$.
The result follows by applying this argument to $\EE_{x \sim
  \Dcal}\nbr{f_1(x,\cdot) - f_2(x,\cdot)}_{\tv}^2$.
\end{proof}

\begin{proof}[Proof of~\pref{thm:mle}]
First note that~\pref{lem:decoupling} can be combined with the Chernoff
method to obtain an exponential tail bound: with probability $1-\delta$ we have
\begin{align*}
-\log \EE_{D'} \exp(L(\hat{f}(D), D')) \leq -L(\hat{f}(D),D) + \log |\Fcal| + \log(1/\delta).
\end{align*}
Now we set $L(f,D) = \sum_{i=1}^n \nicefrac{-1}{2}\cdot
\log(f^\star(x_i,y_i)/f(x_i,y_i))$ where $D$ is a dataset
$\{(x_i,y_i)\}_{i=1}^n$ (and $D' = \{(x_i',y_i')\}_{i=1}^n$ is
a tangent sequence). With this choice, the right hand side is
\begin{align*}
\sum_{i=1}^n \frac{1}{2} \log(f^\star(x_i,y_i)/\hat{f}(x_i,y_i)) + \log |\Fcal| + \log(1/\delta) \leq \log |\Fcal| + \log(1/\delta),
\end{align*}
since $\hat{f}$ is the empirical maximum likelihood estimator and we
are in the well-specified setting. On the other hand, the left hand
side is
\begin{align*}
- \log \EE_{D'} \sbr{\exp\rbr{\sum_{i=1}^n \nicefrac{-1}{2}\log\rbr{\frac{f^\star(x_i',y_i')}{\hat{f}(x_i',y_i')} } } \mid D} &= -\sum_{i=1}^n\log \EE_{x,y\sim \Dcal_i} \exp\rbr{\nicefrac{-1}{2} \log\rbr{\frac{f^\star(x,y)}{\hat{f}(x,y)}}}\\
& \geq \frac{1}{2} \sum_{i=1}^n \EE_{x \sim \Dcal_i} \nbr{\hat{f}(x,\cdot) - f^\star(x,\cdot)}_{\tv}^2.
\end{align*}
Here the first identity uses the independence of the terms, which
holds because $\hat{f}$ is independent of the dataset $D'$. The second
inequality is~\pref{lem:surrogate}. This yields the theorem.
\end{proof}

\section{Auxilliary Lemmas}
\begin{lemma}[Elliptical Potential Lemma]
\label{lem:log_det}
Consider a sequence of $d \times d$ positive semidefinite matrices
$X_1,\ldots,X_T$ with $\max_t \tr(X_t) \leq 1$ and define $M_0 =
\lambda I_{d \times d}, \ldots, M_t = M_{t-1} + X_t$. Then
\begin{align*}
\sum_{t=1}^T \tr(X_t M_{t-1}^{-1}) \leq (1+\nicefrac{1}{\lambda})d \log(1+\nicefrac{T}{d}).
\end{align*}
\end{lemma}
\begin{proof}
Observe that by concavity of the $\log \det(\cdot)$ function, we have
\begin{align*}
\log(\det(M_{t-1})) \leq \log(\det(M_t)) + \tr(M_{t}^{-1}(M_{t-1} - M_t)).
\end{align*}
Re-arranging and summing across all rounds $t$ yields
\begin{align*}
\sum_{t=1}^ T\tr(X_t M_t^{-1}) \leq \sum_{t=1}^T \log(\det(M_t)) - \log(\det(M_{t-1})) = \log(\det(M_T)) - d\lambda.
\end{align*}
We will drop the negative term. By the spectral version of the AM-GM
inequality and linearity of trace, we upper bound the last term:
\begin{align*}
\det(M_T)^{1/d} \leq \tr(M_T)/d \leq 1 + \nicefrac{T}{d}.
\end{align*}
Now, we must convert from $M_t^{-1}$ to $M_{t-1}^{-1}$ on the left
hand side. Fix a round $t$ and let us write $X_t = VV^\top$, which is
always possible as $X_t$ is positive semidefinite. Then by the
Woodbury identity
\begin{align*}
\tr(X_t M_t^{-1}) &= \tr\rbr{V^\top (M_{t-1} + VV^\top)^{-1} V} \\
&= \tr(V^\top M_{t-1}^{-1} V) - \tr(V^\top M_{t-1}^{-1} V(I + V^\top M_{t-1}^{-1} V)^{-1}V^\top M_{t-1}^{-1} V).
\end{align*}
All matrices are simultaneously diagonalizable, so we may pass to a
common eigendecomposition. In particular, with the eigendecomposition
$V^\top M_{t-1}^{-1} V = \sum_{i=1}^d \lambda_i u_iu_i^\top$, we
obtain
\begin{align*}
\tr(X_t M_t^{-1}) = \sum_{i=1}^d \lambda_i - \frac{\lambda_i^2}{1+\lambda_i} = \sum_{i=1}^d \frac{\lambda_i}{1+\lambda_i} \geq \frac{1}{1+\nicefrac{1}{\lambda}} \sum_{i=1}^d \lambda_i = \frac{1}{1+\nicefrac{1}{\lambda}}\tr(X_tM_{t-1}^{-1}).
\end{align*}
The inequality follows from the fact that $\lambda_i \leq \nbr{V^\top
  M_{t-1}^{-1} V}_2 \leq \nicefrac{1}{\lambda}$ due to our initial conditions on $M_0$ and
the normalization for $X_t$.
\end{proof}

\begin{corollary}
\label{corr:log_det}
Consider the setup of~\pref{lem:log_det} and further assume that for
each $t$, we have $\tr(X_tM_{t-1}^{-1}) \geq \beta > 0$. Then $T \leq
2(1+\nicefrac{1}{\lambda}) d\log
(1+2(1+\nicefrac{1}{\lambda})/\beta)/\beta$.
\end{corollary}
\begin{proof}
The stated assumption and \pref{lem:log_det} implies that
$T \leq (1+\nicefrac{1}{\lambda})d \log(1+\nicefrac{T}{d})/\beta$. 
We claim that if $T \leq 2(1+\nicefrac{1}{\lambda}) d\log
(1+2(1+\nicefrac{1}{\lambda})/\beta)/\beta$ then a weakening of this bound is
\begin{align*}
T &\leq \frac{(1+\nicefrac{1}{\lambda})d}{\beta} \log(1+\nicefrac{T}{d})/\beta \leq \frac{(1+\nicefrac{1}{\lambda})d}{\beta} \log\rbr{1+\frac{2(1+\nicefrac{1}{\lambda}) \log(1 + 2(1+\nicefrac{1}{\lambda})/\beta)}{\beta}}\\
&\leq \frac{(1+\nicefrac{1}{\lambda})d}{\beta} \log\rbr{1+\rbr{\frac{2(1+\nicefrac{1}{\lambda})}{\beta}}^2} \leq \frac{2(1+\nicefrac{1}{\lambda})d}{\beta} \log\rbr{1+\frac{2(1+\nicefrac{1}{\lambda})}{\beta}}.
\end{align*}
Therefore, we have established an upper bound on $T$.
\end{proof}
 
\bibliography{refs}

\begin{thebibliography}{59}
\providecommand{\natexlab}[1]{#1}
\providecommand{\url}[1]{\texttt{#1}}
\expandafter\ifx\csname urlstyle\endcsname\relax
  \providecommand{\doi}[1]{doi: #1}\else
  \providecommand{\doi}{doi: \begingroup \urlstyle{rm}\Url}\fi

\bibitem[Agarwal et~al.(2014)Agarwal, Hsu, Kale, Langford, Li, and
  Schapire]{agarwal2014taming}
Alekh Agarwal, Daniel Hsu, Satyen Kale, John Langford, Lihong Li, and Robert
  Schapire.
\newblock Taming the monster: A fast and simple algorithm for contextual
  bandits.
\newblock In \emph{International Conference on Machine Learning}, 2014.

\bibitem[Arora et~al.(2017)Arora, Ge, Liang, Ma, and
  Zhang]{arora2017generalization}
Sanjeev Arora, Rong Ge, Yingyu Liang, Tengyu Ma, and Yi~Zhang.
\newblock Generalization and equilibrium in generative adversarial nets (gans).
\newblock In \emph{International Conference on Machine Learning}, 2017.

\bibitem[Arora et~al.(2019)Arora, Khandeparkar, Khodak, Plevrakis, and
  Saunshi]{arora2019theoretical}
Sanjeev Arora, Hrishikesh Khandeparkar, Mikhail Khodak, Orestis Plevrakis, and
  Nikunj Saunshi.
\newblock A theoretical analysis of contrastive unsupervised representation
  learning.
\newblock In \emph{International Conference on Machine Learning}, 2019.

\bibitem[Barreto et~al.(2011)Barreto, Precup, and
  Pineau]{barreto2011reinforcement}
Andre Barreto, Doina Precup, and Joelle Pineau.
\newblock Reinforcement learning using kernel-based stochastic factorization.
\newblock In \emph{Advances in Neural Information Processing Systems}, 2011.

\bibitem[Barreto and Fragoso(2011)]{barreto2011computing}
Andr{\'e}~MS Barreto and Marcelo~D Fragoso.
\newblock Computing the stationary distribution of a finite markov chain
  through stochastic factorization.
\newblock \emph{SIAM Journal on Matrix Analysis and Applications}, 2011.

\bibitem[Baxter(2000)]{baxter2000model}
Jonathan Baxter.
\newblock A model of inductive bias learning.
\newblock \emph{Journal of artificial intelligence research}, 2000.

\bibitem[Bengio et~al.(2013)Bengio, Courville, and
  Vincent]{bengio2013representation}
Yoshua Bengio, Aaron Courville, and Pascal Vincent.
\newblock Representation learning: A review and new perspectives.
\newblock \emph{IEEE Transactions on Pattern Analysis and Machine
  Intelligence}, 2013.

\bibitem[Cai et~al.(2019)Cai, Yang, Jin, and Wang]{cai2019provably}
Qi~Cai, Zhuoran Yang, Chi Jin, and Zhaoran Wang.
\newblock Provably efficient exploration in policy optimization.
\newblock \emph{arXiv:1912.05830}, 2019.

\bibitem[Chen and Jiang(2019)]{chen2019information}
Jinglin Chen and Nan Jiang.
\newblock Information-theoretic considerations in batch reinforcement learning.
\newblock In \emph{International Conference on Machine Learning}, 2019.

\bibitem[Cohen and Rothblum(1993)]{cohen1993nonnegative}
Joel~E Cohen and Uriel~G Rothblum.
\newblock Nonnegative ranks, decompositions, and factorizations of nonnegative
  matices.
\newblock \emph{Linear Algebra and its Applications}, 1993.

\bibitem[Dani et~al.(2008)Dani, Hayes, and Kakade]{dani2008stochastic}
Varsha Dani, Thomas~P Hayes, and Sham~M Kakade.
\newblock Stochastic linear optimization under bandit feedback.
\newblock In \emph{Conference on Learning Theory}, 2008.

\bibitem[Devroye and Lugosi(2012)]{devroye2012combinatorial}
Luc Devroye and G{\'a}bor Lugosi.
\newblock \emph{Combinatorial methods in density estimation}.
\newblock Springer Science \& Business Media, 2012.

\bibitem[Dong et~al.(2019)Dong, Peng, Wang, and Zhou]{dong2019sqrt}
Kefan Dong, Jian Peng, Yining Wang, and Yuan Zhou.
\newblock $\sqrt{n}$-regret for learning in {M}arkov decision processes with
  function approximation and low {B}ellman rank.
\newblock \emph{arXiv:1909.02506}, 2019.

\bibitem[Du et~al.(2019{\natexlab{a}})Du, Kakade, Wang, and Yang]{du2019good}
Simon~S Du, Sham~M Kakade, Ruosong Wang, and Lin~F Yang.
\newblock Is a good representation sufficient for sample efficient
  reinforcement learning?
\newblock In \emph{International Conference on Learning Representations},
  2019{\natexlab{a}}.

\bibitem[Du et~al.(2019{\natexlab{b}})Du, Krishnamurthy, Jiang, Agarwal,
  Dud{\'\i}k, and Langford]{du2019provably}
Simon~S Du, Akshay Krishnamurthy, Nan Jiang, Alekh Agarwal, Miroslav
  Dud{\'\i}k, and John Langford.
\newblock Provably efficient {RL} with rich observations via latent state
  decoding.
\newblock In \emph{International Conference on Machine Learning},
  2019{\natexlab{b}}.

\bibitem[Du et~al.(2019{\natexlab{c}})Du, Luo, Wang, and Zhang]{du2019dsec}
Simon~S Du, Yuping Luo, Ruosong Wang, and Hanrui Zhang.
\newblock Provably efficient {Q}-learning with function approximation via
  distribution shift error checking oracle.
\newblock In \emph{Advances in Neural Information Processing Systems},
  2019{\natexlab{c}}.

\bibitem[Du et~al.(2020)Du, Lee, Mahajan, and Wang]{du2020agnostic}
Simon~S Du, Jason~D Lee, Gaurav Mahajan, and Ruosong Wang.
\newblock Agnostic {Q}-learning with function approximation in deterministic
  systems: Tight bounds on approximation error and sample complexity.
\newblock \emph{arXiv:2002.07125}, 2020.

\bibitem[Duan et~al.(2020)Duan, Wang, Wen, and Yuan]{duan2020adaptive}
Yaqi Duan, Mengdi Wang, Zaiwen Wen, and Yaxiang Yuan.
\newblock Adaptive low-nonnegative-rank approximation for state aggregation of
  {M}arkov chains.
\newblock \emph{SIAM Journal on Matrix Analysis and Applications}, 2020.

\bibitem[Edmonds(1965)]{edmonds1965maximum}
Jack Edmonds.
\newblock Maximum matching and a polyhedron with 0,1-vertices.
\newblock \emph{Journal of Research of the National Bureau of Standards--B},
  1965.

\bibitem[Feng et~al.(2020)Feng, Wang, Yin, Du, and Yang]{feng2020provably}
Fei Feng, Ruosong Wang, Wotao Yin, Simon~S Du, and Lin~F Yang.
\newblock Provably efficient exploration for rl with unsupervised learning.
\newblock \emph{arXiv:2003.06898}, 2020.

\bibitem[Figurnov et~al.(2018)Figurnov, Mohamed, and
  Mnih]{figurnov2018implicit}
Mikhail Figurnov, Shakir Mohamed, and Andriy Mnih.
\newblock Implicit reparameterization gradients.
\newblock In \emph{Advances in Neural Information Processing Systems}, 2018.

\bibitem[Fiorini et~al.(2013)Fiorini, Kaibel, Pashkovich, and
  Theis]{fiorini2013combinatorial}
Samuel Fiorini, Volker Kaibel, Kanstantsin Pashkovich, and Dirk~Oliver Theis.
\newblock Combinatorial bounds on nonnegative rank and extended formulations.
\newblock \emph{Discrete Mathematics}, 2013.

\bibitem[Goodfellow et~al.(2014)Goodfellow, Pouget-Abadie, Mirza, Xu,
  Warde-Farley, Ozair, Courville, and Bengio]{goodfellow2014generative}
Ian Goodfellow, Jean Pouget-Abadie, Mehdi Mirza, Bing Xu, David Warde-Farley,
  Sherjil Ozair, Aaron Courville, and Yoshua Bengio.
\newblock Generative adversarial nets.
\newblock In \emph{Advances in neural information processing systems}, 2014.

\bibitem[Hazan et~al.(2019)Hazan, Kakade, Singh, and
  Van~Soest]{hazan2018provably}
Elad Hazan, Sham~M Kakade, Karan Singh, and Abby Van~Soest.
\newblock Provably efficient maximum entropy exploration.
\newblock In \emph{International Conference on Machine Learning}, 2019.

\bibitem[Jang et~al.(2017)Jang, Gu, and Poole]{jang2017categorical}
Eric Jang, Shixiang Gu, and Ben Poole.
\newblock Categorical reparametrization with gumbel-softmax.
\newblock In \emph{International Conference on Learning Representations}, 2017.

\bibitem[Jiang et~al.(2015)Jiang, Kulesza, and Singh]{jiang2015abstraction}
Nan Jiang, Alex Kulesza, and Satinder Singh.
\newblock Abstraction selection in model-based reinforcement learning.
\newblock In \emph{International Conference on Machine Learning}, 2015.

\bibitem[Jiang et~al.(2017)Jiang, Krishnamurthy, Agarwal, Langford, and
  Schapire]{jiang2017contextual}
Nan Jiang, Akshay Krishnamurthy, Alekh Agarwal, John Langford, and Robert~E
  Schapire.
\newblock Contextual decision processes with low {B}ellman rank are
  {PAC}-learnable.
\newblock In \emph{International Conference on Machine Learning}. JMLR. org,
  2017.

\bibitem[Jin et~al.(2019)Jin, Yang, Wang, and Jordan]{jin2019provably}
Chi Jin, Zhuoran Yang, Zhaoran Wang, and Michael~I Jordan.
\newblock Provably efficient reinforcement learning with linear function
  approximation.
\newblock \emph{arXiv:1907.05388}, 2019.

\bibitem[Jin et~al.(2020)Jin, Krishnamurthy, Simchowitz, and Yu]{jin2020reward}
Chi Jin, Akshay Krishnamurthy, Max Simchowitz, and Tiancheng Yu.
\newblock Reward-free exploration for reinforcement learning.
\newblock \emph{arXiv:2002.02794}, 2020.

\bibitem[Lattimore and Szepesvari(2019)]{lattimore2019learning}
Tor Lattimore and Csaba Szepesvari.
\newblock Learning with good feature representations in bandits and in rl with
  a generative model.
\newblock \emph{arXiv:1911.07676}, 2019.

\bibitem[Lattimore et~al.(2013)Lattimore, Hutter, Sunehag,
  et~al.]{lattimore2013sample}
Tor Lattimore, Marcus Hutter, Peter Sunehag, et~al.
\newblock The sample-complexity of general reinforcement learning.
\newblock In \emph{International Conference on Machine Learning}. Journal of
  Machine Learning Research, 2013.

\bibitem[Littman and Sutton(2002)]{littman2002predictive}
Michael~L Littman and Richard~S Sutton.
\newblock Predictive representations of state.
\newblock In \emph{Advances in Neural Information Processing Systems}, 2002.

\bibitem[Maurer et~al.(2016)Maurer, Pontil, and
  Romera-Paredes]{maurer2016benefit}
Andreas Maurer, Massimiliano Pontil, and Bernardino Romera-Paredes.
\newblock The benefit of multitask representation learning.
\newblock \emph{The Journal of Machine Learning Research}, 2016.

\bibitem[Misra et~al.(2019)Misra, Henaff, Krishnamurthy, and
  Langford]{misra2019kinematic}
Dipendra Misra, Mikael Henaff, Akshay Krishnamurthy, and John Langford.
\newblock Kinematic state abstraction and provably efficient rich-observation
  reinforcement learning.
\newblock \emph{arXiv:1911.05815}, 2019.

\bibitem[Modi et~al.(2020)Modi, Jiang, Tewari, and Singh]{modi2019sample}
Aditya Modi, Nan Jiang, Ambuj Tewari, and Satinder Singh.
\newblock Sample complexity of reinforcement learning using linearly combined
  model ensembles.
\newblock In \emph{Conference on Artificial Intelligence and Statistics}, 2020.

\bibitem[Oord et~al.(2018)Oord, Li, and Vinyals]{oord2018representation}
Aaron van~den Oord, Yazhe Li, and Oriol Vinyals.
\newblock Representation learning with contrastive predictive coding.
\newblock \emph{arXiv:1807.03748}, 2018.

\bibitem[Ortner et~al.(2014)Ortner, Maillard, and Ryabko]{ortner2014selecting}
Ronald Ortner, Odalric-Ambrym Maillard, and Daniil Ryabko.
\newblock Selecting near-optimal approximate state representations in
  reinforcement learning.
\newblock In \emph{International Conference on Algorithmic Learning Theory}.
  Springer, 2014.

\bibitem[Osband and Van~Roy(2014)]{osband2014model}
Ian Osband and Benjamin Van~Roy.
\newblock Model-based reinforcement learning and the eluder dimension.
\newblock In \emph{Advances in Neural Information Processing Systems}, 2014.

\bibitem[Pathak et~al.(2017)Pathak, Agrawal, Efros, and
  Darrell]{pathak2017curiosity}
Deepak Pathak, Pulkit Agrawal, Alexei~A Efros, and Trevor Darrell.
\newblock Curiosity-driven exploration by self-supervised prediction.
\newblock In \emph{IEEE Conference on Computer Vision and Pattern Recognition
  Workshops}, 2017.

\bibitem[Rendle et~al.(2010)Rendle, Freudenthaler, and
  Schmidt-Thieme]{rendle2010factorizing}
Steffen Rendle, Christoph Freudenthaler, and Lars Schmidt-Thieme.
\newblock Factorizing personalized {M}arkov chains for next-basket
  recommendation.
\newblock In \emph{International Conference on World Wide Web}, 2010.

\bibitem[Rothvo{\ss}(2017)]{rothvoss2017matching}
Thomas Rothvo{\ss}.
\newblock The matching polytope has exponential extension complexity.
\newblock \emph{Journal of the ACM}, 2017.

\bibitem[Russo and Van~Roy(2013)]{russo2013eluder}
Daniel Russo and Benjamin Van~Roy.
\newblock Eluder dimension and the sample complexity of optimistic exploration.
\newblock In \emph{Advances in Neural Information Processing Systems}, 2013.

\bibitem[Singh et~al.(2004)Singh, James, and Rudary]{singh2004predictive}
Satinder Singh, Michael~R James, and Matthew~R Rudary.
\newblock Predictive state representations: a new theory for modeling dynamical
  systems.
\newblock In \emph{Conference on Uncertainty in Artificial Intelligence}, 2004.

\bibitem[Srinivas et~al.(2020)Srinivas, Laskin, and Abbeel]{srinivas2020curl}
Aravind Srinivas, Michael Laskin, and Pieter Abbeel.
\newblock Curl: Contrastive unsupervised representations for reinforcement
  learning.
\newblock \emph{arXiv:2004.04136}, 2020.

\bibitem[Sun et~al.(2019)Sun, Jiang, Krishnamurthy, Agarwal, and
  Langford]{sun2018model}
Wen Sun, Nan Jiang, Akshay Krishnamurthy, Alekh Agarwal, and John Langford.
\newblock Model-based {RL} in contextual decision processes: {PAC} bounds and
  exponential improvements over model-free approaches.
\newblock In \emph{Conference on Learning Theory}, 2019.

\bibitem[Tang et~al.(2017)Tang, Houthooft, Foote, Stooke, Chen, Duan, Schulman,
  DeTurck, and Abbeel]{tang2017exploration}
Haoran Tang, Rein Houthooft, Davis Foote, Adam Stooke, Xi~Chen, Yan Duan, John
  Schulman, Filip DeTurck, and Pieter Abbeel.
\newblock \#{E}xploration: A study of count-based exploration for deep
  reinforcement learning.
\newblock In \emph{Advances in Neural Information Processing Systems}, 2017.

\bibitem[Thon and Jaeger(2015)]{thon2015links}
Michael Thon and Herbert Jaeger.
\newblock Links between multiplicity automata, observable operator models and
  predictive state representations: a unified learning framework.
\newblock \emph{The Journal of Machine Learning Research}, 2015.

\bibitem[Tosh et~al.(2020)Tosh, Krishnamurthy, and Hsu]{tosh2020contrastive}
Christopher Tosh, Akshay Krishnamurthy, and Daniel Hsu.
\newblock Contrastive estimation reveals topic posterior information to linear
  models.
\newblock \emph{arXiv:2003.02234}, 2020.

\bibitem[Tsybakov(2008)]{tsybakov2008introduction}
Alexandre~B Tsybakov.
\newblock \emph{Introduction to nonparametric estimation}.
\newblock Springer Science \& Business Media, 2008.

\bibitem[Van~de Geer(2000)]{geer2000empirical}
Sara Van~de Geer.
\newblock \emph{Empirical Processes in M-estimation}.
\newblock Cambridge University Press, 2000.

\bibitem[Van~Roy and Dong(2019)]{van2019comments}
Benjamin Van~Roy and Shi Dong.
\newblock Comments on the {D}u-{K}akade-{W}ang-{Y}ang lower bounds.
\newblock \emph{arXiv:1911.07910}, 2019.

\bibitem[Wang et~al.(2019)Wang, Wang, Du, and Krishnamurthy]{wang2019optimism}
Yining Wang, Ruosong Wang, Simon~S Du, and Akshay Krishnamurthy.
\newblock Optimism in reinforcement learning with generalized linear function
  approximation.
\newblock \emph{arXiv:1912.04136}, 2019.

\bibitem[Welling and Teh(2011)]{welling2011bayesian}
Max Welling and Yee~W Teh.
\newblock Bayesian learning via stochastic gradient {L}angevin dynamics.
\newblock In \emph{International Conference on Machine Learning}, 2011.

\bibitem[Wen and Van~Roy(2013)]{wen2013efficient}
Zheng Wen and Benjamin Van~Roy.
\newblock Efficient exploration and value function generalization in
  deterministic systems.
\newblock In \emph{Advances in Neural Information Processing Systems}, 2013.

\bibitem[Yang and Wang(2019{\natexlab{a}})]{yang2019sample}
Lin Yang and Mengdi Wang.
\newblock Sample-optimal parametric {Q}-learning using linearly additive
  features.
\newblock In \emph{International Conference on Machine Learning},
  2019{\natexlab{a}}.

\bibitem[Yang and Wang(2019{\natexlab{b}})]{yang2019reinforcement}
Lin~F Yang and Mengdi Wang.
\newblock Reinforcement leaning in feature space: Matrix bandit, kernels, and
  regret bound.
\newblock \emph{arXiv:1905.10389}, 2019{\natexlab{b}}.

\bibitem[Yannakakis(1991)]{yannakakis1991expressing}
Mihalis Yannakakis.
\newblock Expressing combinatorial optimization problems by linear programs.
\newblock \emph{Journal of Computer and System Sciences}, 1991.

\bibitem[Yao et~al.(2014)Yao, Szepesv{\'a}ri, Pires, and Zhang]{yao2014pseudo}
Hengshuai Yao, Csaba Szepesv{\'a}ri, Bernardo~Avila Pires, and Xinhua Zhang.
\newblock Pseudo-mdps and factored linear action models.
\newblock In \emph{IEEE Symposium on Adaptive Dynamic Programming and
  Reinforcement Learning}, 2014.

\bibitem[Zhang(2006)]{zhang2006from}
Tong Zhang.
\newblock From $\epsilon$-entropy to {KL}-entropy: Analysis of minimum
  information complexity density estimation.
\newblock \emph{The Annals of Statistics}, 2006.

\end{thebibliography}
\vfill

\clearpage

\end{document}